\documentclass{article} 
\usepackage{iclr2027_conference,times}

\usepackage{amsmath,amsfonts,bm}

\def\eqref#1{equation~\ref{#1}}

\def\1{\bm{1}}

\DeclareMathAlphabet{\mathsfit}{\encodingdefault}{\sfdefault}{m}{sl}
\SetMathAlphabet{\mathsfit}{bold}{\encodingdefault}{\sfdefault}{bx}{n}

\usepackage{upgreek}
\usepackage{hyperref}
\usepackage{url}
\usepackage{algorithm}
\usepackage{wrapfig}
\usepackage{algorithmic}
\usepackage{xcolor}
\usepackage{amsmath,amssymb}
\usepackage{newfloat}
\usepackage{listings}
\usepackage{multirow}
\usepackage{booktabs} 
\usepackage{graphicx}
\usepackage{subcaption}
\usepackage{amsthm}
\usepackage{hyperref}
\usepackage{enumitem}
\usepackage{thm-restate}
\usepackage{hyperref} 
\usepackage{xspace}
\usepackage{tabularx}
\usepackage{longtable}
\usepackage{array}
\usepackage{booktabs}
\usepackage{ragged2e}
\newcommand{\HARMONIA}{\nolinkurl{HARMONIA}\xspace}

\usepackage[most]{tcolorbox}

\newtcolorbox{insightbox}{
    colback=gray!8,
    colframe=black!75,
    boxrule=0.8pt,
    arc=4pt,
    left=7pt,
    right=7pt,
    top=5pt,
    bottom=5pt,
    before skip=6pt,
    after skip=6pt,
    breakable
}

\newtheorem{lemma}{Lemma}
\newtheorem{theorem}{Theorem}

\newtheorem{proposition}{Proposition}
\newtheorem{corollary}{Corollary}

\title{HARMONIA: Interpretable Graph Learning through Mixtures of Neural Bases}

\author{
Quan D.~Bui\textsuperscript{1,}\thanks{Equal contribution.}
\qquad
Nguyen Do\textsuperscript{2,}\footnotemark[1]
\qquad
An Nguyen Dang\textsuperscript{3,}\footnotemark[1]
\\[0.3em]
\textbf{Huyen Nguyen}\textsuperscript{3}
\qquad
\textbf{Nhu Duc Minh Nguyen}\textsuperscript{4}
\qquad
\textbf{My T.~Thai}\textsuperscript{2,}\thanks{Correspondence to: \texttt{mythai@cise.ufl.edu}.}
\\[0.6em]
\textsuperscript{1}Center for AI Research, VinUniversity, Hanoi, Vietnam
\\
\textsuperscript{2}University of Florida, Gainesville, USA
\\
\textsuperscript{3}Posts and Telecommunications Institute of Technology, Hanoi, Vietnam
\\
\textsuperscript{4}University of Illinois Chicago, Chicago, USA
}
\iclrfinalcopy 
\begin{document}

\maketitle

\vspace{-15pt}

\begin{abstract}
Existing interpretable graph additive models still face limitations in either computational scalability or modeling flexibility. In terms of structural modeling, previous approaches either face quadratic scaling costs or sacrifice explicit source-to-target contribution decomposition. In terms of feature components, they rely either on per-feature neural networks or on single shared bases with limited feature specialization. We address both problems by introducing \HARMONIA: Interpretable Graph Learning through Mixtures of Neural Bases, an interpretable-by-design framework. For feature modeling, \HARMONIA introduces a Mixture of Neural Bases (MoNB), which routes features to specialized basis experts, enabling parameter sharing without sacrificing feature-specific specialization. For structural modeling, \HARMONIA uses Relative Random Walk Probabilities (RRWP) to capture multi-hop and multi-path relationships, and proposes Sparse RRWP Aggregation (SRA) to compute these interactions through sparse graph propagation without quadratic pairwise complexity. \HARMONIA retains a simple additive form in which predictions decompose into feature responses modulated by structural influence. Empirically, \HARMONIA achieves stronger explanation recovery than existing interpretable graph baselines while maintaining competitive predictive performance and scaling to graphs with millions of nodes. These results show that interpretable graph learning can remain both faithful and scalable without sacrificing predictive effectiveness.
\end{abstract}

\section{Introduction}

Graph neural networks (GNNs) have become common for learning relational data, but their predictions are difficult to interpret because
node attributes and graph structure are entangled within latent
representations during message passing. Post-hoc explanation methods seek to
explain a trained GNN by fitting or optimizing an auxiliary explanation
mechanism, such as a surrogate model, feature mask, or subgraph selector
\citep{ying2019gnnexplainer,luo2020parameterized,huang2022graphlime}. Because
these explanations are constructed after the predictor has been learned,
they approximate or summarize the behavior of the original model rather than determine the mechanism that actually produced its prediction. Their
faithfulness depends on how the auxiliary explanation captures the
behavior of the underlying GNN. This limitation has motivated
interpretable-by-design graph models, in which feature and structural effects
are explicit components of the predictive function itself.

Graph additive models (GrAMs)~\citep{bechler2024intelligible,reddy2025interpretable} provide a route toward this goal by retaining
feature-wise response functions while aggregating information from the graph. Making these models both expressive and scalable introduces two
difficulties. First, standard neural additive models use a separate neural network for each feature, so parameter and computational to grow with feature dimension. Neural Basis Models (NBMs) reduce this cost by representing feature
responses through a shared set of basis functions
\citep{radenovic2022neural}. However, this shared representation places all feature-response functions within a common global basis space. When features follow heterogeneous functional patterns, a small shared space may underfit them, whereas expanding the global basis increases capacity for all features, reducing efficiency.
 
The second difficulty lies in structural modeling, where interpretable GrAMs often describe node-to-node influence using simple structural quantities such as shortest-path distance \citep{bechler2024intelligible}. Such summaries do not capture the transition probabilities between node pairs across different walk lengths and paths, even when these differences induce distinct structural effects.
Relative Random Walk Probabilities (RRWP) provide a richer description by
recording transition probabilities across multiple walk lengths
\citep{ma2023GraphInductiveBiases}. However, storing RRWP for all node pairs requires quadratic memory, making structural representations impractical for large graphs.

These observations lead to the question:
\textit{\textbf{Can an interpretable GrAM remain functionally flexible and
computationally practical at the same time, even as both the feature dimension
and graph size increase?}}
To address this question, we introduce \HARMONIA, a GrAM designed to
scale along both dimensions while preserving an explicit additive prediction
structure. On the feature side, \HARMONIA introduces a \emph{Mixture of Neural
Basis (MoNB)}, which replaces a single globally shared basis space with
multiple neural-basis experts and sparsely routes each feature to specialized
functional spaces. This allows related features to share basis functions
without forcing heterogeneous response functions into the same representation
or requiring a separate neural network for every feature. On the structural side, \HARMONIA uses RRWP to capture multi-hop and multi-path
relationships and introduces \emph{Sparse RRWP Aggregation (SRA)}, which
computes linear RRWP aggregation through sparse graph propagation without constructing dense pairwise representations to reduce complexity of RRWP. Thus, \HARMONIA remains interpretable by design while improving functional
flexibility and structural scalability without sacrificing its additive structure.

We further develop a theoretical analysis of both components of \HARMONIA. On
the feature side, we characterize the limitation of a single globally shared
basis space and introduce a \emph{specialization gap} that quantifies the
approximation benefit achieved by multiple specialized basis spaces. We
then analyze how routing errors reduce this gain and identify when
specialization remains beneficial despite imperfect expert assignment. On the
structural side, we study when RRWP hop coefficients admit unique
interpretations and show how the spectral properties of the random walk govern
the number of structural scales that remain distinguishable. This analysis
also yields an effective structural horizon beyond which additional walk
lengths provide increasingly redundant information. Our contributions are summarized as follows:

\begin{itemize}
    \item We introduce \HARMONIA, an interpretable and scalable GrAM
    that combines MoN for feature
    specialization with SRA, which computes
    multi-hop structural effects exactly without materializing dense
    node-pair operators.

    \item We develop theory for each component. For MoNB, we characterize
    when specialization outweighs routing error. For RRWP, we derive
    conditions for identifiable hop effects and an effective structural
    horizon that determines how many propagation depths remain distinguishable.
    We further prove that SRA significantly reduces computation from \(O(Tn^2d)\) to
    \(O(T|E|d)\); in the reversible setting, the structural horizon
    yields an even tighter graph-dependent propagation budget.

    \item We evaluate \HARMONIA on graph- and node-level benchmarks, showing
    competitive predictive performance, faithful feature and multi-hop
    structural explanations, and improved scalability on large sparse graphs.
\end{itemize}

\section{Preliminaries and Problem Setup}
\label{sec:preliminaries}

\paragraph{Neural Basis Models and Functional Sharing.}
Neural Basis Models (NBMs) represent feature-specific nonlinear effects using
a shared set of learnable basis functions. With \(B\) neural bases and
feature-specific coefficients, the response of feature \(k\) is written as
\begin{equation}
\boldsymbol{\psi}(x)
=
[\psi_1(x),\ldots,\psi_B(x)]^\top,
\qquad
f_k(x)
=
\mathbf a_k^\top\boldsymbol{\psi}(x)
=
\sum_{b=1}^{B}a_{k,b}\psi_b(x).
\label{eq:nbm}
\end{equation}
Here, the basis vector \(\boldsymbol{\psi}(x)\) is shared across features,
whereas each feature \(k\) learns its own coefficient vector
\(\mathbf a_k=(a_{k,1},\ldots,a_{k,B})^\top\). Thus, different features can
form different response functions by combining the same bases with different
weights. However, because every response is constructed from the same basis
functions, all \(f_k\) lie in the common functional space
\(\mathcal S_B=\operatorname{span}\{\psi_1,\ldots,\psi_B\}\), \(dim(\mathcal S_B) \leq B\). This sharing reduces the number of nonlinear functions that
must be learned, but can become restrictive when different features require
substantially different functional structures. Increasing \(B\) enlarges the
global basis space, but also increases the representation shared by every
feature. This motivates replacing a single global space with multiple
specialized basis spaces while retaining parameter sharing across features
with similar responses.

\paragraph{Relative Random Walk Probabilities.}
Relative Random Walk Probabilities (RRWP) describe the structural relation
between two nodes across multiple walk lengths. Given adjacency matrix
\(\mathbf A\) and degree matrix \(\mathbf D\), let
\(\mathbf M=\mathbf D^{-1}\mathbf A\) denote the random-walk transition
matrix. Using \(T\) walk orders, the RRWP descriptor between nodes
\(v_i\) and \(v_j\) is
\begin{equation}
\mathbf p_{i,j}
=
\left[
(\mathbf M^{0})_{i,j},
(\mathbf M^{1})_{i,j},
\ldots,
(\mathbf M^{T-1})_{i,j}
\right]^{\top}
\in\mathbb R^{T}.
\label{eq:rrwp}
\end{equation}
Here, \((\mathbf M^{t})_{i,j}\) is the probability that a random walker
starting from \(v_i\) reaches \(v_j\) after exactly \(t\) steps, allowing
\(\mathbf p_{i,j}\) to describe their relation across multiple structural
scales. A direct pairwise implementation materializes an
\(n\times n\times T\) RRWP tensor, requiring
\(\mathcal O(Tn^2)\) memory, while applying all walk channels across
\(d\) feature dimensions can require
\(\mathcal O(Tn^2d)\) computation.

\paragraph{Problem Formulation.}
Let \(\mathcal G=(\mathcal V,\mathcal E,\mathbf X)\) be a graph with
\(n=|\mathcal V|\) nodes and
\(\mathbf X=[x_{i,k}]\in\mathbb R^{n\times d}\).
We seek a graph predictor that incorporates structural information while
preserving an explicit decomposition over the original features. For each
feature \(k\), let \(f_k:\mathbb R\rightarrow\mathbb R\) denote its response
function and let
\(\omega_k^{\sharp}:\mathbb R^T\rightarrow\mathbb R\) be a feature-specific topology
function acting on \(\mathbf p_{i,j}\). The presentation of \(k\)-th feature of node \(i\) is 
\begin{equation}
[h_i]_k
=
\sum_{j\in\mathcal V}
\omega_k^{\sharp}(\mathbf p_{i,j})f_k(x_{j,k}),
\qquad
k=1,\ldots,d.
\label{eq:problem_representation}
\end{equation}
The term \(f_k(x_{j,k})\) describes what feature \(k\) contributes at source
node \(v_j\), while \(\omega_k^{\sharp}(\mathbf p_{i,j})\) determines how structure
modulates its influence on target node \(v_i\). For class \(c\), the logit is
\begin{equation}
\ell_{i,c}
=
\beta_c
+
\sum_{k=1}^{d}
\sum_{j\in\mathcal V}
w_{k,c}\omega_k^{\sharp}(\mathbf p_{i,j})f_k(x_{j,k}),
\label{eq:problem_logit}
\end{equation}
which yields an exact feature--node decomposition at the logit level. The
remaining challenge is to make \(f_k\) flexible without learning an independent
high-capacity model for every feature, while modeling multi-hop structural
influence without dense pairwise computation. \HARMONIA addresses these two
requirements through specialized neural bases and sparse RRWP aggregation.

\begin{figure*}[t]
\centering
\includegraphics[width=0.95\textwidth]{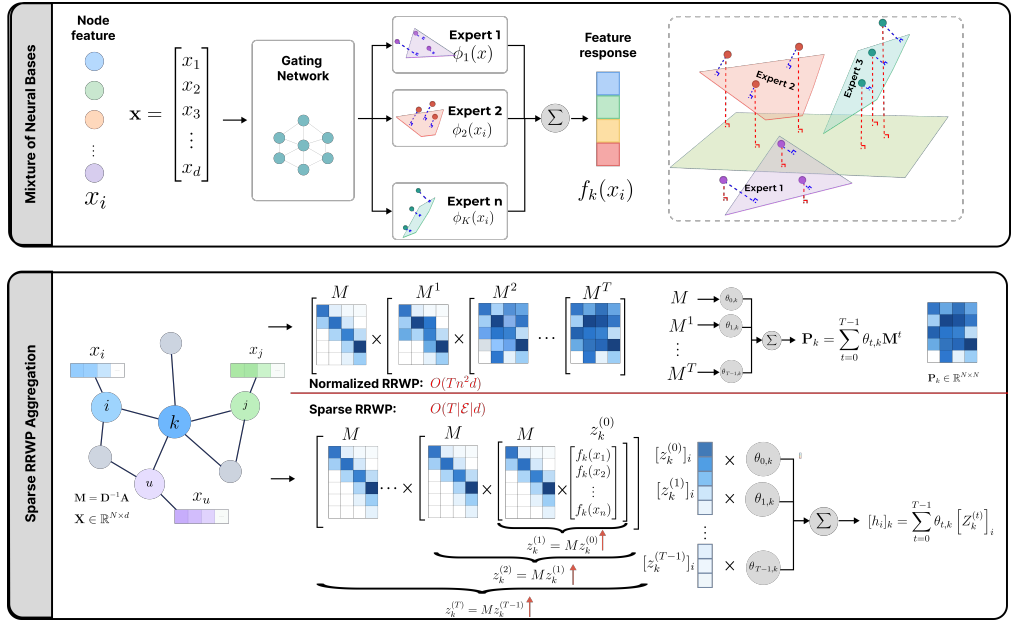}
\caption{\HARMONIA combines feature-dependent mixtures of neural bases
with normalized random-walk aggregation. MoNB shares specialized
functional spaces across features; SRA evaluates their structural
contributions through sparse propagation. The resulting logits retain
an exact feature--node--walk decomposition.}
\label{fig:famework}
\end{figure*}

\section{HARMONIA}
\label{sec:harmonia}

The formulation in Section~\ref{sec:preliminaries} raises two challenges: modeling heterogeneous feature responses efficiently and capturing multi-scale structural influence without dense pairwise computation. \HARMONIA uses a \emph{Mixture of Neural Bases} (MoNB) to route features to specialized basis experts. It uses normalized RRWP and \emph{Sparse RRWP Aggregation} (SRA) to evaluate multi-hop structural effects through sparse graph propagation. Both components preserve the additive form of the predictor, so each logit remains exactly decomposable across features, source nodes, and walk lengths.

\subsection{Feature Specialization through Mixtures of Neural Bases}
\label{subsec:monb}

A standard NBM represents every feature response in the same shared basis
space \(\mathcal S_B\). This is efficient when feature responses have similar
functional structure, but can be restrictive when heterogeneous features
require different functional directions. Increasing \(B\) enlarges this global
space for every feature, but also expands the representation used by
every feature. \HARMONIA instead introduces multiple shared basis
experts, allowing different features to access specialized functional spaces
while retaining parameter sharing. Specifically, \HARMONIA uses \(C\) experts, each containing \(B\) neural basis
functions,
\begin{equation}
\boldsymbol{\psi}_{\varphi_c}(x)
=
\left[
\psi^{(c)}_1(x),\ldots,\psi^{(c)}_B(x)
\right]^\top
\in\mathbb R^B,
\end{equation}
where \(\varphi_c\) denotes the parameters of expert \(c\). Each feature \(k\)
has a learnable embedding \(\mathbf e_k\in\mathbb R^{d_e}\), from which a sparse
router produces expert scores
\(\mathbf Q_k=[Q_{k,1},\ldots,Q_{k,C}]^\top\).
Let \(\mathcal I_k\) denote the top-\(m\) selected experts. The feature
response is
\begin{equation}
f_k(x)
=
\mathbf a_k^\top
\left(
\sum_{c\in\mathcal I_k}
\tilde{\pi}_{k,c}\boldsymbol{\psi}_{\varphi_c}(x)
\right),
\qquad
\tilde{\pi}_{k,c}
=
\mathbf 1\{c\in\mathcal I_k\}\sigma(Q_{k,c}),
\label{eq:monb}
\end{equation}
where
\(\mathbf a_k=(a_{k,1},\ldots,a_{k,B})^\top\in\mathbb R^B\), 
contains the feature-specific coefficients and $\sigma(\cdot)$ denotes the sigmoid function. The noisy score parameterization
used during training is given in
Appendix~\ref{subsec:app_monb_routing}. For fixed routing weights, the mixture induces \(B\) feature-dependent
effective bases,
\begin{equation}
\widetilde{\psi}_{k,b}(x)
=
\sum_{c\in\mathcal I_k}
\tilde{\pi}_{k,c}\psi_b^{(c)}(x),
\qquad
f_k(x)
=
\sum_{b=1}^{B}
a_{k,b}\widetilde{\psi}_{k,b}(x).
\label{eq:effective_monb_basis}
\end{equation}
Thus, routing changes the functional directions available to each feature
without increasing its \(B\)-dimensional coefficient vector. Features with
similar responses can reuse expert bases, whereas heterogeneous features can
form their responses in different effective basis spaces. Because routing
depends only on feature identity through \(\mathbf e_k\), each \(f_k\)
remains a globally defined univariate response function that can be inspected
directly.

\paragraph{When does specialization help?}
Appendix~\ref{subsec:app_global_basis} shows that constraining all feature
responses to one \(B\)-dimensional space incurs an irreducible approximation
error determined by the global functional spectrum. Allowing multiple
specialized spaces can reduce this error. We measure the best possible gain
from specialization by
\begin{equation}
\Gamma_{B,C}
:=
E_{\mathrm{global}}(B)
-
E_{\mathrm{mix}}^\star(B,C)
\ge 0,
\label{eq:specialization_gap}
\end{equation}
where \(E_{\mathrm{global}}(B)\) is the minimum approximation error achievable
by one shared \(B\)-dimensional space, and
\(E_{\mathrm{mix}}^\star(B,C)\) is the oracle error using \(C\) specialized
spaces. The quantity \(\Gamma_{B,C}\) describes the gain under ideal expert
assignment. In practice, routing may assign a feature to a less suitable
functional space. To isolate this effect, consider the hard-routing case and
the oracle expert spaces
\(\mathcal S_1^\star,\ldots,\mathcal S_C^\star\)
that attain \(E_{\mathrm{mix}}^\star(B,C)\). For feature \(k\), define its
error under expert \(c\) as
\(e_k(c)=
\|f_k^\star-\Pi_{\mathcal S_c^\star}f_k^\star\|_{\mathcal H}^2\), where $\Pi_{\mathcal S}f_k^\star$ is the closest approximation of \(f_k^\star\) in $\mathcal S$,
and let \(z_k^\star\) denote an oracle assignment minimizing this error.
For a routing assignment \(\widehat z\), the average excess error introduced
by routing is
\begin{equation}
    R_{\mathrm{route}}(\widehat z)
=
\frac{1}{d}
\sum_{k=1}^{d}
\left[
e_k(\widehat z_k)-e_k(z_k^\star)
\right]
\ge 0.
\end{equation}

\begin{restatable}[Specialization--Routing Trade-off]
    {theorem}{specializationrouting}
\label{thm:app_specialization_routing_tradeoff}
Let \(E_{\mathrm{route}}(\widehat z)\) denote the approximation error obtained
using the oracle expert spaces but routing assignment \(\widehat z\). The $E_{\mathrm{global}}(B)
-
E_{\mathrm{route}}(\widehat z)
=
\Gamma_{B,C}
-
R_{\mathrm{route}}(\widehat z).$ Hence, specialization improves over the best globally shared
\(B\)-dimensional space if and only if
\(R_{\mathrm{route}}(\widehat z)<\Gamma_{B,C}\).
\end{restatable}

\begin{insightbox}
\textbf{Insight.}
The benefit of MoNB depends on a trade-off between specialization
gain and routing cost. Larger \(\Gamma_{B,C}\) allows more routing error before specialization loses its advantage over global sharing.
Appendix~\ref{par:app_topm} provides details and extension to sparse top-\(m\) routing.
\end{insightbox}

\subsection{Normalized Random-Walk Structural Modeling}
\label{subsec:normalized_rrwp}
After learning the feature response \(f_k\), \HARMONIA determines how that response propagates through the graph. Features can depend on different structural scales. One may rely mainly on local information, while another needs information propagated over several steps. Existing interpretable graph additive models~\citep{bechler2024intelligible,reddy2025interpretable} often represent structural relations with a single distance or neighborhood statistic, which can hide differences across propagation lengths and alternative paths. \HARMONIA instead retains the multi-scale structural information captured by RRWP. For feature \(k\), we define
\begin{equation}
\omega_k^{\sharp}(\mathbf p_{i,j})
=
\sum_{t=0}^{T-1}
\theta_{t,k}(\mathbf M^t)_{i,j},
\qquad
[\mathbf P_k]_{i,j}
=
\omega_k^{\sharp}(\mathbf p_{i,j}),
\label{eq:linear_rrwp}
\end{equation}
where \(\theta_{t,k}\) controls the structural weight assigned to walk length
\(t\), and \(\mathbf P_k\in\mathbb R^{n\times n}\) is the resulting feature-specific propagation operator. Different features can therefore emphasize different propagation depths while sharing the same random-walk operators.
\paragraph{Why normalize structural influence?}
Without normalization, the decomposition between feature response and structural influence has a multiplicative scale ambiguity. For any
\(a>0\), replacing \(f_k\) by \(af_k\) and \(\omega_k^{\sharp}\) by \(\omega_k^{\sharp}/a\)
leaves \(\omega_k^{\sharp}(\mathbf p_{i,j})f_k(x_{j,k})\) unchanged. \HARMONIA resolves this ambiguity by learning unconstrained parameters
\(\alpha_{t,k}\in\mathbb R\) and converting them into normalized hop weights,
\begin{equation}
\theta_{t,k}
=
\frac{\alpha_{t,k}^2}
{\sum_{s=0}^{T-1}\alpha_{s,k}^2},
\qquad
\theta_{t,k}\ge0,
\qquad
\sum_{t=0}^{T-1}\theta_{t,k}=1.
\label{eq:rrwp_normalization}
\end{equation}
The model uses \(\alpha_{t,k}\) for optimization and exposes \(\theta_{t,k}\) as the structural quantities. Since \(\mathbf M\) is row-stochastic, every power \(\mathbf M^t\) is also row-stochastic. Therefore, \(\mathbf P_k=\sum_{t=0}^{T-1}\theta_{t,k}\mathbf M^t\) is a convex random-walk operator. Normalization fixes the overall scale of structural influence and places the hop coefficients on a common reference scale, but individual hop weights may still be non-unique. Appendix~\ref{subsec:app_hop_identifiability} shows that the mapping
\(
\boldsymbol{\theta}\mapsto
\sum_{t=0}^{T-1}\theta_t\mathbf M^t
\)
is injective on the probability simplex if and only if \(\mathbf I,\mathbf M,\ldots,\mathbf M^{T-1}\) are linearly independent. A normalized coefficient therefore has a unique hop-level interpretation only when the corresponding random-walk operators contain independent structural information. Moreover, because \(\mathbf P_k\) is a convex random-walk operator, it is non-expansive in \(\ell_\infty\), for any two node
signals \(\mathbf z,\mathbf z'\in\mathbb R^n\), 
\(
\|\mathbf P_k\mathbf z-\mathbf P_k\mathbf z'\|_\infty
\le
\|\mathbf z-\mathbf z'\|_\infty
\),
so normalized structural propagation cannot amplify the largest node-wise perturbation; see Appendix~\ref{subsec:app_rrwp_tradeoff}.

Even when hop coefficients are uniquely identifiable, successive propagation depths may become harder to distinguish. As the random walk mixes, operators associated with later walk lengths become increasingly similar~\citep{3504035.3504468,oono2020graph}, making additional hops structurally redundant. To study this effect, we consider an ergodic reversible random walk, i.e., an irreducible and aperiodic reversible Markov chain, with stationary distribution \(\boldsymbol{\pi}\). This includes random walks on connected undirected graphs, with laziness added when needed to remove periodicity~\citep{levin2017markov}. These assumptions apply only to the following spectral analysis and are not required by the \HARMONIA architecture. Under reversibility, \(\mathbf M\) is self-adjoint with respect to the weighted inner product \(\langle\mathbf u,\mathbf v\rangle_\pi=\sum_i\pi_i u_i v_i\), with induced operator norm \(\|\cdot\|_{2,\pi}\). Writing the eigenvalues of \(\mathbf M\) as \(1=\lambda_1,\lambda_2,\ldots,\lambda_n\), we use \(\rho=\max_{r\ge2}|\lambda_r|<1\) to capture the slowest-decaying nonstationary mode. We measure the distinction between consecutive propagation operators by \(D_t=\|\mathbf M^{t+1}-\mathbf M^t\|_{2,\pi}\). For a resolution \(\varepsilon>0\), the corresponding \emph{effective structural horizon} \(H_\varepsilon=\min\{t\ge0:D_s\le\varepsilon\ \text{for all }s\ge t\}\) is the earliest depth beyond which each additional propagation step changes the random-walk operator by at most \(\varepsilon\). We have the following theorem:

\begin{restatable}[Spectral Decay and Effective Structural Horizon]
    {theorem}{spectraldecay}
\label{thm:structural_horizon}
Under the assumptions above, for every \(t\ge0\), $D_t
=
\max_{r\ge2}
|\lambda_r|^t |1-\lambda_r|
\le
(1+\rho)\rho^t.$ Furthermore, for any resolution
\(0<\varepsilon<1+\rho\), we have  $H_\varepsilon
\le
\left\lceil
\frac{\log((1+\rho)/\varepsilon)}
{-\log\rho}
\right\rceil.$
\end{restatable}

\begin{insightbox}
\textbf{Insight.}
The effective structural horizon can be estimated from the graph topology before training. The nonstationary spectral radius \(\rho\) depends only on the transition operator \(\mathbf M\) and can be approximated with a sparse eigensolver. For a chosen structural resolution \(\varepsilon\), Theorem~\ref{thm:structural_horizon} gives a graph-dependent walk-length budget for \(T\). This criterion identifies when additional hops become structurally redundant, but it does not determine the optimal walk length.
\end{insightbox}
Appendices~\ref{sec:app_identifiability}
and~\ref{sec:app_structural_horizon} provide the full identifiability analysis, decay proof, structural-horizon derivation, and conditioning results.

\subsection{Sparse RRWP Aggregation}
\label{subsec:sra_explanations}

The normalized RRWP formulation captures structural influence across multiple
walk lengths, but a direct realization of Eq.~\ref{eq:linear_rrwp} requires
\((\mathbf M^t)_{i,j}\) for every source--target pair \((i,j)\) and every
walk length \(t\). This corresponds to \(T\) pairwise operators of size
\(n\times n\), leading to quadratic growth in the number of nodes. Moreover,
even when the original transition matrix \(\mathbf M\) is sparse, its powers
\(\mathbf M^t\) can become increasingly dense because multi-step paths connect
nodes that are not directly adjacent. \HARMONIA avoids materializing these
dense pairwise operators and instead evaluates their action through repeated
applications of the original sparse transition matrix. To see this, consider a single feature \(k\). Before structural propagation,
its learned response is evaluated independently at every node and collected
into $\mathbf z_k^{(0)}
=
\left[
f_k(x_{1,k}),
f_k(x_{2,k}),
\ldots,
f_k(x_{n,k})
\right]^\top
\in\mathbb R^n.$ These evaluations can be batched across both nodes and features before any
node-to-node interaction is performed. A single multiplication by the
transition matrix then performs one random-walk propagation step:
\begin{equation}
\mathbf M\mathbf z_k^{(0)}
=
\begin{bmatrix}
M_{1,1} & M_{1,2} & \cdots & M_{1,n} \\
M_{2,1} & M_{2,2} & \cdots & M_{2,n} \\
\vdots  & \vdots  & \ddots & \vdots  \\
M_{n,1} & M_{n,2} & \cdots & M_{n,n}
\end{bmatrix}
\begin{bmatrix}
f_k(x_{1,k}) \\
f_k(x_{2,k}) \\
\vdots \\
f_k(x_{n,k})
\end{bmatrix}
=
\begin{bmatrix}
\sum_j M_{1,j}f_k(x_{j,k}) \\
\sum_j M_{2,j}f_k(x_{j,k}) \\
\vdots \\
\sum_j M_{n,j}f_k(x_{j,k})
\end{bmatrix}.
\label{eq:one_step_propagation}
\end{equation}
The \(i\)-th entry is therefore the response received at node \(v_i\) after
one propagation step. Crucially, a sparse representation \textit{stores only the
nonzero entries of \(\mathbf M\) in the memory}, and sparse multiplication \textit{evaluates only
those entries}. Hence, one propagation step processes the observed graph
transitions rather than all \(n^2\) possible source--target pairs. Longer propagation depths are obtained by repeatedly applying the same sparse
operator:
\begin{equation}
\mathbf z_k^{(1)}
=
\mathbf M\mathbf z_k^{(0)},
\qquad
\mathbf z_k^{(2)}
=
\mathbf M\mathbf z_k^{(1)}
=
\mathbf M^2\mathbf z_k^{(0)},
\qquad
\mathbf z_k^{(t)}
=
\mathbf M^t\mathbf z_k^{(0)}.
\label{eq:recursive_propagation}
\end{equation}
Thus, \([\mathbf z_k^{(t)}]_i\) contains the response of feature \(k\)
that reaches node \(v_i\) after \(t\) random-walk steps. Although
\(\mathbf M^t\) itself may become dense, SRA never constructs it. Instead,
\(\mathbf z_k^{(t-1)}\) already contains information accumulated along
shorter paths, and multiplying once more by the sparse matrix \(\mathbf M\)
propagates that information through the next set of observed transitions.
Multi-hop interactions therefore emerge implicitly through repeated sparse
propagation rather than through explicit all-pairs computation. Since all
features share \(\mathbf M\), their channels can be propagated together using
the same sparse matrix multiplication. This recursive view allows the pairwise RRWP computation to be evaluated
without changing its result. Recall that the response of feature
\(k\) at node \(v_i\) is
\([h_i]_k=\sum_{j\in\mathcal V}
\omega_k^{\sharp}(\mathbf p_{i,j})f_k(x_{j,k})\).
Substituting Eq.~\ref{eq:linear_rrwp} and exchanging the order of summation
gives
\begin{equation}
\begin{aligned}
[h_i]_k
&=
\sum_{j\in\mathcal V}
\left(
\sum_{t=0}^{T-1}
\theta_{t,k}(\mathbf M^t)_{i,j}
\right)
f_k(x_{j,k}) =
\sum_{t=0}^{T-1}
\theta_{t,k}
\left(
\sum_{j\in\mathcal V}
(\mathbf M^t)_{i,j}
f_k(x_{j,k})
\right) \\
&=
\sum_{t=0}^{T-1}
\theta_{t,k}
[\mathbf M^t\mathbf z_k^{(0)}]_i
=
\sum_{t=0}^{T-1}
\theta_{t,k}
[\mathbf z_k^{(t)}]_i .
\end{aligned}
\label{eq:sra_reordering}
\end{equation}
The inner source-node aggregation at walk length \(t\) is exactly the
\(i\)-th entry of \(\mathbf M^t\mathbf z_k^{(0)}\). By
Eq.~\ref{eq:recursive_propagation}, this quantity can be obtained recursively
using only the original sparse transition matrix. Hence, neither
\(\mathbf M^t\) nor the corresponding dense source--target tensor needs to be
materialized. We refer to this exact computational reordering as
\emph{Sparse RRWP Aggregation} (SRA). For the following result, let \(|E|\) denote the number of nonzero
graph transitions and \(d\) the number of feature channels.
\begin{restatable}[Complexity of Sparse RRWP Aggregation]
    {proposition}{sraexact}
\label{prop:sra_exactness}
Explicit pairwise materialization requires
\(O(Tn^2)\) memory and \(O(Tn^2d)\) computation, whereas SRA evaluates the
same representation in \(O(T|E|d)\) computation.
\end{restatable}
\begin{insightbox}
\textbf{Insight.}
SRA computes the exact RRWP aggregation in
\(O(T|E|d)\) time by replacing dense pairwise operators with sparse
propagation. For ergodic reversible random walks,
Theorem~\ref{thm:structural_horizon} provides the graph-dependent choice
\(T=H_\varepsilon+1\), yielding
\(O\!\left(
|E|d
\left[
1+\frac{\log((1+\rho)/\varepsilon)}{-\log\rho}
\right]
\right)\)
computation at resolution \(\varepsilon\). For general directed graphs, this
spectral bound need not hold, while SRA remains exact with complexity
\(O(T|E|d)\).
\end{insightbox}

The full equivalence proof and the forward- and training-memory analyses are
provided in Appendix~\ref{sec:app_sra}. The linear RRWP formulation also preserves the exact additive interpretability
of \HARMONIA. In particular, each prediction can be decomposed into
feature-, source-node-, and walk-length-specific contributions without an
auxiliary explanation model. We provide the complete decomposition and its
aggregation across explanatory dimensions in
Appendix~\ref{subsec:app_exact_explanations}.

\section{Experiments}
\label{sec:experiments}

We evaluate \HARMONIA on real-world node-classification benchmarks to assess predictive performance and scalability. Controlled synthetic experiments test whether it recovers the feature and structural mechanisms used to generate the targets. We also examine feature removal, variation across random seeds, and how expert routing organizes features. Appendix~\ref{sec:app_model_exp_details} provides additional graph-classification experiments and implementation details, including pseudo code.

\subsection{Real-World Benchmarks Accuracy and Scalability}
\label{sec:real_world}

\begin{table*}[htp]
\centering
\small
\setlength{\tabcolsep}{5pt}
\resizebox{\textwidth}{!}{%
\begin{tabular}{llrrrrrr}
\toprule
Type & Method & Cora & CiteSeer & Tolokers & ogbn-arxiv & IGB-Small & ogbn-products \\
\midrule
\multirow{4}{*}{Black-box}
& GCN & $\mathbf{81.23 \pm 1.1}$ & $71.20 \pm 1.7$ & $83.47 \pm 0.7$
    & $71.74 \pm 0.3$ & $70.46 \pm 1.17$ & $75.64 \pm 0.3$ \\
& GAT & $80.32 \pm 2.3$ & $70.26 \pm 2.3$ & $83.76 \pm 1.0$
    & $71.95 \pm 0.6$ & $70.93 \pm 2.15$ & $\mathbf{79.45 \pm 0.5}$ \\
& GraphSAGE & $79.94 \pm 3.4$ & $75.90 \pm 5.0$ & $82.58 \pm 0.5$
    & $71.49 \pm 0.2$ & $\mathbf{75.49 \pm 1.08}$ & $75.63 \pm 0.3$ \\
& Graph Transformer & $80.70 \pm 0.5$ & $\mathbf{76.00 \pm 0.9}$ & $83.27 \pm 0.9$
    & $70.13 \pm 0.5$ & $70.73 \pm 1.26$ & $74.74 \pm 0.5$ \\
\midrule
\multirow{4}{*}{Interpretable}
& NAM & $51.35 \pm 2.3$ & $55.40 \pm 1.9$ & $62.48 \pm 2.1$
    & $56.12 \pm 3.4$ & OOM & OOM \\
& GP-NAM & $59.96 \pm 3.2$ & $60.30 \pm 3.9$ & $83.12 \pm 0.4$
    & $62.35 \pm 4.2$ & $55.92 \pm 3.32$ & $60.13 \pm 3.9$ \\
& GNAN & $72.89 \pm 5.1$ & $65.23 \pm 3.7$ & $\underline{\mathbf{84.47 \pm 0.8}}$
    & $69.56 \pm 0.9$ & OOM & OOM \\
& G-NAMRFF & $77.84 \pm 1.7$ & $69.45 \pm 2.5$ & $82.88 \pm 0.3$
    & $70.02 \pm 3.9$ & $63.10 \pm 1.26$ & $72.13 \pm 0.4$ \\
\midrule
Interpretable
& \textbf{\HARMONIA} & $\underline{79.60 \pm 1.3}$ & $\underline{72.80 \pm 1.1}$
    & $84.18 \pm 0.6$ & $\underline{\mathbf{80.14 \pm 2.3}}$
    & $\underline{68.93 \pm 1.08}$ & $\underline{78.14 \pm 0.7}$ \\
\bottomrule
\end{tabular}}
\caption{Node-classification accuracy (\%).
Methods are categorized as black-box or interpretable-by-design.
Bold denotes the highest mean among all compared methods, underlining
denotes the highest mean among interpretable methods.
OOM(Out Of Memory)}
\label{tab:node-classification}
\end{table*}

\paragraph{Datasets and baselines.}
We use six node-classification benchmarks: Cora and CiteSeer~\citep{sen2008collective}, Tolokers~\citep{platonov2023critical}, ogbn-arxiv and ogbn-products~\citep{hu2020open}, and IGB-Small~\citep{khatua2023igb}. They ranging from a few thousand to nearly 2.5 million nodes. We compare \HARMONIA with black-box GNNs, including GCN~\citep{kipf2016semi}, GAT~\citep{velickovic2018graph}, GraphSAGE~\citep{hamilton2017inductive}, and Graph Transformer~\citep{ma2023GraphInductiveBiases}, as well as interpretable-by-design baselines NAM~\citep{agarwal2021neural}, GP-NAM~\citep{zhang2024gaussian}, GNAN~\citep{bechler2024intelligible}, and G-NAMRFF~\citep{reddy2025interpretable}.

\begin{wrapfigure}{r}{0.35\textwidth}
    \centering
    \vspace{-10pt}
    \includegraphics[width=\linewidth]{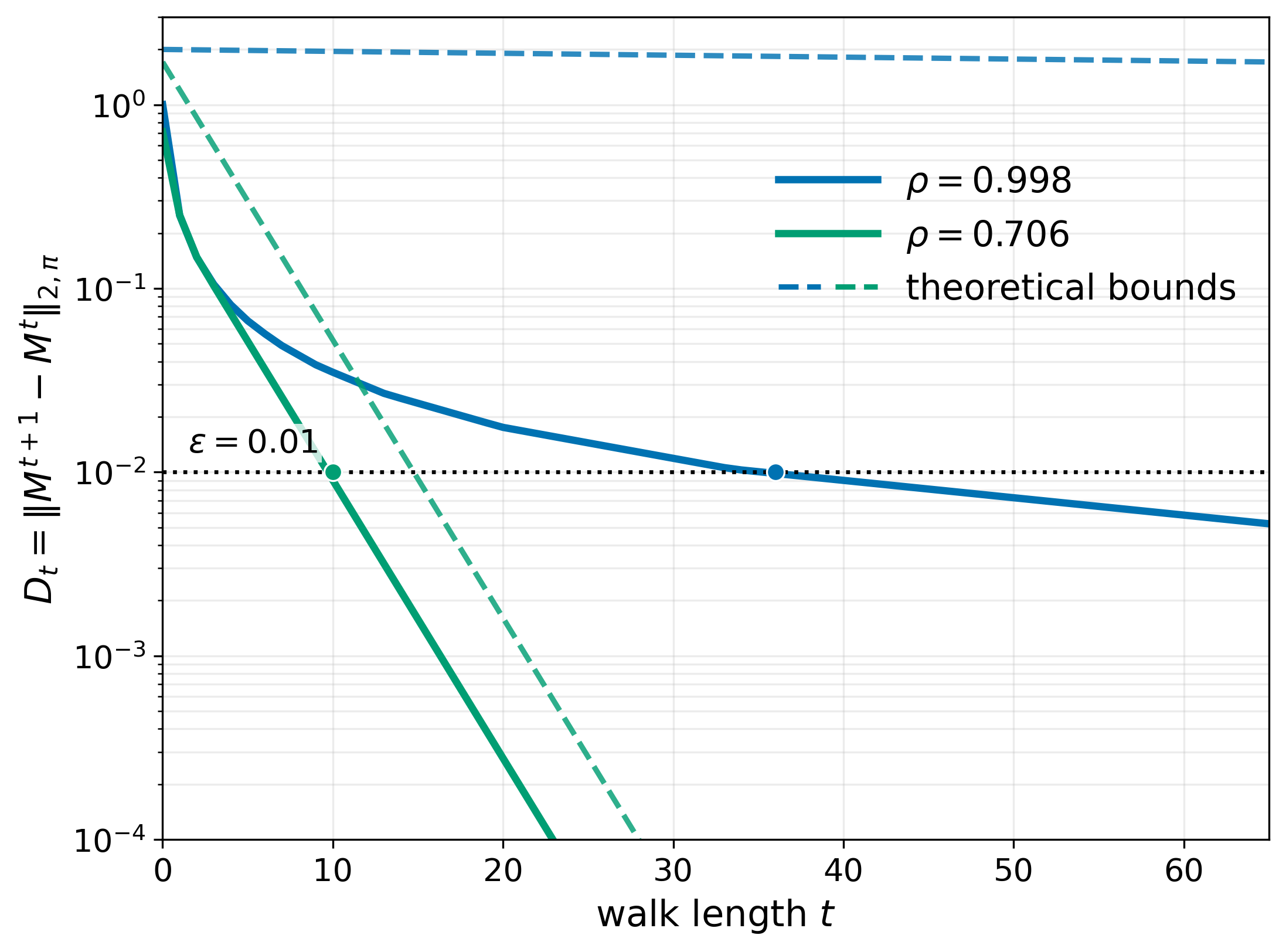}
    \vspace{-18pt}
    \caption{Hop distinguishability decays with increasing walk length.}
    \label{fig:effective_structural_horizon_analysis}
    \vspace{-10pt}
\end{wrapfigure}

\paragraph{Predictive performance and scalability.}
Table~\ref{tab:node-classification} shows that \HARMONIA outperforms existing
interpretable baselines on five of six datasets and scales to large graphs
where some baselines run out of memory. Against black-box models, it achieves
the best overall accuracy on Tolokers and ogbn-arxiv while remaining competitive
on Cora and ogbn-products. We also examine how hop distinguishability decays to guide the propagation budget. 
Figure~\ref{fig:effective_structural_horizon_analysis} supports choosing
$T=H_\varepsilon+1$, beyond this depth, consecutive propagation operators differ by at most $\varepsilon$. This avoids prescribing a large $T$
with cost $O(T|\mathcal E|d)$. Under the assumptions of
Theorem~\ref{thm:structural_horizon}, the resulting SRA computation
is bounded by
$O\!\left(
|\mathcal E|d
\left[1+\frac{\log((1+\rho)/\varepsilon)}{-\log\rho}\right]
\right).$
Together, the graph-dependent horizon and sparse propagation support
efficient structural modeling on large graphs.

\begin{wrapfigure}{r}{0.45\textwidth}
    \centering
    \vspace{-10pt}

    \captionsetup[subfigure]{skip=0pt}

    \begin{subfigure}[t]{0.48\linewidth}
        \centering
        \includegraphics[width=\linewidth]{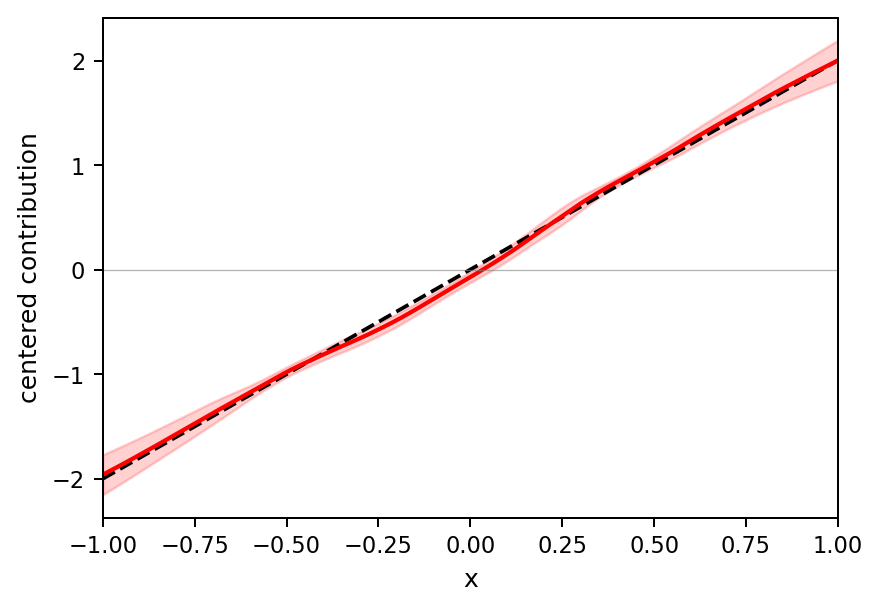}
        \caption{$f_k^{\star}=2x$}
    \end{subfigure}
    \hfill
    \begin{subfigure}[t]{0.48\linewidth}
        \centering
        \includegraphics[width=\linewidth]{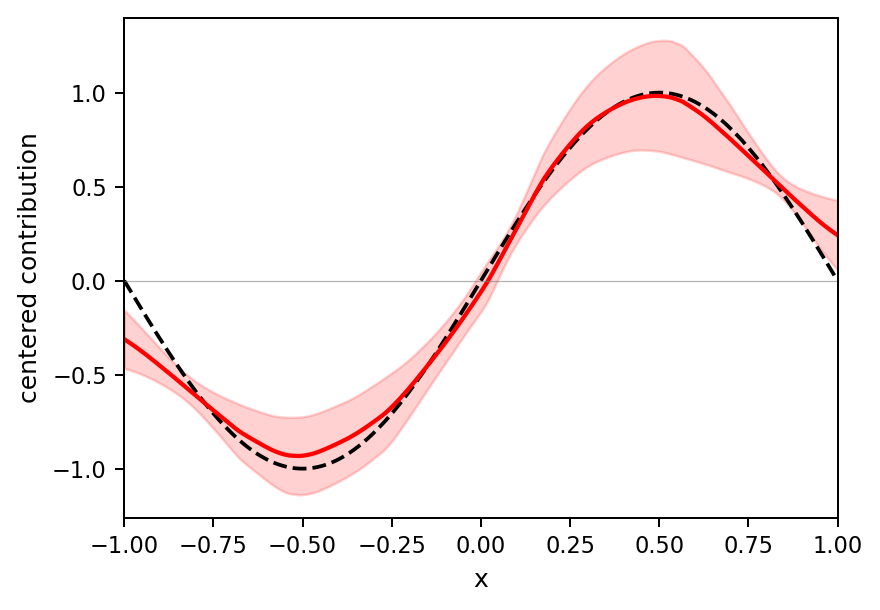}
        \caption{$f_k^{\star}=\sin(\pi x)$}
    \end{subfigure}

    \vspace{0mm}

    \begin{subfigure}[t]{0.48\linewidth}
        \centering
        \includegraphics[width=\linewidth]{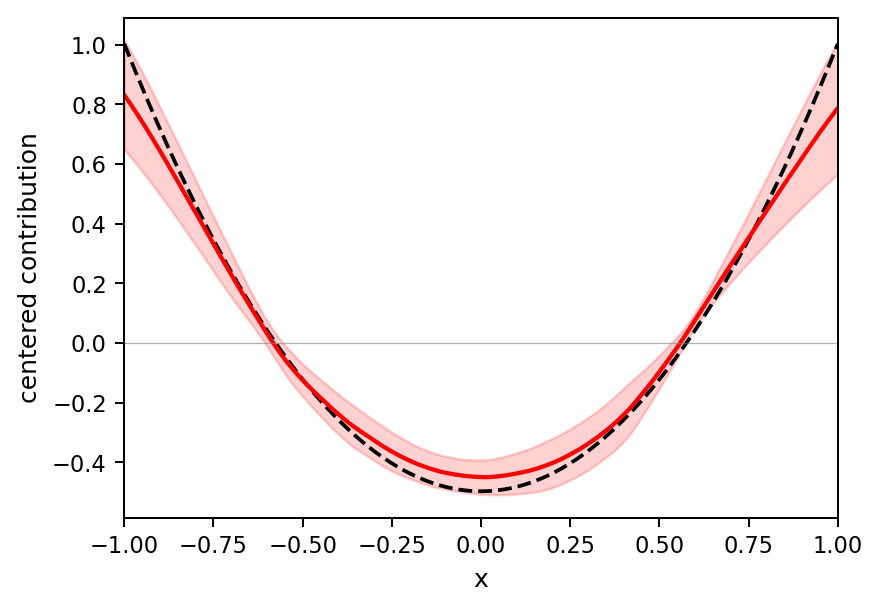}
        \caption{$f_k^{\star}=1.5(x^2-1)$}
    \end{subfigure}
    \hfill
    \begin{subfigure}[t]{0.48\linewidth}
        \centering
        \includegraphics[width=\linewidth]{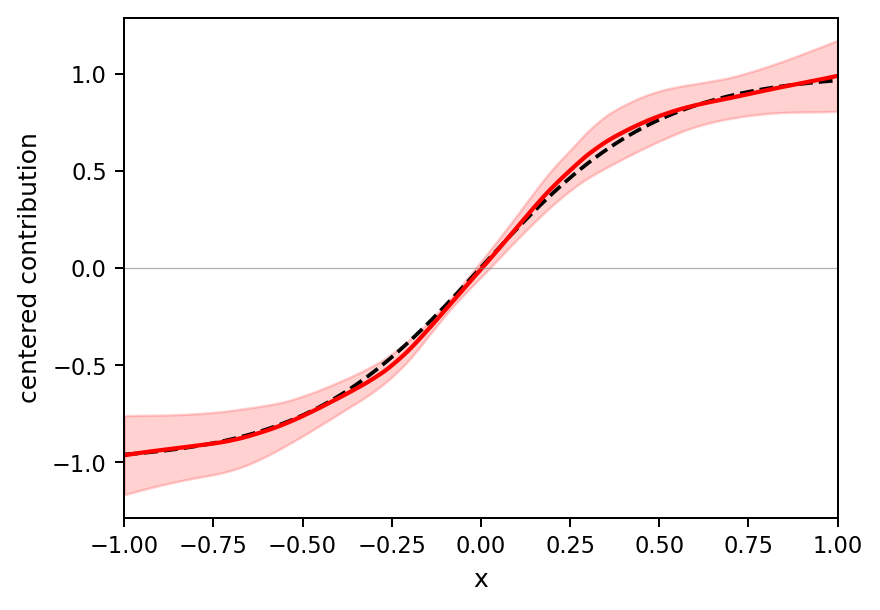}
        \caption{$f_k^{\star}=\tanh(2x)$}
    \end{subfigure}

    \vspace{-5pt}

    \caption{HARMONIA univariate response recovery for 4 ground-truth
functions.}

    \label{fig:synthetic_function_recovery}

    \vspace{-10pt}
\end{wrapfigure}

\subsection{Explanation Faithfulness and Completeness}
\label{sec:synthetic}

We use synthetic graphs with 20 continuous node features sampled in
$[-1,1]$. Four signal features generate labels through linear, quadratic,
sinusoidal, and saturating responses with feature-specific hop profiles.
The remaining 16 features are sampled within the same range but do not
contribute to label generation. Construction and evaluation details are
provided in Appendices~\ref{app:synthetic_setup}--\ref{app:synthetic_edges}.

We compare additive responses, model-intervention scores, and post-hoc
explanation scores against known feature effects and oracle logit changes
under edge deletion or training-mean replacement of a source node's features. In addition to interpretable-by-design graph additive baselines, we include two post-hoc GNN explainers, PGExplainer~\citep{luo2020parameterized} and GOAt~\citep{lu2024goat}. Precision measures top-set overlap and NDCG evaluates relevance-weighted
ranking, using four features/nodes or the top $20\%$ of edges, rounded up;
Effective and Signed NRMSE measure errors in response curves and signed
node effects, respectively.
Table~\ref{tab:synthetic_recovery} shows that \HARMONIA matches GNAN's
perfect feature selection while reducing response error ($0.193$ versus
$0.435$), achieves the highest edge NDCG and precision ($0.9425$ and
$0.8162$), and leads node recovery in both ranking and signed-effect
accuracy ($0.9658$ NDCG@4 and $0.2992$ Signed NRMSE).
Different combinations of feature responses and structural effects can yield similar outputs, so we test whether the learned responses recover each underlying feature signal. Figure~\ref{fig:synthetic_function_recovery} shows that \HARMONIA's centered responses closely track the four ground-truth functions, including their linear, quadratic, sinusoidal, and saturating shapes. The close match supports faithful component recovery and empirical identifiability of the model.
\paragraph{Expert specialization.}
Figure~\ref{fig:expert_identifiability} shows UMAP clusters of Cora features for different numbers of MoNB experts and helps guide the choice of expert count. An underspecified expert count limits specialization by forcing heterogeneous features into few expert spaces. An overspecified count fragments groups, weakening separation and identifiability. Appendix~\ref{subsec:app_routing_tradeoff} examines this trade-off.
\begin{table*}[htp]
\centering
\small
\setlength{\tabcolsep}{3pt}
\resizebox{\textwidth}{!}{%
\begin{tabular}{lcccccccc}
\toprule
& \multicolumn{3}{c}{Feature and response recovery}
& \multicolumn{2}{c}{Edge recovery}
& \multicolumn{3}{c}{Node recovery}\\
\cmidrule(lr){2-4}\cmidrule(lr){5-6}\cmidrule(lr){7-9}
Method
& \shortstack{Precision\\@4 $\uparrow$} 
& \shortstack{NDCG\\@4 $\uparrow$}
& \shortstack{Effective\\NRMSE $\downarrow$}
& \shortstack{NDCG\\@20\% $\uparrow$} 
& \shortstack{Precision\\@20\% $\uparrow$}
& \shortstack{Precision\\@4 $\uparrow$} 
& \shortstack{NDCG\\@4 $\uparrow$}
& \shortstack{Signed\\NRMSE $\downarrow$}\\
\midrule
\HARMONIA
& $\mathbf{1.000}$ & $\mathbf{1.000}$ & $\mathbf{0.193}$
& $\mathbf{0.9425}$ & $\mathbf{0.8162}$
& $\mathbf{0.8487}$ & $\mathbf{0.9658}$ & $\mathbf{0.2992}$\\
GNAN
& $\mathbf{1.000}$ & $0.996$ & $0.435$
& $0.6305$ & $0.5005$
& $0.7057$ & $0.8861$ & $0.5853$\\
G-NAMRFF
& $0.783$ & $0.897$ & $0.735$
& $0.7589$ & $0.6430$
& $0.6413$ & $0.7975$ & $0.8433$\\
GCN + GNNExplainer
& $0.850$ & $0.934$ & N/A
& $0.6259$ & $0.5708$
& $0.5312$ & $0.6914$ & N/A\\
GCN + GOAt
& $0.950$ & $0.967$ & N/A
& $0.7790$ & $0.6597$
& $0.6732$ & $0.8294$ & N/A\\
\bottomrule
\end{tabular}}
\caption{Synthetic recovery means. Bold marks the best
mean, including ties. Full results and available standard deviations are
in Tables~\ref{tab:synthetic_feature_full},
\ref{tab:synthetic_edge_full}, and~\ref{tab:synthetic_node_full}.}
\label{tab:synthetic_recovery}
\end{table*}

\begin{figure}[htp]
    \centering
    \begin{minipage}[t]{0.24\textwidth}
        \centering
        \includegraphics[width=\linewidth]{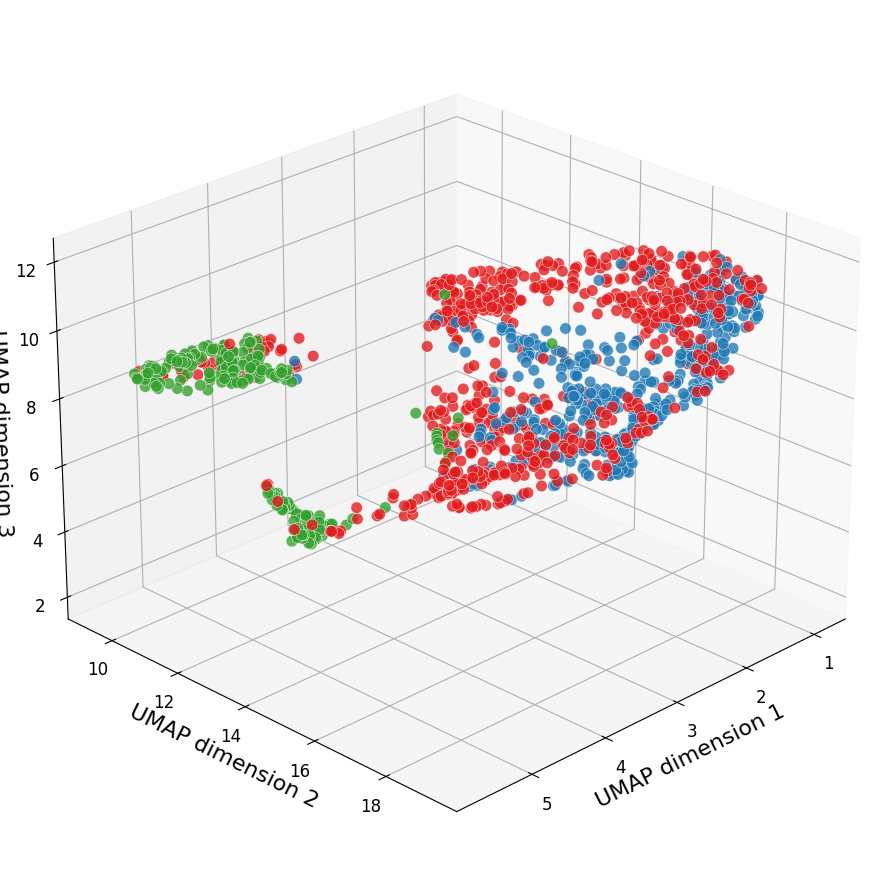}
        \vspace{-5pt}
        \subcaption{3 Experts}
        \label{fig:fig1}
    \end{minipage}
    \hfill
    \begin{minipage}[t]{0.24\textwidth}
        \centering
        \includegraphics[width=\linewidth]{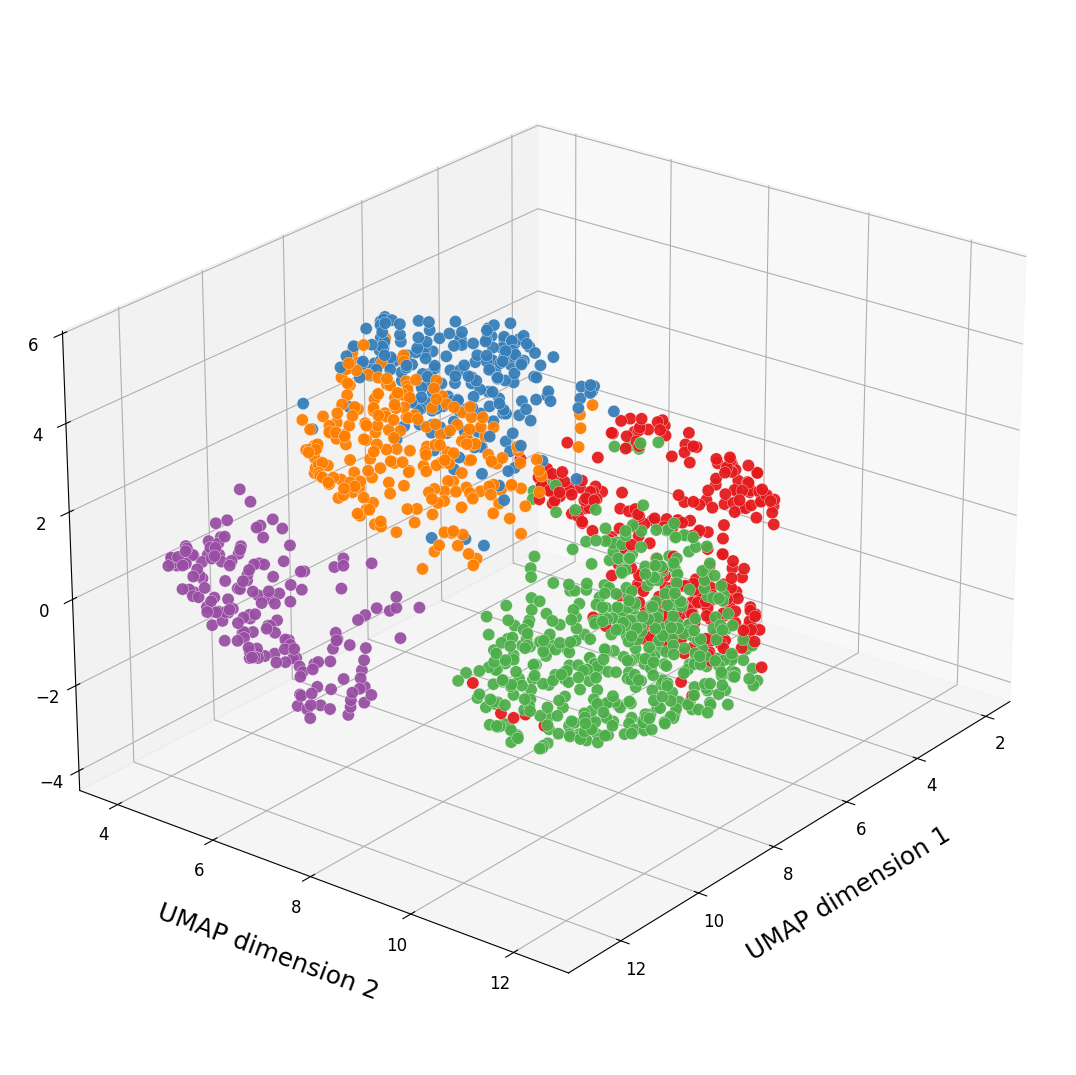}
        \vspace{-5pt}
        \subcaption{5 Experts}
        \label{fig:fig2}
    \end{minipage}
    \hfill
    \begin{minipage}[t]{0.24\textwidth}
        \centering
        \includegraphics[width=\linewidth]{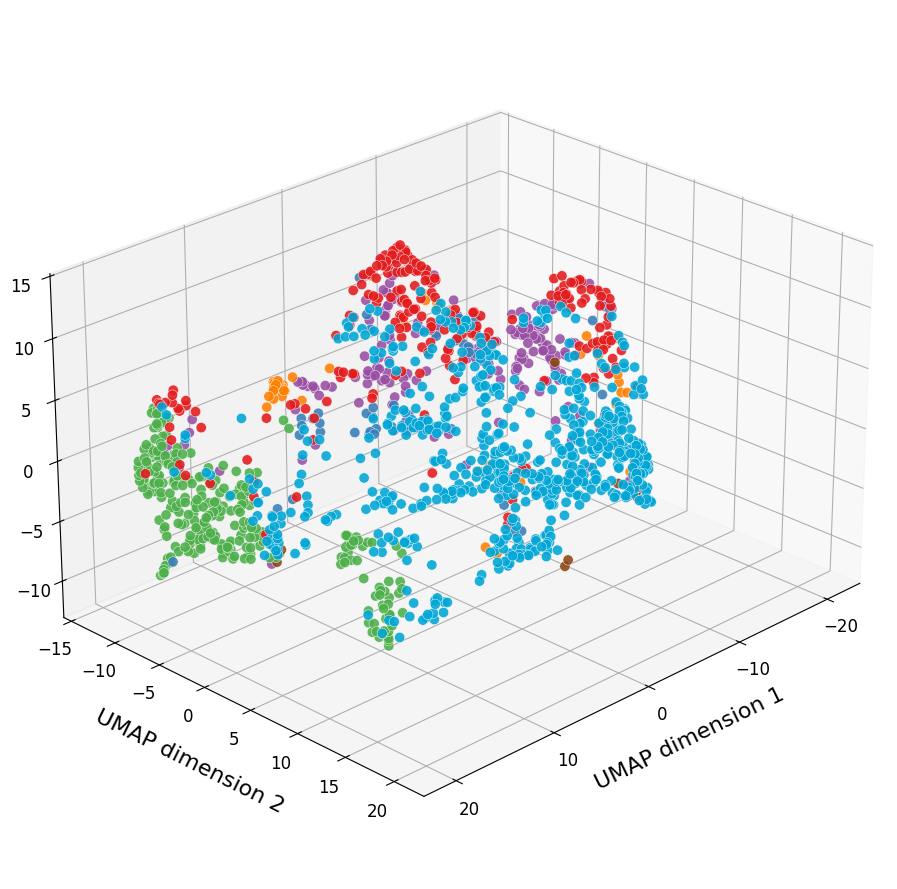}
        \vspace{-5pt}
        \subcaption{7 Experts}
        \label{fig:fig3}
    \end{minipage}
    \hfill
    \begin{minipage}[t]{0.24\textwidth}
        \centering
        \includegraphics[width=\linewidth]{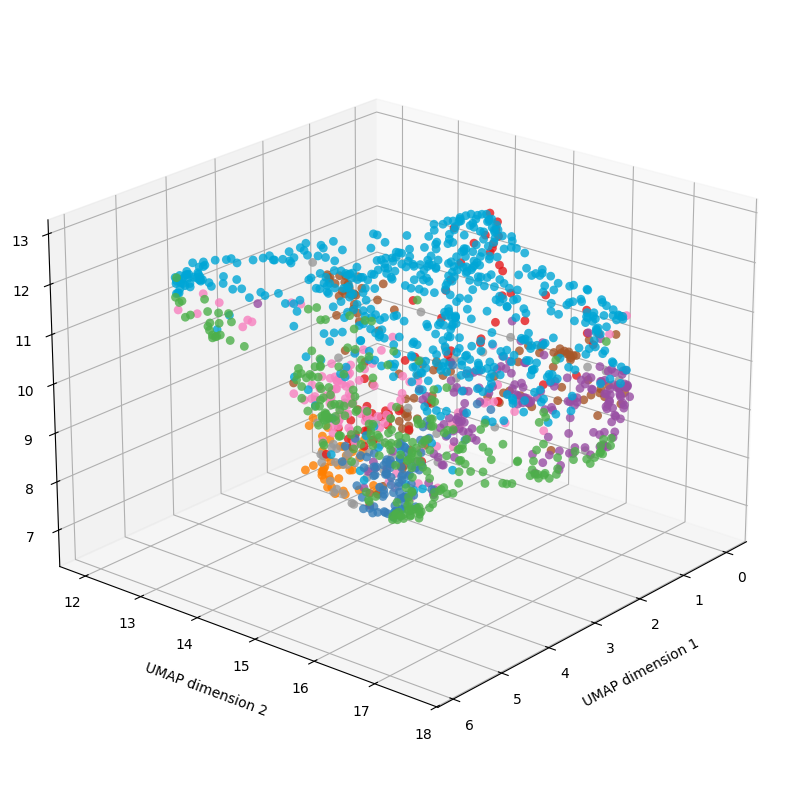}
        \vspace{-5pt}
        \subcaption{9 Experts}
        \label{fig:fig4}
    \end{minipage}
    \par
    \vspace{-25pt}

    \begin{minipage}{0.9\textwidth}
        \centering
        \includegraphics[width=\linewidth]{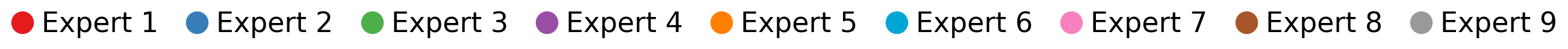}
    \end{minipage}

    \vspace{7pt}

    \caption{UMAP visualizations of Cora features with different numbers of MoNB experts. Each point represents a feature, and its color indicates the expert assignment.}
    \label{fig:expert_identifiability}
\end{figure}

\section{Related Work}

\paragraph{Interpretable additive models.}
Generalized Additive Models (GAMs)~\citep{hastie2017generalized} express predictions as sums of feature-wise shape functions. NAMs~\citep{agarwal2021neural} learn these functions with neural networks, and NBMs~\citep{radenovic2022neural} reduce parameter costs by representing feature responses through shared neural bases. GrAMs extend the additive structure to relational data. GNAN~\citep{bechler2024intelligible} models structural influence using shortest-path distance; G-NAMRFF~\citep{reddy2025interpretable} uses random Fourier features for parameter-efficient feature modeling. GMAN~\citep{bechler2025gman} and SUPERMAN~\citep{zerio2026superman} extend GNAN to sets of sparse temporal or heterogeneous graphs. They also allow feature or graph grouping, trading fine-grained interpretability for greater expressivity. \HARMONIA combines specialized shared neural bases with multi-scale structural modeling through RRWP.

\paragraph{Post-hoc GNN explainers.}
Many widely used GNN explanation methods operate post hoc on a trained
predictor. GNNExplainer \citep{ying2019gnnexplainer} and PGExplainer
\citep{luo2020parameterized} identify influential subgraphs and features,
GraphLIME \citep{huang2022graphlime} provides local feature attributions, and
GRAPHTRAIL \citep{armgaan2024graphtrail} translates learned representations
into human-readable rules. Surveys are provided in
\citep{yuan2022explainability,kakkad2023survey}. Unlike these auxiliary
explanation mechanisms, interpretable-by-design models such as \HARMONIA expose
feature and structural contributions directly through the predictive
function.

\section{Conclusion and Future Work}
\HARMONIA is an interpretable graph additive framework that scales with feature dimension and graph size. MoNB models feature responses with specialized bases, while SRA computes multi-hop propagation without dense RRWP tensors and preserves fine-grained explanations. Future work will extend \HARMONIA to hypergraphs, improve the identifiability and stability of its feature and structural components, and develop principled mechanisms for selecting the expert count to avoid overspecified expert configurations.

\section*{AI Use Statement}
Generative AI tools were used solely to assist with code debugging and to improve the grammatical clarity of the manuscript. The experimental design, analysis, results, and conclusions were developed independently by the authors without generative AI assistance. All scientific content and interpretations presented in this work remain the responsibility of the authors.
\bibliography{iclr2027_conference}
\bibliographystyle{iclr2027_conference}

\appendix
\newpage

\begin{center}
    {\LARGE\bfseries Supplementary Material for\\
    Interpretable and Efficient Graph Additive Model with Mixture of Neural Basis}
\end{center}

\noindent\rule{\textwidth}{1.5pt}

\vspace{0.75em}

\textbf{This supplementary material provides complete theoretical derivations, mathematical analysis, implementation details, full experimental
settings, theory-driven evaluations, scalability studies, and additional
interpretability results for \HARMONIA. It is organized as follows:}

\vspace{0.5em}

\begin{itemize}[leftmargin=*, label={}]

\item \phantomsection
\textbf{\ref{sec:notation}. \nameref{sec:notation}}
\dotfill \pageref{sec:notation}


\item \phantomsection
\textbf{\ref{sec:app_monb_theory}. \nameref{sec:app_monb_theory}}
\dotfill \pageref{sec:app_monb_theory}

\begin{itemize}[leftmargin=2em, label={}]

    \item \ref{subsec:app_notation}. \nameref{subsec:app_notation}
    \dotfill \pageref{subsec:app_notation}

    \item \ref{subsec:app_global_basis}. \nameref{subsec:app_global_basis}
    \dotfill \pageref{subsec:app_global_basis}

    \item \ref{subsec:app_specialization_gap}. \nameref{subsec:app_specialization_gap}
    \dotfill \pageref{subsec:app_specialization_gap}

    \item \ref{subsec:app_routing_tradeoff}. \nameref{subsec:app_routing_tradeoff}
    \dotfill \pageref{subsec:app_routing_tradeoff}

    \item \ref{subsec:app_statistical_efficiency}. \nameref{subsec:app_statistical_efficiency}
    \dotfill \pageref{subsec:app_statistical_efficiency}

\end{itemize}


\item \phantomsection
\textbf{\ref{sec:app_identifiability}. \nameref{sec:app_identifiability}}
\dotfill \pageref{sec:app_identifiability}

\begin{itemize}[leftmargin=2em, label={}]

    \item \ref{subsec:app_normalization_identifiability}. 
    \nameref{subsec:app_normalization_identifiability}
    \dotfill \pageref{subsec:app_normalization_identifiability}

    \item \ref{subsec:app_hop_identifiability}. 
    \nameref{subsec:app_hop_identifiability}
    \dotfill \pageref{subsec:app_hop_identifiability}

    \item \ref{subsec:app_approx_identifiability}. 
    \nameref{subsec:app_approx_identifiability}
    \dotfill \pageref{subsec:app_approx_identifiability}

\end{itemize}


\item \phantomsection
\textbf{\ref{sec:app_structural_horizon}. \nameref{sec:app_structural_horizon}}
\dotfill \pageref{sec:app_structural_horizon}

\begin{itemize}[leftmargin=2em, label={}]

    \item \ref{subsec:app_hop_decay}. 
    \nameref{subsec:app_hop_decay}
    \dotfill \pageref{subsec:app_hop_decay}
    
    \item \ref{subsec:app_structural_horizon}. 
    \nameref{subsec:app_structural_horizon}
    \dotfill \pageref{subsec:app_structural_horizon}

\end{itemize}


\item \phantomsection
\textbf{\ref{sec:app_rrwp_expressivity}. \nameref{sec:app_rrwp_expressivity}}
\dotfill \pageref{sec:app_rrwp_expressivity}

\begin{itemize}[leftmargin=2em, label={}]

    \item \ref{subsec:app_rrwp_spectral}. 
    \nameref{subsec:app_rrwp_spectral}
    \dotfill \pageref{subsec:app_rrwp_spectral}

    \item \ref{subsec:app_rrwp_tradeoff}. 
    \nameref{subsec:app_rrwp_tradeoff}
    \dotfill \pageref{subsec:app_rrwp_tradeoff}

    \item \ref{subsec:app_exact_explanations}. 
    \nameref{subsec:app_exact_explanations}
    \dotfill \pageref{subsec:app_exact_explanations}

\end{itemize}


\item \phantomsection
\textbf{\ref{sec:app_sra}. \nameref{sec:app_sra}}
\dotfill \pageref{sec:app_sra}

\begin{itemize}[leftmargin=2em, label={}]

    \item \ref{subsec:app_sra_equivalence}. 
    \nameref{subsec:app_sra_equivalence}
    \dotfill \pageref{subsec:app_sra_equivalence}

    \item \ref{subsec:app_sra_complexity}. 
    \nameref{subsec:app_sra_complexity}
    \dotfill \pageref{subsec:app_sra_complexity}

\end{itemize}


\item \phantomsection
\textbf{\ref{sec:app_model_exp_details}. \nameref{sec:app_model_exp_details}}
\dotfill \pageref{sec:app_model_exp_details}

\begin{itemize}[leftmargin=2em, label={}]

    \item \ref{subsec:hardware_setting}. 
    \nameref{subsec:hardware_setting}
    \dotfill \pageref{subsec:hardware_setting}

    \item \ref{subsec:app_datasets_splits}. 
    \nameref{subsec:app_datasets_splits}
    \dotfill \pageref{subsec:app_datasets_splits}

    \item \ref{subsec:app_baselines}. 
    \nameref{subsec:app_baselines}
    \dotfill \pageref{subsec:app_baselines}

    \item \ref{subsec:app_hyperparameters}. 
    \nameref{subsec:app_hyperparameters}
    \dotfill \pageref{subsec:app_hyperparameters}

    \item \ref{subsec:app_monb_routing}. 
    \nameref{subsec:app_monb_routing}
    \dotfill \pageref{subsec:app_monb_routing}

    \item \ref{subsec:app_prediction_readout}. 
    \nameref{subsec:app_prediction_readout}
    \dotfill \pageref{subsec:app_prediction_readout}

    \item \ref{subsec:app_algorithm_training}. 
    \nameref{subsec:app_algorithm_training}
    \dotfill \pageref{subsec:app_algorithm_training}

    \item \ref{subsec:app_compute_reproducibility}. 
    \nameref{subsec:app_compute_reproducibility}
    \dotfill \pageref{subsec:app_compute_reproducibility}

    \item \ref{subsec:graph_classification}. 
    \nameref{subsec:graph_classification}
    \dotfill \pageref{subsec:graph_classification}

    \item \ref{app:synthetic_setup}. 
    \nameref{app:synthetic_setup}
    \dotfill \pageref{app:synthetic_setup}

    \item \ref{app:recovery_metrics}. 
    \nameref{app:recovery_metrics}
    \dotfill \pageref{app:recovery_metrics}

    \item \ref{app:synthetic_edges}. 
    \nameref{app:synthetic_edges}
    \dotfill \pageref{app:synthetic_edges}

    \item \ref{app:synthetic_nodes}. 
    \nameref{app:synthetic_nodes}
    \dotfill \pageref{app:synthetic_nodes}

    \item \ref{app:synthetic_function_recovery}. 
    \nameref{app:synthetic_function_recovery}
    \dotfill \pageref{app:synthetic_function_recovery}

\end{itemize}


\item \phantomsection
\textbf{\ref{sec:app_additional_analysis}. \nameref{sec:app_additional_analysis}}
\dotfill \pageref{sec:app_additional_analysis}

\begin{itemize}[leftmargin=2em, label={}]

    \item \ref{subsec:app_exp_specialization}. 
    \nameref{subsec:app_exp_specialization}
    \dotfill \pageref{subsec:app_exp_specialization}
    
    \item \ref{subsec:app_exp_num_experts_bases}. 
    \nameref{subsec:app_exp_num_experts_bases}
    \dotfill \pageref{subsec:app_exp_num_experts_bases}

    \item \ref{subsec:app_exp_routing}. 
    \nameref{subsec:app_exp_routing}
    \dotfill \pageref{subsec:app_exp_routing}

    \item \ref{subsec:app_exp_hop_identifiability}. 
    \nameref{subsec:app_exp_hop_identifiability}
    \dotfill \pageref{subsec:app_exp_hop_identifiability}

    \item \ref{subsec:app_exp_horizon}. 
    \nameref{subsec:app_exp_horizon}
    \dotfill \pageref{subsec:app_exp_horizon}

    \item \ref{subsec:app_exp_rrwp_ablation}. 
    \nameref{subsec:app_exp_rrwp_ablation}
    \dotfill \pageref{subsec:app_exp_rrwp_ablation}
    
    \item \ref{subsec:app_explanation_stability}. 
    \nameref{subsec:app_explanation_stability}
    \dotfill \pageref{subsec:app_explanation_stability}

    \item \ref{subsec:app_parameter_complexity}. 
    \nameref{subsec:app_parameter_complexity}
    \dotfill \pageref{subsec:app_parameter_complexity}


    \item \ref{subsec:app_feature_shapes}. 
    \nameref{subsec:app_feature_shapes}
    \dotfill \pageref{subsec:app_feature_shapes}



\end{itemize}

\end{itemize}

\newpage

\newcommand{\firstgrouprow}[1]{%
    \addlinespace[2pt]
    \multicolumn{2}{l}{\textbf{#1}}\\[-2pt]
    \addlinespace[1pt]
}

\newcommand{\grouprow}[1]{%
    \midrule
    \addlinespace[3pt]
    \multicolumn{2}{l}{\textbf{#1}}\\[-2pt]
    \addlinespace[1pt]
}

\section{TABLE OF NOTATION}
\label{sec:notation}

Table~\ref{tab:notation} summarizes the main notation used throughout the paper.

{
\small
\setlength{\tabcolsep}{6pt}
\renewcommand{\arraystretch}{1.12}
\setlength{\LTleft}{0pt}
\setlength{\LTright}{0pt}

\begin{longtable}{
    >{\raggedright\arraybackslash}p{0.25\textwidth}
    >{\raggedright\arraybackslash}p{0.69\textwidth}
}

\caption{Summary of the main notation used in \HARMONIA.}
\label{tab:notation}
\\

\toprule
\textbf{Notation} & \textbf{Meaning} \\
\midrule
\endfirsthead

\multicolumn{2}{l}{\small\textit{Table~\ref{tab:notation} continued.}}\\
\toprule
\textbf{Notation} & \textbf{Meaning} \\
\midrule
\endhead

\midrule
\multicolumn{2}{r}{\small\textit{Continued on next page}}\\
\endfoot

\bottomrule
\endlastfoot


\firstgrouprow{Graph and prediction}

$\mathcal G=(\mathcal V,\mathcal E,\mathbf X),\ n,\ d$
& Input graph, number of nodes, and number of features. \\

$v_i,\ x_{i,k}$
& Node $i$ and its $k$-th feature. \\

$f_k,\ \mathbf h_i$
& Feature-$k$ response function and propagated representation of node $i$. \\

$\ell_{i,c},\ \beta_c,\ w_{k,c}$
& Class-$c$ logit, bias, and output coefficient for feature $k$. \\


\grouprow{Mixture of Neural Bases (MoNB)}

$B,\ C,\ m$
& Number of bases per expert, number of experts, and number of selected experts per feature. \\

$\boldsymbol{\psi}_{\varphi_c},\ \psi_b^{(c)},\ \varphi_c$
& Basis-output vector of expert $c$, its $b$-th basis function, and expert parameters. \\

$\mathbf e_k,\ \mathbf Q_k,\ \mathcal I_k$
& Feature embedding, router scores, and top-$m$ selected expert set. \\

$\widetilde{\pi}_{k,c},\ \mathbf a_k,\ \widetilde{\psi}_{k,b}$
& Routing weight, feature-specific coefficients, and effective routed basis. \\

$\mathcal S_B,\ \mathcal S_c,\ \mathcal S_{\mathrm{all}}$
& Global NBM space, expert-$c$ space, and joint span of all expert spaces. \\


\grouprow{Functional specialization and routing}

$\mathcal H=L_2(\mathcal X,\nu),\ f_k^\star$
& Functional Hilbert space and target response of feature $k$. \\

$\Pi_{\mathcal S}$
& Orthogonal projection onto functional subspace $\mathcal S$. \\

$\mathcal C,\ (\lambda_j,\phi_j)$
& Global feature-function second-moment operator and its eigenpairs. \\

$E_{\mathrm{global}}(B),\ E_{\mathrm{mix}}^\star(B,C)$
& Optimal global and oracle specialized approximation errors. \\

$\Gamma_{B,C}$
& Specialization gap,
$E_{\mathrm{global}}(B)-E_{\mathrm{mix}}^\star(B,C)$. \\

$z_k^\star,\ \widehat z_k,\ e_k(c),\ \Delta_k(c)$
& Oracle and learned assignments, expert-specific error, and routing penalty. \\

$\mathcal R_{\mathrm{route}}(\widehat{\mathbf z})$
& Average excess approximation error induced by learned routing. \\


\grouprow{RRWP structural modeling}

$\mathbf A,\ \mathbf D,\ \mathbf M=\mathbf D^{-1}\mathbf A,\ T$
& Adjacency, degree, random-walk transition matrix, and number of walk orders. \\

$\mathbf p_{i,j},\ \omega_k^\sharp$
& RRWP descriptor and feature-specific topology function. \\

$\alpha_{t,k},\ \theta_{t,k}$
& Unconstrained structural parameter and normalized hop weight. \\

$\mathbf P_k$
& Feature-specific propagation operator
$\sum_{t=0}^{T-1}\theta_{t,k}\mathbf M^t$. \\


\grouprow{Identifiability and spectral analysis}

$\boldsymbol{\theta},\ \Delta^{T-1},\ \mathbf P_{\boldsymbol{\theta}}$
& Normalized hop coefficients, their probability simplex, and induced RRWP operator. \\

$\mu_{\mathbf M},\ q_M$
& Minimal polynomial of $\mathbf M$ and its degree. \\

$\pi,\ \rho$
& Stationary distribution and largest nonstationary eigenvalue magnitude of $\mathbf M$. \\

$D_t,\ H_\varepsilon$
& Consecutive-hop distinguishability and effective structural horizon at resolution $\varepsilon$. \\

$\gamma_T$
& Restricted conditioning constant governing the stability of hop-coefficient recovery. \\


\grouprow{Sparse RRWP Aggregation (SRA)}

$\mathbf z_k^{(t)},\ \mathbf Z^{(t)}$
& Propagated feature-$k$ signal and matrix of all feature channels after $t$ steps. \\

$E_M=\operatorname{nnz}(\mathbf M)$
& Number of nonzero transitions used in the sparse-complexity analysis. \\


\grouprow{Statistical analysis}

$R=\dim(\mathcal S_{\mathrm{all}}),\ r_k$
& Model-wide joint dimension and active specialized dimension for feature $k$. \\

$N,\ \sigma^2,\ \Phi_q$
& Number of observations, noise variance, and design matrix in the coefficient-estimation analysis. \\

\end{longtable}
}
\newpage

\section{FOUNDATIONS OF MIXTURE OF NEURAL BASIS}
\label{sec:app_monb_theory}

\subsection{Notation and Technical Setup}
\label{subsec:app_notation}
We first focus on feature modeling and temporarily ignore graph propagation. We ask a simple question: \emph{what happens when all feature-response functions must share the same basis space?} A standard Neural Basis Model (NBM) uses one shared set of basis functions for all features, whereas MoNB allows different features to use different basis experts. Let \(d\) denote the number of input features, and let \(f_k^\star\), for \(k\in[d]\), denote the target response function of feature \(k\). Since these functions can have very different shapes, a small shared basis may not represent all of them well.

To compare NBM and MoNB, we place all feature-response functions in a common Hilbert space
$\mathcal H=L_2(\mathcal X,\nu)$, where $\mathcal X$ is the feature domain and $\nu$ determines
how different parts of this domain are weighted. The distance
$\|f-g\|_{\mathcal H}^2$ measures how different two response functions are over the feature domain.
A standard NBM learns $B$ shared basis functions $\psi_1,\ldots,\psi_B$, where $B$ is the number
of basis functions used to represent each feature response. These bases define the shared function
space $\mathcal S=\operatorname{span}\{\psi_1,\ldots,\psi_B\}$, whose dimension is at most $B$.
If a target response $f_k^\star$ does not lie in $\mathcal S$, it cannot be represented exactly.
We denote its closest approximation in $\mathcal S$ by $\Pi_{\mathcal S}f_k^\star$, giving the
approximation error $\|f_k^\star-\Pi_{\mathcal S}f_k^\star\|_{\mathcal H}^2$. Since NBM uses the
same space $\mathcal S$ for all $d$ features, its average approximation error is
\begin{equation}
\ E(\mathcal S)
=
\frac{1}{d}
\sum_{k=1}^{d}
\left\|
f_k^\star-\Pi_{\mathcal S}f_k^\star
\right\|_{\mathcal H}^{2}.
\label{eq:app_global_space_error}
\end{equation}
Eq.~\ref{eq:app_global_space_error} measures the approximation error for a particular shared space \(\mathcal S\). To characterize the best possible global basis representation, we minimize this error over all spaces of dimension at most \(B\):
\begin{equation}
\ E_{\mathrm{global}}(B)
:=
\inf_{\substack{\mathcal S\subseteq\mathcal H\\ \dim(\mathcal S)\le B}}
\ E(\mathcal S).
\label{eq:global_error}
\end{equation}
This quantity is the smallest approximation error achievable by any shared \(B\)-dimensional space, and therefore isolates the limitation of global basis sharing from errors caused by a particular basis choice or poor training.

MoNB relaxes the shared-space constraint by learning \(C\) basis
experts, where expert \(c\) spans \(\mathcal
S_c=\operatorname{span}\{\psi_1^{(c)},\ldots,\psi_B^{(c)}\}\).
Different features can therefore use different low-dimensional spaces.
Consider the case in which each feature is assigned to one expert, and
let \(z_k\in[C]\) denote the expert selected for feature \(k\). The
model approximates feature \(k\) within \(\mathcal S_{z_k}\). For a
fixed assignment \(z=(z_1,\ldots,z_d)\), the approximation error is the
average distance between each target function and its assigned expert
space. Optimizing both the expert spaces and the assignments gives the
mixture error:
\begin{equation}
\ E_{\mathrm{mix}}^\star(B,C)
:=
\inf_{\substack{
z\in[C]^d,\;
\mathcal S_1,\ldots,\mathcal S_C\subseteq\mathcal H\\
\dim(\mathcal S_c)\leq B
}}
\frac{1}{d}
\sum_{k=1}^{d}
\left\|
f_k^\star-\Pi_{\mathcal S_{z_k}}f_k^\star
\right\|_{\mathcal H}^{2}.
\label{eq:app_oracle_mixture_error}
\end{equation}
The gap between \(\ E_{\mathrm{global}}(B)\) and \(\ E_{\mathrm{mix}}^\star(B,C)\) measures the
potential benefit of specialization. When this gap is small, one shared basis space is sufficient and the mixture
offers little additional expressive power. A substantially lower mixture error indicates that the feature-response
functions have different functional structures that cannot be represented efficiently in a common low-dimensional
space.
The hard-routing formulation isolates this effect. \HARMONIA uses sparse top-\(m\) routing, which allows each feature
to combine several basis experts. We consider this more general setting in the extension to sparse top-$m$ routing in
Appendix~\ref{subsec:app_routing_tradeoff}. The next
subsection examines the smallest approximation error attainable by a single shared basis space.

\subsection{Irreducible Approximation Error by a Global Basis Space}
\label{subsec:app_global_basis}

Because a global NBM represents every target feature-response function \(f_k^\star\) in the same \(B\)-dimensional functional space, its best-case
  performance depends on whether a single such space can approximate all feature-response functions with small error. This property is independent of
  the particular bases learned and is determined by how the functions \(f_1^\star,\ldots,f_d^\star\) are distributed across functional directions. If the feature-response functions can be well described using only a few common functional directions, then a small shared basis is sufficient. In contrast, if they require many different directions, a shared space with only \(B\) dimensions cannot represent all of them exactly and must incur approximation error.

We describe this geometry through the \emph{feature-function second-moment operator}
$\mathcal C:\mathcal H\rightarrow\mathcal H$. For any functional direction
$g\in\mathcal H$, we define
\begin{equation}
\mathcal Cg
:=
\frac{1}{d}
\sum_{k=1}^{d}
\langle f_k^\star,g\rangle_{\mathcal H} f_k^\star.
\label{eq:app_function_second_moment}
\end{equation}
For each target function $f_k^\star$, the inner product
$\langle f_k^\star, g\rangle_{\mathcal H}$ measures how strongly
$f_k^\star$ aligns with the direction $g$. The operator $\mathcal C$
uses these alignment scores to combine the target functions, so
$\mathcal C g$ summarizes how the collection of feature-response functions
is distributed relative to $g$. Since $\mathcal C$ is constructed from
only $d$ target functions, it has finite rank and is positive semidefinite.

We are particularly interested in directions that are preserved by
$\mathcal C$. Let
$(\lambda_1,\phi_1),(\lambda_2,\phi_2),\ldots$
denote the eigenvalue--eigenfunction pairs of $\mathcal C$, satisfying
\[
\mathcal C\phi_j=\lambda_j\phi_j,
\qquad
\lambda_1\geq\lambda_2\geq\cdots\geq 0.
\]
Each $\phi_j$ is a principal functional direction, while the corresponding
eigenvalue $\lambda_j$ measures how much feature-function energy lies along
that direction. Thus, $(\lambda_1,\phi_1)$ represents the strongest
functional direction, $(\lambda_2,\phi_2)$ the second strongest, and so on.
A larger eigenvalue indicates that the corresponding direction is more
important for describing the collection of target functions $f_k^\star$.
\begin{theorem}[Irreducible Global Basis Approximation]
\label{thm:app_optimal_global_basis}
Let \(\ E_{\mathrm{global}}(B)\) be defined as in
Eq.~\ref{eq:global_error}. An optimal
\(B\)-dimensional shared space is
$\mathcal S_B^\star
=
\operatorname{span}\{\phi_1,\ldots,\phi_B\},$
and the minimum approximation error is
\begin{equation}
\ E_{\mathrm{global}}(B)
=
\sum_{j>B}\lambda_j.
\end{equation}
\end{theorem}

\begin{proof}
Let \(\mathcal S\subseteq\mathcal H\) satisfy
\(\dim(\mathcal S)\leq B\). Since enlarging a subspace cannot increase
the projection error, it suffices to consider
\(\dim(\mathcal S)=B\). Choose an orthonormal basis
\(u_1,\ldots,u_B\) of \(\mathcal S\). For each target feature-response
function \(f_k^\star\), let \(\Pi_{\mathcal S}f_k^\star\) denote its
orthogonal projection onto \(\mathcal S\), i.e., the closest function
in \(\mathcal S\) to \(f_k^\star\) under the Hilbert-space norm.
Because \(u_1,\ldots,u_B\) form an orthonormal basis of
\(\mathcal S\), this projection is obtained by summing the component
of \(f_k^\star\) along each basis direction:
\begin{equation}
\Pi_{\mathcal S}f_k^\star
=
\sum_{b=1}^{B}
\langle f_k^\star,u_b\rangle_{\mathcal H}u_b.
\label{eq:app_projection_expansion}
\end{equation}
Because \(\Pi_{\mathcal S}f_k^\star\) and
\(f_k^\star-\Pi_{\mathcal S}f_k^\star\) are orthogonal, the
Pythagorean identity yields
\begin{equation}
\begin{aligned}
\|f_k^\star\|_{\mathcal H}^{2}
&=
\|\Pi_{\mathcal S}f_k^\star\|_{\mathcal H}^{2}
+
\|f_k^\star-\Pi_{\mathcal S}f_k^\star\|_{\mathcal H}^{2}.
\end{aligned}
\label{eq:app_pythagorean_projection}
\end{equation}
And since \(u_1,\ldots,u_B\) are orthonormal, the squared norm of the projection is

\begin{equation}
\begin{aligned}
\|\Pi_{\mathcal S}f_k^\star\|_{\mathcal H}^{2}
&=
\left\|
\sum_{b=1}^{B}
\langle f_k^\star,u_b\rangle_{\mathcal H}u_b
\right\|_{\mathcal H}^{2}
\\
&=
\sum_{b=1}^{B}
\left|
\langle f_k^\star,u_b\rangle_{\mathcal H}
\right|^{2}.
\end{aligned}
\label{eq:app_projection_energy}
\end{equation}

Substituting Eq.~\ref{eq:app_projection_energy} into
Eq.~\ref{eq:app_pythagorean_projection} yields

\begin{equation}
\begin{aligned}
\left\|
f_k^\star-\Pi_{\mathcal S}f_k^\star
\right\|_{\mathcal H}^{2}
&=
\|f_k^\star\|_{\mathcal H}^{2}
-
\sum_{b=1}^{B}
\left|
\langle f_k^\star,u_b\rangle_{\mathcal H}
\right|^{2}.
\end{aligned}
\label{eq:app_single_function_error}
\end{equation}

Substituting Eq.~\ref{eq:app_single_function_error} into the global
shared-space error in Eq.~\ref{eq:app_global_space_error} shows that
the approximation error is precisely the functional energy not captured
by \(\mathcal S\). Averaging over all \(d\) feature-response functions gives

\begin{equation}
\begin{aligned}
\ E(\mathcal S)
&=
\frac{1}{d}
\sum_{k=1}^{d}
\left\|
f_k^\star-\Pi_{\mathcal S}f_k^\star
\right\|_{\mathcal H}^{2}
\\
&=
\frac{1}{d}
\sum_{k=1}^{d}
\|f_k^\star\|_{\mathcal H}^{2}
-
\frac{1}{d}
\sum_{k=1}^{d}
\sum_{b=1}^{B}
\left|
\langle f_k^\star,u_b\rangle_{\mathcal H}
\right|^{2}
\\
&=
\frac{1}{d}
\sum_{k=1}^{d}
\|f_k^\star\|_{\mathcal H}^{2}
-
\sum_{b=1}^{B}
\frac{1}{d}
\sum_{k=1}^{d}
\left|
\langle f_k^\star,u_b\rangle_{\mathcal H}
\right|^{2}.
\end{aligned}
\label{eq:app_average_projection_error}
\end{equation}

We next express the second term through the feature-function second-moment
operator \(\mathcal C\). From Eq.~\ref{eq:app_function_second_moment},
setting \(g=u_b\) for a shared basis direction \(u_b\) gives

\begin{equation}
\begin{aligned}
\langle u_b,\mathcal C u_b\rangle_{\mathcal H}
&=
\left\langle
u_b,
\frac{1}{d}
\sum_{k=1}^{d}
\langle f_k^\star,u_b\rangle_{\mathcal H}f_k^\star
\right\rangle_{\mathcal H}
\\
&=
\frac{1}{d}
\sum_{k=1}^{d}
\langle f_k^\star,u_b\rangle_{\mathcal H}
\langle u_b,f_k^\star\rangle_{\mathcal H}
\\
&=
\frac{1}{d}
\sum_{k=1}^{d}
\left|
\langle f_k^\star,u_b\rangle_{\mathcal H}
\right|^{2}.
\end{aligned}
\label{eq:app_direction_energy_operator}
\end{equation}

Thus, \(\langle u_b,\mathcal C u_b\rangle_{\mathcal H}\) is precisely the average feature-function energy captured along direction \(u_b\). Using Eq.~\ref{eq:app_direction_energy_operator}, Eq.~\ref{eq:app_average_projection_error} becomes

\begin{equation}
\ E(\mathcal S)
=
\frac{1}{d}
\sum_{k=1}^{d}
\|f_k^\star\|_{\mathcal H}^{2}
-
\sum_{b=1}^{B}
\langle u_b,\mathcal C u_b\rangle_{\mathcal H}.
\label{eq:app_global_error_energy}
\end{equation}

 The first term depends only on the target functions and is therefore
independent of the shared space. Hence, minimizing
\(\ E(\mathcal S)\) is equivalent to maximizing
$\sum_{b=1}^{B}
\langle u_b,\mathcal C u_b\rangle_{\mathcal H}.$
Since \(\mathcal C\) is self-adjoint, positive semidefinite, and finite
rank, choose an orthonormal eigenbasis
\(\phi_1,\phi_2,\ldots\) satisfying
\(\mathcal C\phi_j=\lambda_j\phi_j\), with
$\lambda_1\geq\lambda_2\geq\cdots\geq0.$
Expanding each \(u_b\) in this basis gives
\begin{equation}
u_b
=
\sum_j
\langle u_b,\phi_j\rangle_{\mathcal H}\phi_j.
\label{eq:app_ub_eigen_expansion}
\end{equation}
Applying \(\mathcal C\) to
Eq.~\ref{eq:app_ub_eigen_expansion} gives
\begin{equation}
\begin{aligned}
\mathcal C u_b
&=
\mathcal C
\left(
\sum_j
\langle u_b,\phi_j\rangle_{\mathcal H}\phi_j
\right)
\\
&=
\sum_j
\langle u_b,\phi_j\rangle_{\mathcal H}
\mathcal C\phi_j
\\
&=
\sum_j
\lambda_j
\langle u_b,\phi_j\rangle_{\mathcal H}\phi_j.
\end{aligned}
\label{eq:app_operator_on_ub}
\end{equation}

The first equality substitutes the expansion of \(u_b\), the second uses the linearity of \(\mathcal C\), and the third uses the eigenvalue relation \(\mathcal C\phi_j=\lambda_j\phi_j\). Taking the inner product with \(u_b\), we obtain

\begin{equation}
\begin{aligned}
\langle u_b,\mathcal C u_b\rangle_{\mathcal H}
&=
\left\langle
u_b,
\sum_j
\lambda_j
\langle u_b,\phi_j\rangle_{\mathcal H}\phi_j
\right\rangle_{\mathcal H}
\\
&=
\sum_j
\lambda_j
\langle u_b,\phi_j\rangle_{\mathcal H}
\langle u_b,\phi_j\rangle_{\mathcal H}
\\
&=
\sum_j
\lambda_j
\left|
\langle u_b,\phi_j\rangle_{\mathcal H}
\right|^{2}.
\end{aligned}
\label{eq:app_direction_captured_energy}
\end{equation}

Thus, the energy captured by \(u_b\) is a weighted sum of the
eigenvalues, where
\(\lvert\langle u_b,\phi_j\rangle_{\mathcal H}\rvert^2\) is the squared
component of \(u_b\) along \(\phi_j\), and \(\lambda_j\) quantifies the
feature-function variation along that direction. Summing
Eq.~\ref{eq:app_direction_captured_energy} over
\(b=1,\ldots,B\) gives

\begin{equation}
\begin{aligned}
\sum_{b=1}^{B}
\langle u_b,\mathcal C u_b\rangle_{\mathcal H}
&=
\sum_{b=1}^{B}
\sum_j
\lambda_j
\left|
\langle u_b,\phi_j\rangle_{\mathcal H}
\right|^{2}
\\
&=
\sum_j
\lambda_j
\sum_{b=1}^{B}
\left|
\langle u_b,\phi_j\rangle_{\mathcal H}
\right|^{2}
\\
&=
\sum_j
\lambda_j w_j,
\end{aligned}
\label{eq:app_global_captured_energy}
\end{equation}

where we define $w_j := \sum_{b=1}^{B} \left| \langle u_b,\phi_j\rangle_{\mathcal H} \right|^{2}.$ The quantity \(w_j\) has a direct geometric interpretation. Since

\begin{equation}
\begin{aligned}
\Pi_{\mathcal S}\phi_j
&=
\sum_{b=1}^{B}
\langle \phi_j,u_b\rangle_{\mathcal H}u_b,
\\
\|\Pi_{\mathcal S}\phi_j\|_{\mathcal H}^{2}
&=
\sum_{b=1}^{B}
\left|
\langle \phi_j,u_b\rangle_{\mathcal H}
\right|^{2}
\\
&=
w_j,
\end{aligned}
\label{eq:app_wj_projection}
\end{equation}

\(w_j\) measures how much of the eigen-direction \(\phi_j\) is contained inside the shared space \(\mathcal S\). Because orthogonal projection cannot increase norm,

\begin{equation}
\begin{aligned}
w_j
&=
\|\Pi_{\mathcal S}\phi_j\|_{\mathcal H}^{2}
\\
&\leq
\|\phi_j\|_{\mathcal H}^{2}
\\
&=
1.
\end{aligned}
\label{eq:app_wj_upper_bound}
\end{equation}

Since \(w_j\) is the squared norm of the projection of the normalized
eigenfunction \(\phi_j\) onto \(\mathcal S\),
\(0\leq w_j\leq1\). The value is \(1\) when
\(\phi_j\in\mathcal S\) and \(0\) when
\(\phi_j\perp\mathcal S\); intermediate values represent partial
overlap. The weights also satisfy a sum constraint. Substituting their
definition and exchanging the order of summation gives
\begin{equation}
\begin{aligned}
\sum_j w_j
&=
\sum_j
\sum_{b=1}^{B}
\left|
\langle u_b,\phi_j\rangle_{\mathcal H}
\right|^{2}
\\
&=
\sum_{b=1}^{B}
\sum_j
\left|
\langle u_b,\phi_j\rangle_{\mathcal H}
\right|^{2}.
\end{aligned}
\label{eq:app_wj_budget_step1}
\end{equation}
Because the eigenfunctions form an orthonormal basis, Parseval's identity gives

\begin{equation}
\begin{aligned}
\sum_j
\left|
\langle u_b,\phi_j\rangle_{\mathcal H}
\right|^{2}
&=
\|u_b\|_{\mathcal H}^{2} = 1.
\end{aligned}
\label{eq:app_parseval_ub}
\end{equation}

Substituting Eq.~\ref{eq:app_parseval_ub} into
Eq.~\ref{eq:app_wj_budget_step1} yields
$\sum_j w_j = \sum_{b=1}^{B} 1 = B.$
Thus, the \(B\)-dimensional shared space has a total capacity budget of
\(B\). Each \(w_j\in[0,1]\) specifies how much of this capacity is
allocated to the functional direction \(\phi_j\), and the resulting
captured energy is \(\sum_j\lambda_jw_j\).
We next show that this capacity is optimally allocated to the directions
with the largest eigenvalues. Suppose that there exist indices \(i<j\)
such that \(\lambda_i\geq\lambda_j\), but \(w_i<1\) and \(w_j>0\).
Transfer an amount
\(\delta := \min\{1-w_i,\;w_j\}\) of capacity from direction \(j\) to
direction \(i\). This transfer preserves the constraints
\(0\leq w_j\leq1\) and \(\sum_jw_j=B\). The corresponding change in
captured energy is

\begin{equation}
\begin{aligned}
\Delta
&=
\lambda_i(w_i+\delta)
+\lambda_j(w_j-\delta)
-\lambda_i w_i
-\lambda_j w_j
\\
&=
\delta(\lambda_i-\lambda_j)
\geq 0.
\end{aligned}
\label{eq:app_energy_transfer}
\end{equation}

Therefore, moving capacity from a direction with a smaller eigenvalue to
one with a larger eigenvalue can never decrease the captured energy.
Repeating this exchange shows that an optimum is obtained by assigning
$w_j=1 \quad\text{for }j\leq B,
\qquad
w_j=0 \quad\text{for }j>B.$
Hence, one optimal shared space is
$\mathcal S_B^\star
=
\operatorname{span}\{\phi_1,\ldots,\phi_B\}.$
If eigenvalues are tied at the \(B\)-th position, the optimal space need
not be unique, but any optimal choice spans \(B\) directions associated
with the largest eigenvalues.
It remains to determine how much feature-function energy lies outside
this optimal space. From the definition of \(\mathcal C\), its trace
equals the average total energy of the target functions. Expressing this
trace in terms of its eigenvalues,

\begin{equation}
\begin{aligned}
\frac{1}{d}
\sum_{k=1}^{d}
\|f_k^\star\|_{\mathcal H}^{2}
&=
\operatorname{Tr}(\mathcal C)
\\
&=
\sum_j\lambda_j.
\end{aligned}
\label{eq:app_total_function_energy}
\end{equation}

The optimal shared space captures exactly the first \(B\) eigen-directions, so its captured energy is \(\sum_{j=1}^{B}\lambda_j\). Substituting this into Eq.~\ref{eq:app_global_error_energy} gives

\begin{equation}
\begin{aligned}
\ E_{\mathrm{global}}(B)
&=
\sum_j\lambda_j
-
\sum_{j=1}^{B}\lambda_j
\\
&=
\sum_{j>B}\lambda_j.
\end{aligned}
\label{eq:app_global_error_tail_derivation}
\end{equation}

Therefore, no globally shared \(B\)-dimensional basis can represent the
feature-function variation lying beyond the first \(B\) eigenfunctions.
The remaining spectral tail \(\sum_{j>B}\lambda_j\) is the
irreducible approximation error caused by forcing all feature-response
functions to lie in the same \(B\)-dimensional functional space.
\end{proof}

Theorem~\ref{thm:app_optimal_global_basis} shows when a single shared basis
becomes limiting. A \(B\)-dimensional NBM can preserve only the \(B\)
most important functional directions, while the remaining directions
contribute to the approximation error
\(\sum_{j>B}\lambda_j\). If the eigenvalues decrease quickly, most target
feature-response functions can be well represented in the same
low-dimensional space, and a shared basis is sufficient. In contrast,
if a noticeable amount of variation remains beyond the first \(B\)
directions, then even the best \(B\)-dimensional shared space cannot
represent all target functions accurately. This error is caused by the
limited dimension of the shared function space, rather than by
optimization or insufficient training.

\paragraph{Interpretation Example for Theorem \ref{thm:app_optimal_global_basis}}
Consider four normalized feature-response functions spanning four mutually orthogonal functional directions, with the model restricted to \(B=2\) shared bases. No \(2\)-dimensional global space can represent all four functions. Even with the optimal basis choice, the residual corresponds to two orthogonal directions and produces irreducible approximation error. If all four functions instead lie in the same \(2\)-dimensional space, the spectral tail after \(B=2\) is zero, so global sharing introduces no approximation loss.

\subsection{Functional Specialization and the Specialization Gap}
\label{subsec:app_specialization_gap}
The previous subsection showed that even the best globally shared
\(B\)-dimensional basis space incurs approximation error
\(\ E_{\mathrm{global}}(B)=\sum_{j>B}\lambda_j\).
This error arises because all feature-response functions are restricted
to the same \(B\)-dimensional space, so functional structure outside the
first \(B\) directions cannot be represented. We now ask whether this
error can be reduced by allowing different groups of features to use
different basis spaces. This is useful when the full collection of
feature-response functions is difficult to represent accurately in one
low-dimensional space, while smaller groups have simpler structures.
For example, one group of features may have approximately polynomial
responses, whereas another may have periodic responses. A single shared
space must represent both types of structure at once, while specialized
experts can model them separately.

Let \(\mathcal P=\{\mathcal K_1,\ldots,\mathcal K_C\}\) be a partition of the \(d\) features into \(C\) non-overlapping groups, and let \(d_c:=|\mathcal K_c|\). For each group \(c\), we define a local second-moment operator \(\mathcal C_c:\mathcal H\rightarrow\mathcal H\) by

\begin{equation}
\mathcal C_c g
:=
\frac{1}{d_c}
\sum_{k\in\mathcal K_c}
\langle f_k^\star,g\rangle_{\mathcal H}f_k^\star.
\label{eq:app_local_second_moment}
\end{equation}

This operator has the same role as the global operator \(\mathcal C\), except that it describes variation only among the feature-response functions assigned to group \(c\). Let \(\lambda_{c,1}\geq\lambda_{c,2}\geq\cdots\geq0\) denote the eigenvalues of \(\mathcal C_c\). A rapidly decaying local spectrum means that the functions in that group can be represented accurately inside a small functional space, even if the spectrum of the full collection decays much more slowly.

For a fixed partition \(\mathcal P\), suppose that group \(c\) is represented by its own space \(\mathcal S_c\) with \(\dim(\mathcal S_c)\leq B\). The best approximation error attainable under this partition is

\begin{equation}
\ E_{\mathrm{mix}}(B;\mathcal P)
:=
\inf_{\substack{
\mathcal S_1,\ldots,\mathcal S_C\subseteq\mathcal H\\
\dim(\mathcal S_c)\leq B}}
\frac{1}{d}
\sum_{c=1}^{C}
\sum_{k\in\mathcal K_c}
\left\|
f_k^\star-\Pi_{\mathcal S_c}f_k^\star
\right\|_{\mathcal H}^{2}.
\label{eq:app_fixed_partition_mixture_error}
\end{equation}
The next result shows that specialization replaces the single global spectral tail by a collection of local spectral tails.

\begin{theorem}[Optimal Specialized Approximation]
\label{thm:app_optimal_specialized_approximation}
For a fixed partition
\(\mathcal P=\{\mathcal K_1,\ldots,\mathcal K_C\}\), the minimum approximation
error is
\begin{equation}
E_{\mathrm{mix}}(B;\mathcal P)
=
\sum_{c=1}^{C}
\frac{d_c}{d}
\sum_{j>B}\lambda_{c,j}.
\label{eq:app_local_spectral_tails}
\end{equation}
An optimal space for group \(c\) is spanned by the first \(B\) eigenfunctions
of its local operator \(\mathcal C_c\).
\end{theorem}

\begin{proof}
For group \(c\), define \( \ E_c(\mathcal S_c) := d_c^{-1}\sum_{k\in\mathcal K_c} \|f_k^\star-\Pi_{\mathcal S_c}f_k^\star\|_{\mathcal H}^{2} \). The total approximation error can then be written as a weighted sum of these group-wise errors. Since each expert space \(\mathcal S_c\) is optimized independently,

\begin{equation}
\begin{aligned}
E_{\mathrm{mix}}(B;\mathcal P)
&=
\inf_{\mathcal S_1,\ldots,\mathcal S_C}
\sum_{c=1}^{C}
\frac{d_c}{d}
E_c(\mathcal S_c)
\\
&=
\sum_{c=1}^{C}
\frac{d_c}{d}
\inf_{\dim(\mathcal S_c)\leq B}
E_c(\mathcal S_c).
\end{aligned}
\label{eq:app_groupwise_optimization}
\end{equation}

Each term on the final line is exactly the approximation problem studied in Theorem~\ref{thm:app_optimal_global_basis}, applied only to the functions in group \(c\). Replacing the global second-moment operator by \(\mathcal C_c\) therefore gives \(\inf_{\dim(\mathcal S_c)\leq B}E_c(\mathcal S_c) =\sum_{j>B}\lambda_{c,j}\). Substituting this identity into Eq.~\ref{eq:app_groupwise_optimization} proves Eq.~\ref{eq:app_local_spectral_tails}.
\end{proof}

Theorem~\ref{thm:app_optimal_specialized_approximation} makes the difference between
global sharing and specialization clear. In a global NBM, all feature-response
functions contribute to a single spectrum, and the model must discard the
spectral tail beyond the first \(B\) directions. In contrast, a specialized
model divides the features into groups and only discards the spectral tail of
the local spectrum within each group. Specialization is therefore most useful
when the full collection of feature-response functions is complex when modeled
together, while each individual group has a simpler structure. In contrast, if
all feature-response functions already lie close to the same low-dimensional
space, then the global and local spectral tails will be similar, and
specialization provides little additional representational benefit.

To measure this benefit independently of a particular partition, let \(\mathfrak P_C\) denote the set of partitions into at most \(C\) nonempty groups, and define the oracle mixture error as \(E_{\mathrm{mix}}^\star(B,C) :=\min_{\mathcal P\in\mathfrak P_C} E_{\mathrm{mix}}(B;\mathcal P)\). The specialized model can always reproduce the global solution by assigning the same optimal global space to every group, so \(E_{\mathrm{mix}}^\star(B,C) \leq E_{\mathrm{global}}(B)\). This observation motivates the \emph{specialization gap}, defined as \(\Gamma_{B,C} :=E_{\mathrm{global}}(B) -E_{\mathrm{mix}}^\star(B,C)\geq0\).

The quantity \(\Gamma_{B,C}\) measures the approximation error attributable specifically to global sharing. When \(\Gamma_{B,C}\) is close to zero, one shared \(B\)-dimensional space already captures nearly all of the useful functional structure, so introducing multiple experts has little representational value. A large specialization gap instead indicates that important directions are lost only because unrelated functional structures are forced into the same space. Using the spectral characterizations above, the same quantity can be written as

\begin{equation}
\Gamma_{B,C}
=
\sum_{j>B}\lambda_j
-
\min_{\mathcal P\in\mathfrak P_C}
\sum_{c=1}^{|\mathcal P|}
\frac{d_c}{d}
\sum_{j>B}\lambda_{c,j}.
\label{eq:app_specialization_gap_spectral}
\end{equation}

This expression gives a precise meaning to functional heterogeneity in the context of MoNB. Heterogeneity is not simply the visual observation that two feature-response curves look different. It refers to a mismatch between global and local functional complexity: the collection requires many directions when modeled jointly, but substantially fewer directions after the functions are organized into appropriate groups.

The following limiting case makes this distinction particularly clear.

\begin{corollary}[Strict Benefit of Functional Specialization]
\label{cor:app_strict_specialization}
Suppose the features can be partitioned into \(C\) groups such that, for each
group \(c\), all target functions belong to some subspace
\(\mathcal S_c^\star\) with
\(\dim(\mathcal S_c^\star)\leq B\). If the full collection of target functions
spans a space of dimension strictly larger than \(B\), then
\(E_{\mathrm{mix}}^\star(B,C)=0\),
\(E_{\mathrm{global}}(B)>0\), and consequently
\(\Gamma_{B,C}>0\).
\end{corollary}

\begin{proof}
Choosing \(\mathcal S_c^\star\) as the expert space for group \(c\) represents
every function in that group exactly, so the oracle mixture error is zero.
However, if the full collection spans more than \(B\) independent functional
directions, then the global second-moment operator has rank greater than
\(B\), which implies \(\lambda_{B+1}>0\). By
Theorem~\ref{thm:app_optimal_global_basis},
\begin{equation}
E_{\mathrm{global}}(B)
=
\sum_{j>B}\lambda_j
\geq
\lambda_{B+1}
>
0.
\end{equation}
Therefore,
\(\Gamma_{B,C}
=
E_{\mathrm{global}}(B)
-E_{\mathrm{mix}}^\star(B,C)>0\).
\end{proof}

As a simple example, suppose \(B=2\), with one group of feature-response functions lying in \(\operatorname{span}\{1,x\}\) and another lying in \(\operatorname{span}\{\sin x,\cos x\}\). Each group can be represented exactly by a two-dimensional expert. If these four functions are linearly independent, however, their union spans a four-dimensional space, so no single two-dimensional global basis can represent all of them exactly. This is the regime captured by a positive specialization gap: each functional family is simple by itself, but forcing the families to share the same low-dimensional representation creates unnecessary approximation error.

The analysis above isolates the representational benefit of specialization under an optimal assignment of features to experts. \HARMONIA, however, must learn this assignment through its router, and routing errors can reduce the oracle gain measured by \(\Gamma_{B,C}\). We study this additional cost in Section~\ref{subsec:app_routing_tradeoff}.

\paragraph{Connection to the MoNB parameterization.} The analysis above was stated in terms of abstract specialized functional spaces. We how that, under hard routing, this geometry is exactly the representation structure induced by MoNB. This connection is important because it establishes that the specialization error analyzed above is not an auxiliary function-space construction. It is the approximation error obtained by the actual MoNB parameterization when each feature selects one expert.

Recall that expert \(c\) produces \(B\) basis functions \(\psi_1^{(c)},\ldots,\psi_B^{(c)}\), and let \(\mathcal S_c:=\operatorname{span} \{\psi_1^{(c)},\ldots,\psi_B^{(c)}\}\) denote the functional space generated by that expert. In the hard-routing case, feature \(k\) selects one expert \(z_k\in[C]\). The MoNB shape function then reduces to

\begin{equation}
f_k(x)
=
\sum_{b=1}^{B}
a_{k,b}\psi_b^{(z_k)}(x).
\label{eq:app_hard_routed_monb}
\end{equation}

For a fixed routing assignment \(z_k\), varying the coefficient vector \(\mathbf a_k=(a_{k,1},\ldots,a_{k,B})\) therefore generates exactly the functions contained in \(\mathcal S_{z_k}\). The following result makes the connection to the approximation analysis precise.

\begin{theorem}[Representation Equivalence of Hard-Routed MoNB]
\label{thm:app_monb_representation_equivalence}
Fix expert spaces
\(\mathcal S_1,\ldots,\mathcal S_C\), where
\(\mathcal S_c=\operatorname{span}
\{\psi_1^{(c)},\ldots,\psi_B^{(c)}\}\), and fix a hard-routing assignment
\(z=(z_1,\ldots,z_d)\). Then the smallest average approximation error achievable
by the MoNB coefficients is
\begin{equation}
\inf_{\{\mathbf a_k\}_{k=1}^{d}}
\frac{1}{d}
\sum_{k=1}^{d}
\left\|
f_k^\star
-
\sum_{b=1}^{B}
a_{k,b}\psi_b^{(z_k)}
\right\|_{\mathcal H}^{2}
=
\frac{1}{d}
\sum_{k=1}^{d}
\left\|
f_k^\star-\Pi_{\mathcal S_{z_k}}f_k^\star
\right\|_{\mathcal H}^{2}.
\label{eq:app_monb_projection_equivalence}
\end{equation}
Consequently, hard-routed MoNB realizes exactly the specialized-subspace
approximation problem studied in
Section~\ref{subsec:app_specialization_gap}.
\end{theorem}
\begin{proof}
Consider one feature \(k\) and suppose it is routed to expert \(z_k=c\).
According to Eq.~\ref{eq:app_hard_routed_monb}, the set of functions that
can be produced by varying its coefficient vector is
\begin{equation}
\begin{aligned}
\mathcal F_{k,c}
&=
\left\{
\sum_{b=1}^{B}
a_{k,b}\psi_b^{(c)}
:
(a_{k,1},\ldots,a_{k,B})\in\mathbb R^B
\right\}
\\
&=
\operatorname{span}
\{\psi_1^{(c)},\ldots,\psi_B^{(c)}\}
\\
&=
\mathcal S_c.
\end{aligned}
\label{eq:app_monb_expert_function_class}
\end{equation}
The first line is the set of functions directly parameterized by the
feature-specific coefficients. The second follows from the definition of a
linear span, and the final equality is the definition of the expert space
\(\mathcal S_c\). Thus, once the routing decision is fixed, optimizing the
MoNB coefficients is equivalent to choosing the best function inside the
selected expert space.

The approximation problem for feature \(k\) can therefore be rewritten as
\begin{equation}
\begin{aligned}
\inf_{\mathbf a_k\in\mathbb R^B}
\left\|
f_k^\star
-
\sum_{b=1}^{B}
a_{k,b}\psi_b^{(c)}
\right\|_{\mathcal H}^{2}
&=
\inf_{f\in\mathcal S_c}
\|f_k^\star-f\|_{\mathcal H}^{2}.
\end{aligned}
\label{eq:app_monb_coefficient_to_subspace}
\end{equation}
Because \(\mathcal S_c\) is finite-dimensional, it is a closed subspace of
\(\mathcal H\). The closest element of a closed subspace to
\(f_k^\star\) is its orthogonal projection
\(\Pi_{\mathcal S_c}f_k^\star\). Hence
\begin{equation}
\begin{aligned}
\inf_{f\in\mathcal S_c}
\|f_k^\star-f\|_{\mathcal H}^{2}
&=
\left\|
f_k^\star-\Pi_{\mathcal S_c}f_k^\star
\right\|_{\mathcal H}^{2}.
\end{aligned}
\label{eq:app_monb_projection_solution}
\end{equation}

Substituting \(c=z_k\) and applying the same argument independently to every
feature gives
\begin{equation}
\begin{aligned}
&\inf_{\{\mathbf a_k\}_{k=1}^{d}}
\frac{1}{d}
\sum_{k=1}^{d}
\left\|
f_k^\star
-
\sum_{b=1}^{B}
a_{k,b}\psi_b^{(z_k)}
\right\|_{\mathcal H}^{2}
\\
&\qquad=
\frac{1}{d}
\sum_{k=1}^{d}
\inf_{\mathbf a_k\in\mathbb R^B}
\left\|
f_k^\star
-
\sum_{b=1}^{B}
a_{k,b}\psi_b^{(z_k)}
\right\|_{\mathcal H}^{2}
\\
&\qquad=
\frac{1}{d}
\sum_{k=1}^{d}
\left\|
f_k^\star-\Pi_{\mathcal S_{z_k}}f_k^\star
\right\|_{\mathcal H}^{2}.
\end{aligned}
\label{eq:app_monb_full_projection_equivalence}
\end{equation}

The first equality separates the coefficient optimization because each feature
has its own coefficient vector, while the second applies
Eq.~\ref{eq:app_monb_projection_solution} to the expert selected for each
feature. This proves Eq.~\ref{eq:app_monb_projection_equivalence}.
\end{proof}
Theorem~\ref{thm:app_monb_representation_equivalence} shows that the abstract geometry used in the specialization analysis coincides with the actual hard-routed MoNB architecture. Each expert defines one low-dimensional functional space, the router determines which space is available to a feature, and the feature-specific coefficient vector determines the particular function inside that space. Thus the role of the router is more precise than simply ``selecting an expert'': it selects the functional coordinate system in which the feature-response function will be represented.

The same result also explains why MoNB can maintain greater functional diversity than a single \(B\)-basis NBM without increasing the active basis dimension of each feature. Across the full model, the expert spaces are contained in their linear sum \(\mathcal S_{\mathrm{all}} :=\mathcal S_1+\cdots+\mathcal S_C\), whose dimension satisfies \(\dim(\mathcal S_{\mathrm{all}})\leq CB\). When the expert spaces contain distinct functional directions, this model-wide space may therefore be much larger than a single \(B\)-dimensional NBM space. Nevertheless, under hard routing each feature accesses only one expert space and hence uses at most \(B\) active functional directions. The model can therefore maintain globally diverse functional structures while preserving a small local representation for each feature.

This distinction should not be confused with an absolute expressivity advantage over an arbitrarily wide NBM. A global basis model with enough basis functions could span \(\mathcal S_{\mathrm{all}}\) and represent the same collection of target functions. In the worst case, doing so may require up to \(CB\) global basis directions. The relevant difference is therefore not whether the functions are representable with unlimited capacity, but how that capacity is organized: a wide global model exposes every feature to the full basis space, whereas MoNB maintains several specialized spaces and activates only the space selected for that feature. The statistical consequence of this difference is studied in Section~\ref{subsec:app_statistical_efficiency}.

A useful corollary follows immediately from the representation equivalence. Suppose the target feature functions admit a partition \(\mathcal K_1,\ldots,\mathcal K_C\) such that every \(f_k^\star\) in group \(c\) lies in a target subspace \(\mathcal S_c^\star\) of dimension at most \(B\). If expert \(c\) learns a basis whose span contains \(\mathcal S_c^\star\), then routing every \(k\in\mathcal K_c\) to that expert allows MoNB to represent all target functions in the group exactly. Indeed, \(f_k^\star\in\mathcal S_c\) implies \(\Pi_{\mathcal S_c}f_k^\star=f_k^\star\), so its approximation error in Eq.~\ref{eq:app_monb_projection_equivalence} is zero. This is precisely the architectural condition under which the zero local spectral tails considered in Corollary~\ref{cor:app_strict_specialization} can be attained by MoNB.

The result above concerns representation capacity under a correct hard-routing assignment. It does not assume that the learned router always identifies the appropriate expert. If a feature is assigned to a space that does not contain the relevant functional directions, its approximation error increases from the distance to the intended space to the distance to the selected one. The next part quantifies this loss and determines when the approximation benefit created by specialization remains larger than the error introduced by imperfect routing.

\subsection{Routing Error and the Specialization--Routing Trade-off}
\label{subsec:app_routing_tradeoff}

The previous analysis considered an oracle assignment of features to
specialized functional spaces. Under this assignment, MoNB can reduce the
approximation error by the specialization gap
\(\Gamma_{B,C}\). In practice, however, the appropriate expert is not known in
advance and must be inferred by the router. If a feature is routed to an unsuitable expert, its approximation error can increase because it is represented using a less appropriate functional space. This routing error can reduce, or even eliminate, the benefit gained from specialization. We therefore quantify this additional error and determine when specialization still provides a net advantage despite imperfect routing.

To isolate the routing effect from errors in learning the expert bases
themselves, consider an oracle solution consisting of expert spaces
\(\mathcal S_1^\star,\ldots,\mathcal S_C^\star\) and an associated assignment
\(z^\star=(z_1^\star,\ldots,z_d^\star)\) that attain
\(E_{\mathrm{mix}}^\star(B,C)\). If the infimum is not attained
exactly, the same argument can be applied to an arbitrarily close
\(\varepsilon\)-optimal solution with an additional \(\varepsilon\) term.
For feature \(k\), recall that \(e_k(c)\) denotes the approximation error
obtained when \(f_k^\star\) is represented using expert \(c\). Since
\(z^\star\) is an optimal assignment for the ideal expert spaces, we may choose
\(z_k^\star\) such that
\[
e_k(z_k^\star)=\min_{c\in[C]} e_k(c).
\]
Let \(\widehat z_k\) denote the expert selected by the learned hard router.
We define the routing penalty for assigning feature \(k\) to expert \(c\)
instead of an optimal expert as
\[
\Delta_k(c):=e_k(c)-e_k(z_k^\star).
\]
By construction,
\(\Delta_k(c)\geq0\), and the penalty is zero whenever expert \(c\) provides
an equally good approximation to the target function as the oracle expert.
This definition is therefore more informative than simply checking whether
\(c=z_k^\star\): two experts may have different indices while spanning
functionally equivalent spaces, in which case switching between them should
not be counted as a harmful routing error. For a routing assignment
\(\widehat z=(\widehat z_1,\ldots,\widehat z_d)\), define the average routing
penalty as

\begin{equation}
\mathcal R_{\mathrm{route}}(\widehat z)
:=
\frac{1}{d}
\sum_{k=1}^{d}
\Delta_k(\widehat z_k).
\label{eq:app_routing_penalty}
\end{equation}

This quantity measures the additional approximation error caused solely by
using the selected expert instead of the best expert available under the
oracle specialized representation. The following result shows that this
penalty competes directly with the specialization gap.

\specializationrouting*

\begin{proof}
By definition, the approximation error under the learned routing assignment is
\begin{equation}
E_{\mathrm{route}}(\widehat z)
=
\frac{1}{d}
\sum_{k=1}^{d}
e_k(\widehat z_k).
\label{eq:app_routed_approximation_error}
\end{equation}
For each feature \(k\), the definition of
\(\Delta_k(\widehat z_k)\) gives
\(e_k(\widehat z_k)
=
e_k(z_k^\star)+\Delta_k(\widehat z_k)\).
Substituting this identity into
Eq.~\ref{eq:app_routed_approximation_error} yields
\begin{equation}
\begin{aligned}
E_{\mathrm{route}}(\widehat z)
&=
\frac{1}{d}
\sum_{k=1}^{d}
\left[
e_k(z_k^\star)
+
\Delta_k(\widehat z_k)
\right]
\\
&=
\frac{1}{d}
\sum_{k=1}^{d}
e_k(z_k^\star)
+
\frac{1}{d}
\sum_{k=1}^{d}
\Delta_k(\widehat z_k)
\\
&=
E_{\mathrm{mix}}^\star(B,C)
+
\mathcal R_{\mathrm{route}}(\widehat z).
\end{aligned}
\label{eq:app_routed_error_decomposition}
\end{equation}

The first term is the oracle approximation error obtained under the optimal
specialized spaces and routing assignment. The second term contains exactly
the additional error introduced by changing the routing decisions while
keeping those spaces fixed. Recall that the specialization gap is defined by
\(\Gamma_{B,C}
=
E_{\mathrm{global}}(B)
-
E_{\mathrm{mix}}^\star(B,C)\).
Rearranging this definition gives
\(E_{\mathrm{mix}}^\star(B,C)
=
E_{\mathrm{global}}(B)-\Gamma_{B,C}\).
Substituting into
Eq.~\ref{eq:app_routed_error_decomposition} gives
\begin{equation}
\begin{aligned}
E_{\mathrm{route}}(\widehat z)
&=
E_{\mathrm{global}}(B)
-
\Gamma_{B,C}
+
\mathcal R_{\mathrm{route}}(\widehat z),
\end{aligned}
\end{equation}

and therefore $E_{\mathrm{route}}(\widehat z)
-
E_{\mathrm{global}}(B)
=
\mathcal R_{\mathrm{route}}(\widehat z)
-
\Gamma_{B,C}.$ It follows immediately that
\(E_{\mathrm{route}}(\widehat z)
<
E_{\mathrm{global}}(B)\)
exactly when
\(\mathcal R_{\mathrm{route}}(\widehat z)<\Gamma_{B,C}\).
\end{proof}

Theorem~\ref{thm:app_specialization_routing_tradeoff} gives a direct
interpretation of the specialization gap. It is not only the improvement
obtained under perfect routing; it also determines how much routing error the
specialized representation can tolerate before losing its advantage over
global sharing. When \(\Gamma_{B,C}\) is large, the target feature functions
contain substantial heterogeneous structure, and specialization can remain
useful even when routing is imperfect. When \(\Gamma_{B,C}\) is small, the
global basis already represents the feature functions well, so even a modest
routing penalty may eliminate the benefit of introducing multiple experts.

The identity also reveals three regimes. If
\(\mathcal R_{\mathrm{route}}(\widehat z)<\Gamma_{B,C}\), specialization
retains a net approximation advantage. If the two quantities are equal, the
routing error exactly consumes the gain from specialization and the two
representations have the same approximation error. If
\(\mathcal R_{\mathrm{route}}(\widehat z)>\Gamma_{B,C}\), the cost of routing
into inappropriate functional spaces exceeds the structure recovered by
specialization, and the best globally shared \(B\)-dimensional representation
becomes preferable at the approximation level.

\paragraph{Why routing accuracy alone is insufficient.}
The routing penalty depends not only on how often a feature is assigned away
from an oracle expert, but also on how different the selected functional space
is from the best available one. To make this distinction explicit, define the
set of features that incur a positive routing penalty as
\(\mathcal M(\widehat z)
:=
\{k:\Delta_k(\widehat z_k)>0\}\), and let
\(p_{\mathrm{route}}
:=|\mathcal M(\widehat z)|/d\).
When \(\mathcal M(\widehat z)\) is nonempty, define the average cost of a
harmful routing decision as
\[
\overline{\Delta}_{\mathrm{route}}
:=
\frac{1}{|\mathcal M(\widehat z)|}
\sum_{k\in\mathcal M(\widehat z)}
\Delta_k(\widehat z_k).
\]
The routing penalty then satisfies the exact identity
\(\mathcal R_{\mathrm{route}}
=
p_{\mathrm{route}}\overline{\Delta}_{\mathrm{route}}\).
The specialization condition from
Theorem~\ref{thm:app_specialization_routing_tradeoff} can therefore be written
as
\begin{equation}
p_{\mathrm{route}}
\overline{\Delta}_{\mathrm{route}}
<
\Gamma_{B,C}.
\label{eq:app_routing_frequency_cost_tradeoff}
\end{equation}

Equation~\ref{eq:app_routing_frequency_cost_tradeoff} separates two
different sources of routing failure. The quantity \(p_{\mathrm{route}}\)
measures how frequently the router makes a functionally harmful decision,
whereas \(\overline{\Delta}_{\mathrm{route}}\) measures the average
approximation cost of such a decision. A router can therefore have imperfect
assignment accuracy while still incurring little functional loss if the
confused experts span similar spaces. Conversely, a small number of mistakes
can be expensive when they route features between strongly different
functional families. For this reason, expert-assignment accuracy alone is not
a sufficient measure of routing quality for MoNB; the geometry of the expert
spaces determines the consequence of each error.

This geometry can be summarized through a feature-specific routing margin.
For feature \(k\), define $m_k
:=
\min_{c:\Delta_k(c)>0}
\Delta_k(c),$ whenever at least one expert has positive excess error. The quantity \(m_k\)
is the smallest approximation loss incurred by moving feature \(k\) away from
the set of optimal expert spaces. A large margin means that the feature has a
strongly preferred functional family, whereas a small margin indicates that
multiple experts provide nearly equivalent representations. For every harmful
routing decision,
\(\Delta_k(\widehat z_k)\geq m_k\), and consequently

\begin{equation}
\mathcal R_{\mathrm{route}}(\widehat z)
\geq
\frac{1}{d}
\sum_{k\in\mathcal M(\widehat z)}
m_k.
\label{eq:app_routing_margin_lower_bound}
\end{equation}

This result highlights a trade-off that is easy to miss when routing is viewed
only as a classification problem. Strongly separated expert spaces make the
functional specialization more distinct, but they can also make incorrect
assignments more costly. Overlapping experts may be harder to distinguish by
their indices, yet confusing two nearly equivalent spaces can have little
effect on the represented feature function. The relevant object is therefore
not discrete routing accuracy by itself, but the approximation loss induced by
the chosen expert.

A complementary sufficient condition can be obtained by bounding the largest
possible routing penalty. Let
\(\Delta_{\max}:=\max_{k,c}\Delta_k(c)\). Since only features in
\(\mathcal M(\widehat z)\) contribute positive error,

\begin{equation}
\begin{aligned}
\mathcal R_{\mathrm{route}}(\widehat z)
&=
\frac{1}{d}
\sum_{k\in\mathcal M(\widehat z)}
\Delta_k(\widehat z_k)
\\
&\leq
\frac{|\mathcal M(\widehat z)|}{d}
\Delta_{\max}
\\
&=
p_{\mathrm{route}}\Delta_{\max}.
\end{aligned}
\label{eq:app_routing_penalty_upper_bound}
\end{equation}

Hence a sufficient condition for specialization to remain beneficial is
\(p_{\mathrm{route}}\Delta_{\max}<\Gamma_{B,C}\). Although this worst-case
bound may be conservative, it makes the dependence particularly transparent:
a larger specialization gap permits either more routing errors or more costly
individual errors before the advantage over global sharing disappears.

\paragraph{Why more experts need not always help.}
The oracle mixture error cannot increase as the number of available experts
grows, so the specialization gap
\(\Gamma_{B,C}\) is nondecreasing in \(C\). At the representation level,
additional experts can therefore only increase the opportunity for
specialization. The learned model, however, must also distinguish among a
larger collection of functional spaces. Writing
\(\mathcal R_{\mathrm{route}}(C)\) for the routing penalty obtained with
\(C\) experts, Theorem~\ref{thm:app_specialization_routing_tradeoff} gives

\begin{equation}
E_{\mathrm{route}}(C)
=
E_{\mathrm{global}}(B)
-
\Gamma_{B,C}
+
\mathcal R_{\mathrm{route}}(C).
\label{eq:app_number_experts_tradeoff}
\end{equation}

Increasing \(C\) is beneficial only when the additional reduction in oracle
approximation error is larger than the additional routing cost. Thus, even
though representational capacity improves monotonically with the number of
experts, the approximation error of a learned routed model need not. This
provides a theoretical explanation for why the number of experts should be
treated as a genuine model-selection parameter rather than increased
indefinitely.

The analysis in this subsection deliberately isolates routing from estimation
of the expert spaces. It assumes that the specialized spaces themselves are
available and asks what is lost when features are assigned to them imperfectly.
In a learned model, the basis functions must also be estimated from finite
data. The next subsection studies this complementary source of error and asks
whether specialization can provide a statistical advantage over representing
the same model-wide functional diversity with one widened global basis space.

\paragraph{Extension to Sparse Top-$m$ Routing.}
\label{par:app_topm}
\HARMONIA uses sparse top-\(m\) routing rather than assigning each feature to a
single expert. Let \(\tilde{\pi}_{k,c}\) denote the routing weight of feature \(k\) for
expert \(c\), with at most \(m\) nonzero entries. The corresponding feature
response is
\begin{equation}
f_k(x)
=
\sum_{c=1}^{C}\tilde{\pi}_{k,c}
\left(
\sum_{b=1}^{B}a_{k,b}\psi_b^{(c)}(x)
\right)
=
\sum_{b=1}^{B}
a_{k,b}\widetilde{\psi}_{k,b}(x),
\qquad
\widetilde{\psi}_{k,b}(x)
:=
\sum_{c=1}^{C}\tilde{\pi}_{k,c}\psi_b^{(c)}(x).
\label{eq:app_topm_effective_basis}
\end{equation}
Thus, top-\(m\) routing induces a feature-dependent effective basis space
\(\mathcal S(\boldsymbol{\pi}_k)
:=
\operatorname{span}\{\widetilde{\psi}_{k,1},\ldots,
\widetilde{\psi}_{k,B}\}\).
Because the coefficients \(a_{k,b}\) are shared across the active experts,
activating \(m\) experts does not expose the feature to \(mB\) independent
coefficients; for fixed routing weights, its effective space still has
dimension at most \(B\). Hard routing is recovered when
\(\boldsymbol{\pi}_k\) has a single nonzero entry, whereas top-\(m\) routing
allows interpolation among expert bases and can therefore represent functional
spaces that do not coincide with any individual expert.

The routing-error analysis extends by replacing the discrete expert space
\(\mathcal S_{z_k}\) with the effective space
\(\mathcal S(\boldsymbol{\pi}_k)\). Let
\(\boldsymbol{\pi}_k^\star\) denote the oracle sparse routing weights and
\(\widehat{\boldsymbol{\pi}}_k\) the learned weights. The excess routing cost
for feature \(k\) becomes
\[
\Delta_k^{(m)}
=
\left\|
f_k^\star-
\Pi_{\mathcal S(\widehat{\boldsymbol{\pi}}_k)}f_k^\star
\right\|_{\mathcal H}^{2}
-
\left\|
f_k^\star-
\Pi_{\mathcal S(\boldsymbol{\pi}_k^\star)}f_k^\star
\right\|_{\mathcal H}^{2}.
\]
Averaging \(\Delta_k^{(m)}\) over features gives the same decomposition between
oracle specialization gain and routing penalty as in the hard-routing case.
Top-\(m\) routing therefore enriches the family of functional spaces available
to each feature while preserving the central specialization--routing
trade-off.

\subsection{Statistical Efficiency of Specialized Basis Sharing}
\label{subsec:app_statistical_efficiency}

One may then ask whether specialization is still useful if a single global basis is made sufficiently wide. Let $\mathcal S_1,\ldots,\mathcal S_C$ denote the expert spaces, with $\dim(\mathcal S_c)\leq B$, and let $\mathcal S_{\mathrm{all}}=\mathcal S_1+\cdots+\mathcal S_C$ have dimension $R\leq CB$. A global basis spanning $\mathcal S_{\mathrm{all}}$ can represent every function available to the specialized experts, so specialization no longer offers an advantage in representational capacity once the global basis is sufficiently wide. What remains is a statistical difference. The global model represents each feature using all $R$ directions, whereas a correctly routed feature only needs the directions of its selected expert, whose dimension is at most $B$.

We isolate this effect by conditioning on fixed basis functions and a fixed
hard-routing assignment. Consider a feature \(k\) whose target response
\(f_k^\star\) belongs to an \(r_k\)-dimensional expert space
\(\mathcal S_{z_k}\), where \(r_k\leq B\). Suppose \(N\) noisy observations
\((x_s,y_s)_{s=1}^{N}\) are available, with
\(y_s=f_k^\star(x_s)+\varepsilon_s\), where the noise variables are independent,
have zero mean, and satisfy
\(\mathbb E[\varepsilon_s^2]=\sigma^2\). For a \(q\)-dimensional basis
\(\phi_1,\ldots,\phi_q\), let
\(\boldsymbol{\Phi}_q\in\mathbb R^{N\times q}\) denote the design matrix with
\([\boldsymbol{\Phi}_q]_{s,b}=\phi_b(x_s)\). Assuming the target is correctly
specified in this basis, we may write
\(\mathbf y=\boldsymbol{\Phi}_q\boldsymbol{\beta}^\star+\boldsymbol{\varepsilon}\).

\begin{theorem}[Dimension-Dependent Coefficient Estimation]
\label{thm:app_statistical_efficiency}
Assume that \(\boldsymbol{\Phi}_q\) has full column rank and estimate the
basis coefficients by least squares,
\[
\widehat{\boldsymbol{\beta}}
=
(\boldsymbol{\Phi}_q^\top\boldsymbol{\Phi}_q)^{-1}
\boldsymbol{\Phi}_q^\top\mathbf y.
\]
Conditioned on the design matrix, the expected empirical prediction error
introduced by coefficient estimation is
\begin{equation}
\mathbb E_{\boldsymbol{\varepsilon}}
\left[
\frac{1}{N}
\left\|
\boldsymbol{\Phi}_q
(\widehat{\boldsymbol{\beta}}-\boldsymbol{\beta}^\star)
\right\|_2^2
\;\middle|\;
\boldsymbol{\Phi}_q
\right]
=
\frac{\sigma^2 q}{N}.
\label{eq:app_dimension_estimation_error}
\end{equation}
Therefore, if feature \(k\) is represented in its \(r_k\)-dimensional
specialized space, its estimation error is
\(\sigma^2r_k/N\), whereas a dense global parameterization over
\(\mathcal S_{\mathrm{all}}\) incurs \(\sigma^2R/N\). When every expert has
dimension \(B\), the difference becomes
\(\sigma^2(R-B)/N\) per feature.
\end{theorem}

\begin{proof}
Since the model is correctly specified,
\(\mathbf y=\boldsymbol{\Phi}_q\boldsymbol{\beta}^\star+
\boldsymbol{\varepsilon}\). Substituting this expression into the least-squares
estimator gives
\begin{equation}
\begin{aligned}
\widehat{\boldsymbol{\beta}}
&=
(\boldsymbol{\Phi}_q^\top\boldsymbol{\Phi}_q)^{-1}
\boldsymbol{\Phi}_q^\top
\left(
\boldsymbol{\Phi}_q\boldsymbol{\beta}^\star
+
\boldsymbol{\varepsilon}
\right)
\\
&=
\boldsymbol{\beta}^\star
+
(\boldsymbol{\Phi}_q^\top\boldsymbol{\Phi}_q)^{-1}
\boldsymbol{\Phi}_q^\top\boldsymbol{\varepsilon},
\end{aligned}
\label{eq:app_ls_decomposition}
\end{equation}
where the second equality uses
\((\boldsymbol{\Phi}_q^\top\boldsymbol{\Phi}_q)^{-1}
\boldsymbol{\Phi}_q^\top\boldsymbol{\Phi}_q=\mathbf I_q\).
Hence the coefficient-estimation error is
\[
\widehat{\boldsymbol{\beta}}-\boldsymbol{\beta}^\star
=
(\boldsymbol{\Phi}_q^\top\boldsymbol{\Phi}_q)^{-1}
\boldsymbol{\Phi}_q^\top\boldsymbol{\varepsilon}.
\]

Multiplying by \(\boldsymbol{\Phi}_q\) converts coefficient error into error in
the fitted function values:
\begin{equation}
\boldsymbol{\Phi}_q
(\widehat{\boldsymbol{\beta}}-\boldsymbol{\beta}^\star)
=
\mathbf P_q\boldsymbol{\varepsilon},
\qquad
\mathbf P_q
:=
\boldsymbol{\Phi}_q
(\boldsymbol{\Phi}_q^\top\boldsymbol{\Phi}_q)^{-1}
\boldsymbol{\Phi}_q^\top.
\label{eq:app_hat_projection}
\end{equation}
The matrix \(\mathbf P_q\) is the orthogonal projector onto the column space of
\(\boldsymbol{\Phi}_q\). Because
\(\boldsymbol{\Phi}_q\) has \(q\) linearly independent columns, this column
space has dimension \(q\). Consequently,
\(\mathbf P_q^\top=\mathbf P_q\),
\(\mathbf P_q^2=\mathbf P_q\), and
\(\operatorname{Tr}(\mathbf P_q)=q\).

Using Eq.~\ref{eq:app_hat_projection}, the expected squared prediction error
is
\begin{equation}
\begin{aligned}
\mathbb E_{\boldsymbol{\varepsilon}}
\left[
\frac{1}{N}
\left\|
\boldsymbol{\Phi}_q
(\widehat{\boldsymbol{\beta}}-\boldsymbol{\beta}^\star)
\right\|_2^2
\;\middle|\;
\boldsymbol{\Phi}_q
\right]
&=
\frac{1}{N}
\mathbb E_{\boldsymbol{\varepsilon}}
\left[
\boldsymbol{\varepsilon}^{\top}
\mathbf P_q^{\top}\mathbf P_q
\boldsymbol{\varepsilon}
\right]
\\
&=
\frac{1}{N}
\mathbb E_{\boldsymbol{\varepsilon}}
\left[
\boldsymbol{\varepsilon}^{\top}
\mathbf P_q
\boldsymbol{\varepsilon}
\right]
\\
&=
\frac{1}{N}
\operatorname{Tr}
\left(
\mathbf P_q
\mathbb E[
\boldsymbol{\varepsilon}
\boldsymbol{\varepsilon}^{\top}]
\right)
\\
&=
\frac{\sigma^2}{N}
\operatorname{Tr}(\mathbf P_q)
\\
&=
\frac{\sigma^2q}{N}.
\end{aligned}
\label{eq:app_estimation_trace}
\end{equation}
The second equality follows from
\(\mathbf P_q^\top\mathbf P_q=\mathbf P_q\), while the fourth uses
\(\mathbb E[
\boldsymbol{\varepsilon}\boldsymbol{\varepsilon}^{\top}]
=\sigma^2\mathbf I_N\). This proves
Eq.~\ref{eq:app_dimension_estimation_error}. For the specialized representation of feature \(k\), the active space has
dimension \(q=r_k\), giving expected error \(\sigma^2r_k/N\). A dense global
basis spanning \(\mathcal S_{\mathrm{all}}\) uses \(q=R\), giving
\(\sigma^2R/N\). The comparison follows immediately.
\end{proof}

The theorem separates model-wide functional capacity from the dimension needed
by an individual feature. A global basis must contain the directions required
by all functional families, whereas a routed feature only needs the directions
of its selected expert. As a result, a dense global model may estimate many
coefficients that are irrelevant to a particular feature, increasing estimation
variance under finite data. 

For example,  suppose \HARMONIA contains \(C=3\) experts, each spanning two independent basis
directions,
\(\mathcal S_1=\operatorname{span}\{\psi_1^{(1)},\psi_2^{(1)}\}\),
\(\mathcal S_2=\operatorname{span}\{\psi_1^{(2)},\psi_2^{(2)}\}\), and
\(\mathcal S_3=\operatorname{span}\{\psi_1^{(3)},\psi_2^{(3)}\}\).
If these spaces contain distinct directions, their joint span has dimension
\(R=6\). A global NBM with six basis functions can therefore represent the
same collection of feature-response functions as the three specialized
experts, hence, the advantage of MoNB is not greater representational capacity in
this case. \textit{The difference is the number of coefficients that each feature must estimate}.
Under the global model, feature \(k\) is represented as
\(f_k(x)=\sum_{b=1}^{6}a_{k,b}\psi_b(x)\), and therefore estimates six
coefficients. Suppose instead that its target response belongs to
\(\mathcal S_2\). Under correct routing, MoNB represents the same function
using only
\(f_k(x)=a_{k,1}\psi_1^{(2)}(x)+a_{k,2}\psi_2^{(2)}(x)\), so only two
coefficients are estimated. In the least-squares setting of
Theorem~\ref{thm:app_statistical_efficiency}, these two parameterizations incur
expected estimation errors \(6\sigma^2/N\) and \(2\sigma^2/N\), respectively.
Thus, MoNB can preserve the same model-wide functional diversity while keeping
the active estimation dimension of each feature small. If the expert spaces
largely overlap, this advantage decreases; if they coincide completely, then
the global and specialized representations have the same effective dimension.

\section{IDENTIFIABILITY OF RANDOM-WALK STRUCTURAL EXPLANATIONS}
\label{sec:app_identifiability}

\HARMONIA uses the RRWP coefficient $\theta_{t,k}$ to represent how much
structural influence feature $k$ assigns to random-walk length $t$.
For these hop weights to provide a meaningful explanation, they should be
uniquely determined by the resulting structural operator. If two different
coefficient vectors produce the same operator, then they induce exactly the
same model behavior but give different hop-level explanations. In that case,
the hop-level explanation is not identifiable.

This issue is separate from predictive accuracy. A model may learn the correct
structural transformation while its individual hop coefficients are still not
uniquely determined. In other words, different hop-weight vectors may produce
the same structural behavior.

Even when the hop coefficients are theoretically unique, the explanation may
still be unstable. Different hop decompositions can induce very similar
structural operators, so small changes caused by finite data or optimization
perturbations may lead to large changes in the learned coefficients.

We therefore study structural identifiability at two levels. First, we ask when
the random-walk coefficients are uniquely determined. Second, we ask when the
mapping from hop coefficients to the structural operator is sufficiently
well-conditioned so that the coefficients remain stable under small
perturbations. Before addressing these questions, normalization fixes the
overall scale of the structural component, allowing hop-level uniqueness to be
defined meaningfully.

\subsection{Normalization and Scale Identifiability}
\label{subsec:app_normalization_identifiability}

For target node $i$, \HARMONIA constructs a propagated representation
$h_i \in \mathbb{R}^d$, where $[h_i]_k$ denotes the contribution associated
with feature $k$. This contribution is written as
\[
[h_i]_k
=
\sum_{j\in V}
\omega_k^{\sharp}(p_{i,j})\,f_k(x_{j,k}).
\]
Here, $f_k(x_{j,k})$ is the response of feature $k$ at source node $j$,
$p_{i,j}$ is the RRWP descriptor that captures the structural relation
between nodes $j$ and $i$ across multiple walk lengths, and
$\omega_k^{\sharp}(p_{i,j})$ determines how strongly the response from $j$ contributes
to target node $i$. Thus, $[h_i]_k$ is the total propagated contribution of
feature $k$ received by node $i$ from all source nodes.

To interpret the feature response and structural influence separately, their
scales must be fixed. Without normalization, only their product is determined.
For any nonzero scalar $a$, replacing $f_k$ by $a f_k$ and $\omega_k^{\sharp}$ by
$a^{-1}\omega_k^{\sharp}$ leaves
$\omega_k^{\sharp}(p_{i,j})f_k(x_{j,k})$ unchanged. Thus, the model prediction is the
same even though the magnitudes assigned to the feature and structural terms
can change arbitrarily. This creates a multiplicative scale ambiguity.

For feature $k$, \HARMONIA models the structural influence between nodes $j$
and $i$ as a weighted combination of random-walk lengths,
\[
\omega_k^{\sharp}(p_{i,j})
=
\sum_{t=0}^{T-1}
\theta_{t,k}(\mathbf M^t)_{i,j},
\]
where $(\mathbf M^t)_{i,j}$ represents the $t$-step random-walk relation
between the two nodes, and $\theta_{t,k}$ specifies how much relative
structural weight feature $k$ assigns to walk length $t$. Collecting these
pairwise structural weights over all node pairs gives the feature-specific
structural operator
\[
\mathbf P_k
:=
\sum_{t=0}^{T-1}
\theta_{t,k}\mathbf M^t.
\]

We constrain the hop weights to the probability simplex,
$\theta_{t,k}\geq 0$ and
$\sum_{t=0}^{T-1}\theta_{t,k}=1$, so they describe a normalized allocation
of structural influence across walk lengths. Since $\mathbf M$ is
row-stochastic, every power $\mathbf M^t$ is also row-stochastic. Therefore,
$\mathbf P_k$, being a convex combination of these random-walk operators, is
also row-stochastic and inherits the same normalization.

\begin{proposition}[Removal of Multiplicative Scale Ambiguity]
\label{prop:app_scale_identifiability}
Suppose \(\mathbf M\) is row-stochastic and
\(\boldsymbol{\theta}_k\) belongs to the probability simplex. Then
\(\mathbf P_k=\sum_{t=0}^{T-1}\theta_{t,k}\mathbf M^t\) is row-stochastic.
Consequently, if another normalized RRWP operator satisfies
\(\widetilde{\mathbf P}_k=a\mathbf P_k\) for some scalar \(a\), then
\(a=1\). Thus the global multiplicative scale ambiguity between the
feature-response function and the topology operator is removed.
\end{proposition}

\begin{proof}
Because \(\mathbf M\) is row-stochastic,
\(\mathbf M\mathbf 1=\mathbf 1\), where \(\mathbf 1\) is the all-ones vector.
Repeated application gives
\(\mathbf M^t\mathbf 1=\mathbf 1\) for every \(t\geq0\). Therefore,
\begin{equation}
\begin{aligned}
\mathbf P_k\mathbf 1
&=
\left(
\sum_{t=0}^{T-1}
\theta_{t,k}\mathbf M^t
\right)\mathbf 1
\\
&=
\sum_{t=0}^{T-1}
\theta_{t,k}
\mathbf M^t\mathbf 1
\\
&=
\sum_{t=0}^{T-1}
\theta_{t,k}\mathbf 1
\\
&=
\left(
\sum_{t=0}^{T-1}\theta_{t,k}
\right)\mathbf 1
=
\mathbf 1.
\end{aligned}
\label{eq:app_rrwp_row_normalization}
\end{equation}
Thus every row of \(\mathbf P_k\) sums to one. Now suppose
\(\widetilde{\mathbf P}_k=a\mathbf P_k\)
is also a normalized RRWP operator. Its rows must also sum to one, so
\[
\mathbf 1
=
\widetilde{\mathbf P}_k\mathbf 1
=
a\mathbf P_k\mathbf 1
=
a\mathbf 1.
\]
Hence \(a=1\). A nontrivial rescaling of the structural operator is therefore
incompatible with the normalization constraint.
\end{proof}

The effect of this constraint is easiest to see through a simple example.
Suppose, for one target node, the structural weights assigned to three source
nodes are
\(\omega_k^{\sharp}=(0.6,0.3,0.1)\), and the corresponding feature responses are
\((2,1,-1)\). Their aggregate contribution is
\(0.6(2)+0.3(1)+0.1(-1)=1.4\). Without normalization, the same contribution
could be written using feature responses
\((20,10,-10)\) and structural weights
\((0.06,0.03,0.01)\). The product is unchanged, but the two decompositions
assign entirely different numerical scales to the feature and structural
components. In particular, the second structural vector sums to \(0.1\)
rather than one, so its magnitude no longer admits the same probabilistic
interpretation. Under simplex-normalized RRWP, this rescaled decomposition is
not admissible, leaving the normalized structural weighting as the valid
representation.

This normalization places the topology coefficients on a common scale, so
$\theta_{t,k}$ can be interpreted as the fraction of structural influence that
feature $k$ assigns to walk length $t$. Without normalization, the magnitude
of the structural term is arbitrary because $f_k$ can be increased while
$\omega_k^{\sharp}$ is decreased by the same factor, or vice versa, without changing the
prediction. Normalization removes this freedom in scale. It also makes
$\mathbf P_k$ a convex combination of the walk operators
$\mathbf I,\mathbf M,\ldots,\mathbf M^{T-1}$, a property that will be useful
later in the stability analysis of RRWP aggregation.

However, fixing the overall scale does not guarantee that the individual hop
coefficients are uniquely determined. Two different coefficient vectors
$\boldsymbol{\theta}_k$ and $\boldsymbol{\theta}_k'$ may both lie on the
simplex and still produce the same structural operator when
$\mathbf I,\mathbf M,\ldots,\mathbf M^{T-1}$ are linearly dependent. The next
section studies exactly when this ambiguity occurs through the identifiability
of the hop coefficients.

\subsection{Identifiability of Random-Walk Hop Coefficients}
\label{subsec:app_hop_identifiability}

After fixing the overall scale of the topology function, the next question is
whether the individual random-walk coefficients are uniquely determined. For a
fixed feature $k$, consider the linear RRWP operator
\begin{equation}
\mathbf P_{\boldsymbol{\theta}}
:=
\sum_{t=0}^{T-1}
\theta_t \mathbf M^t,
\qquad
\boldsymbol{\theta}\in\Delta^{T-1},
\label{eq:app_rrwp_operator}
\end{equation}
where
\[
\Delta^{T-1}
=
\left\{
\boldsymbol{\theta}\in\mathbb R^T:
\theta_t\geq 0,\,
\sum_t \theta_t=1
\right\}
\]
is the probability simplex. Each coefficient $\theta_t$ represents the
relative contribution of the $t$-step random-walk operator. For this hop-level
interpretation to be meaningful, each structural operator should correspond to
a unique coefficient vector. If two different coefficient vectors produce the
same structural operator, then the contribution assigned to each hop is not
uniquely identifiable.

\begin{theorem}[Identifiability of Random-Walk Hop Coefficients]
\label{thm:app_hop_identifiability}
Assume that \(\mathbf M\) is row-stochastic. The mapping
\(\boldsymbol{\theta}\mapsto\mathbf P_{\boldsymbol{\theta}}\) is injective on
\(\Delta^{T-1}\) if and only if the matrices
\(\mathbf I,\mathbf M,\ldots,\mathbf M^{T-1}\) are linearly independent.
Equivalently,
$\mathbf P_{\boldsymbol{\theta}}
=
\mathbf P_{\boldsymbol{\theta}'} \text{ then }
\boldsymbol{\theta}
=
\boldsymbol{\theta}'$
for all
\(\boldsymbol{\theta},\boldsymbol{\theta}'\in\Delta^{T-1}\)
exactly when no nontrivial linear combination of the first \(T\) random-walk
powers vanishes.
\end{theorem}

\begin{proof}
Suppose first that
\(\mathbf I,\mathbf M,\ldots,\mathbf M^{T-1}\)
are linearly independent. If two coefficient vectors satisfy
\(\mathbf P_{\boldsymbol{\theta}}
=
\mathbf P_{\boldsymbol{\theta}'}\), then subtracting the two operators gives

\begin{equation}
\sum_{t=0}^{T-1}
(\theta_t-\theta_t')
\mathbf M^t
=
\mathbf 0.
\label{eq:app_hop_difference}
\end{equation}

By linear independence, every coefficient in this linear combination must be
zero. Hence
\(\theta_t-\theta_t'=0\) for every \(t\), and therefore
\(\boldsymbol{\theta}=\boldsymbol{\theta}'\). For the converse, suppose that
\(\mathbf I,\mathbf M,\ldots,\mathbf M^{T-1}\)
are linearly dependent. Then there exists a nonzero vector
\(\mathbf a=(a_0,\ldots,a_{T-1})\) such that

\begin{equation}
\sum_{t=0}^{T-1}
a_t\mathbf M^t
=
\mathbf 0.
\label{eq:app_hop_linear_dependence}
\end{equation}

Because \(\mathbf M\) is row-stochastic,
\(\mathbf M^t\mathbf 1=\mathbf 1\) for every \(t\). Multiplying
Eq.~\ref{eq:app_hop_linear_dependence} by \(\mathbf 1\) therefore gives

\begin{equation}
\mathbf 0
=
\sum_{t=0}^{T-1}
a_t\mathbf M^t\mathbf 1
=
\left(
\sum_{t=0}^{T-1}a_t
\right)\mathbf 1,
\end{equation}

so necessarily \(\sum_t a_t=0\). This property is important because it means
that moving in the direction \(\mathbf a\) preserves the simplex
normalization. Choose any coefficient vector in the interior of the simplex, for example
\(\bar{\boldsymbol{\theta}}
=(1/T,\ldots,1/T)\). Since all of its entries are strictly positive, there
exists a sufficiently small \(\varepsilon>0\) such that both
\(\bar{\boldsymbol{\theta}}\) and
\(\bar{\boldsymbol{\theta}}+\varepsilon\mathbf a\)
remain in \(\Delta^{T-1}\). They are distinct because
\(\mathbf a\neq0\), but

\begin{equation}
\begin{aligned}
\mathbf P_{\bar{\boldsymbol{\theta}}
+\varepsilon\mathbf a}
&=
\sum_{t=0}^{T-1}
(\bar{\theta}_t+\varepsilon a_t)\mathbf M^t
\\
&=
\sum_{t=0}^{T-1}
\bar{\theta}_t\mathbf M^t
+
\varepsilon
\sum_{t=0}^{T-1}
a_t\mathbf M^t
\\
&=
\mathbf P_{\bar{\boldsymbol{\theta}}}.
\end{aligned}
\end{equation}

Thus two distinct simplex coefficient vectors produce the same structural
operator, so the hop coefficients are not identifiable.
\end{proof}

The theorem shows that normalization alone is insufficient for hop-level
interpretability. The simplex fixes the overall scale of the topology
function, but uniqueness of the individual coefficients depends on the graph
through the algebraic relationships among
\(\mathbf I,\mathbf M,\mathbf M^2,\ldots\). If two walk lengths induce the
same, or an exactly redundant, structural operator, no prediction based on
their weighted sum can determine how much importance should be assigned to
one hop rather than the other.

\paragraph{A concrete example.}
Consider the two-node transition matrix

\[
\mathbf M
=
\begin{bmatrix}
1/2 & 1/2\\
1/2 & 1/2
\end{bmatrix}.
\]

After one step, the walker is already uniformly distributed over the two
nodes, and applying another transition changes nothing. Indeed,
\(\mathbf M^2=\mathbf M\). For \(T=3\), the RRWP operator becomes

\[
\mathbf P_{\boldsymbol{\theta}}
=
\theta_0\mathbf I
+
\theta_1\mathbf M
+
\theta_2\mathbf M^2
=
\theta_0\mathbf I
+
(\theta_1+\theta_2)\mathbf M.
\]

Hence the coefficient vectors
\(\boldsymbol{\theta}=(0.2,0.3,0.5)\) and
\(\boldsymbol{\theta}'=(0.2,0.6,0.2)\) produce exactly the same operator,
because both allocate total weight \(0.8\) to
\(\mathbf M\). One explanation would assign weight \(0.5\) to the two-hop
effect, while the other assigns only \(0.2\), yet the prediction mechanism is
identical. The distinction between one-hop and two-hop influence is therefore
not identifiable on this graph.

Theorem~\ref{thm:app_hop_identifiability} shows that hop coefficients
are identifiable only while the walk operators
$\mathbf I,\mathbf M,\ldots,\mathbf M^{T-1}$ remain linearly
independent. To determine how long this independence can hold,
we use the minimal polynomial $\mu_{\mathbf M}$ of $\mathbf M$,
defined as the monic polynomial of smallest degree satisfying
$\mu_{\mathbf M}(\mathbf M)=\mathbf 0$. Let
$q_M=\deg(\mu_{\mathbf M})$. By minimality,
$\mathbf I,\mathbf M,\ldots,\mathbf M^{q_M-1}$ are linearly
independent, whereas $\mathbf M^{q_M}$ can be written as a linear
combination of the preceding powers. Thus, $q_M$ gives the algebraic
limit on the number of walk-power operators that can remain
independently distinguishable, and therefore the maximum number
of RRWP hop coefficients that can be identifiable.

\begin{corollary}[Maximum Number of Algebraically Identifiable Hop Coefficients]
\label{cor:app_minimal_polynomial}
Assume that $\mathbf M$ is row-stochastic, and let
$q_M=\deg(\mu_{\mathbf M})$ be the degree of its minimal polynomial.
Then the coefficients of a $T$-term linear RRWP operator are
identifiable for $T\leq q_M$. For $T>q_M$, the mapping from simplex
coefficients to structural operators is not injective.
\end{corollary}

\begin{proof}
By the minimality of $\mu_{\mathbf M}$, no nonzero polynomial
of degree smaller than $q_M$ annihilates $\mathbf M$. Hence
$\mathbf I,\mathbf M,\ldots,\mathbf M^{q_M-1}$ are linearly
independent. If $T\leq q_M$, then
$\mathbf I,\ldots,\mathbf M^{T-1}$ form a subset of this independent
sequence, so Theorem~\ref{thm:app_hop_identifiability} gives
identifiability.

If $T>q_M$, the relation $\mu_{\mathbf M}(\mathbf M)=\mathbf 0$
gives a nontrivial linear dependence among
$\mathbf I,\mathbf M,\ldots,\mathbf M^{q_M}$, all of which are
included among $\mathbf I,\ldots,\mathbf M^{T-1}$.
The converse direction of
Theorem~\ref{thm:app_hop_identifiability} therefore gives
non-identifiability.
\end{proof}

Corollary~\ref{cor:app_minimal_polynomial} expresses the
identifiability limit through the degree of the minimal polynomial.
When $\mathbf M$ is diagonalizable, this algebraic limit can be
interpreted directly through its spectrum. Write
$\mathbf M=\mathbf V\boldsymbol{\Lambda}\mathbf V^{-1}$,
where the diagonal entries of $\boldsymbol{\Lambda}$ are the
eigenvalues of $\mathbf M$. Then
\[
\mathbf P_{\boldsymbol{\theta}}
=
\mathbf V
\left(
\sum_{t=0}^{T-1}
\theta_t\boldsymbol{\Lambda}^t
\right)
\mathbf V^{-1}.
\]
Thus, the hop coefficients define the polynomial
$g_{\boldsymbol{\theta}}(\lambda)
=\sum_{t=0}^{T-1}\theta_t\lambda^t$, and the structural operator
depends on this polynomial only through its values at the
eigenvalues of $\mathbf M$. Consequently, the number of distinct
eigenvalues determines how many walk-power directions can remain
independently distinguishable, leading to the following spectral
characterization.

\begin{corollary}[Spectral Characterization]
\label{cor:app_spectral_hop_identifiability}
Assume that $\mathbf M$ is row-stochastic and diagonalizable.
Then $q_M=\deg(\mu_{\mathbf M})$ equals the number of distinct
eigenvalues of $\mathbf M$, and the first $T$ random-walk powers
are linearly independent exactly when $T\leq q_M$.
Hence a linear RRWP parameterization with $T$ hop coefficients
is identifiable when $T\leq q_M$ and is non-identifiable when
$T>q_M$.
\end{corollary}

\begin{proof}
For a diagonalizable matrix, the minimal polynomial has one
linear factor for each distinct eigenvalue. Writing these
distinct eigenvalues as $\lambda_1,\ldots,\lambda_{q_M}$ gives
\[
\mu_{\mathbf M}(z)
=
\prod_{r=1}^{q_M}(z-\lambda_r),
\]
which has degree $q_M$. The result follows directly from
Corollary~\ref{cor:app_minimal_polynomial}.
\end{proof}

Corollary~\ref{cor:app_minimal_polynomial} expresses the identifiability limit through the degree of the
minimal polynomial. When $\mathbf M$ is diagonalizable, this algebraic limit
can be interpreted directly through its spectrum. Write
$\mathbf M=\mathbf V\boldsymbol{\Lambda}\mathbf V^{-1}$, where the diagonal
entries of $\boldsymbol{\Lambda}$ are the eigenvalues of $\mathbf M$. Then
\[
\mathbf P_{\boldsymbol{\theta}}
=
\mathbf V
\left(
\sum_{t=0}^{T-1}
\theta_t\boldsymbol{\Lambda}^t
\right)
\mathbf V^{-1}.
\]
Thus, the hop coefficients define the polynomial
$g_{\boldsymbol{\theta}}(\lambda)
=\sum_{t=0}^{T-1}\theta_t\lambda^t$, and the structural operator depends on
this polynomial only through its values at the eigenvalues of $\mathbf M$.
Consequently, the number of distinct eigenvalues determines how many
walk-power directions can remain independently distinguishable, leading to
the following spectral characterization.

\begin{corollary}[Spectral Characterization]
\label{cor:app_spectral_hop_identifiability}
If \(\mathbf M\) is diagonalizable and has \(q\) distinct eigenvalues, then
the first \(T\) random-walk powers are linearly independent exactly when
\(T\leq q\). Hence a linear RRWP parameterization with \(T\) hop coefficients
is identifiable when \(T\leq q\) and is non-identifiable when \(T>q\).
\end{corollary}

\begin{proof}
For a diagonalizable matrix, the minimal polynomial has one linear factor for
each distinct eigenvalue,
\(\mu_{\mathbf M}(z)=\prod_{r=1}^{q}(z-\lambda_r)\), and therefore has degree
\(q\). The result follows directly from
Corollary~\ref{cor:app_minimal_polynomial}.
\end{proof}

This spectral view provides a direct interpretation of the identifiability
limit. When the graph has many distinct eigenvalues, the walk operators
$\mathbf I,\mathbf M,\mathbf M^2,\ldots$ can exhibit more distinct behaviors,
making the effects of different walk lengths easier to distinguish. In
contrast, repeated eigenvalues reduce the number of independent walk
operators. Increasing $T$ beyond this limit may therefore introduce additional
hop coefficients whose individual effects cannot be uniquely identified.

Exact identifiability, however, is only the first requirement for a reliable
explanation. Even when the walk operators remain linearly independent, they
can become very similar after many random-walk steps. In this case, two
substantially different hop-weight vectors may produce nearly identical
structural operators, so a small perturbation can lead to a large change in
the estimated hop weights. We next study this stability aspect of
identifiability.

\subsection{Approximate Identifiability and Spectral Conditioning}
\label{subsec:app_approx_identifiability}

Exact identifiability guarantees that two distinct hop-weight vectors cannot
produce exactly the same structural operator. For reliable interpretation,
however, exact uniqueness is not sufficient. Two different coefficient vectors
may induce operators that are extremely close, so a small perturbation in the
learned operator can produce a large change in the recovered hop weights. We
therefore study the conditioning of the map
\(\boldsymbol{\theta}\mapsto
\sum_{t=0}^{T-1}\theta_t\mathbf M^t\).

Because both
\(\boldsymbol{\theta}\) and \(\boldsymbol{\theta}'\) lie on the probability
simplex, their difference
\(\boldsymbol{\delta}
:=\boldsymbol{\theta}-\boldsymbol{\theta}'\)
satisfies \(\mathbf 1^\top\boldsymbol{\delta}=0\). Define the tangent space
\(\mathcal T
:=\{\boldsymbol{\delta}\in\mathbb R^T:
\mathbf 1^\top\boldsymbol{\delta}=0\}\)
and the restricted conditioning constant

\begin{equation}
\gamma_T
:=
\inf_{\substack{
\boldsymbol{\delta}\in\mathcal T\\
\|\boldsymbol{\delta}\|_2=1}}
\left\|
\sum_{t=0}^{T-1}
\delta_t\mathbf M^t
\right\|_F.
\label{eq:app_conditioning_constant}
\end{equation}

The quantity \(\gamma_T\) measures the smallest observable change in the
structural operator produced by a unit change in the normalized hop
coefficients. Exact identifiability corresponds to \(\gamma_T>0\), whereas a
small positive value indicates an ill-conditioned decomposition: the
coefficients remain mathematically unique, but substantially different hop
allocations can generate nearly indistinguishable structural operators.

\begin{theorem}[Stability of Hop-Coefficient Recovery]
\label{thm:app_hop_stability}
Let \(\boldsymbol{\theta}^\star\in\Delta^{T-1}\) denote the true hop
coefficients and suppose an estimated structural operator satisfies
\[
\widehat{\mathbf P}
=
\mathbf P_{\boldsymbol{\theta}^\star}
+
\mathbf E,
\qquad
\|\mathbf E\|_F\leq\varepsilon.
\]
Let
\(\widehat{\boldsymbol{\theta}}\in\Delta^{T-1}\)
minimize
\(\|\widehat{\mathbf P}
-\mathbf P_{\boldsymbol{\theta}}\|_F\)
over the simplex. If \(\gamma_T>0\), then

\begin{equation}
\left\|
\widehat{\boldsymbol{\theta}}
-
\boldsymbol{\theta}^\star
\right\|_2
\leq
\frac{2\varepsilon}{\gamma_T}.
\label{eq:app_coefficient_stability}
\end{equation}
\end{theorem}

\begin{proof}
Since \(\boldsymbol{\theta}^\star\) is feasible for the optimization defining
\(\widehat{\boldsymbol{\theta}}\),

\[
\left\|
\widehat{\mathbf P}
-
\mathbf P_{\widehat{\boldsymbol{\theta}}}
\right\|_F
\leq
\left\|
\widehat{\mathbf P}
-
\mathbf P_{\boldsymbol{\theta}^\star}
\right\|_F
=
\|\mathbf E\|_F
\leq\varepsilon.
\]

Using the triangle inequality,
\begin{equation}
\begin{aligned}
\left\|
\mathbf P_{\widehat{\boldsymbol{\theta}}}
-
\mathbf P_{\boldsymbol{\theta}^\star}
\right\|_F
&\leq
\left\|
\mathbf P_{\widehat{\boldsymbol{\theta}}}
-
\widehat{\mathbf P}
\right\|_F
+
\left\|
\widehat{\mathbf P}
-
\mathbf P_{\boldsymbol{\theta}^\star}
\right\|_F
\\
&\leq
2\varepsilon.
\end{aligned}
\label{eq:app_operator_perturbation_bound}
\end{equation}

Now define
\(\boldsymbol{\delta}
=
\widehat{\boldsymbol{\theta}}
-
\boldsymbol{\theta}^\star\).
Since both coefficient vectors lie on the simplex,
\(\mathbf 1^\top\boldsymbol{\delta}=0\), so
\(\boldsymbol{\delta}\in\mathcal T\). Moreover,
\[
\mathbf P_{\widehat{\boldsymbol{\theta}}}
-
\mathbf P_{\boldsymbol{\theta}^\star}
=
\sum_{t=0}^{T-1}
\delta_t\mathbf M^t.
\]

By the definition of \(\gamma_T\), $\left\|
\mathbf P_{\widehat{\boldsymbol{\theta}}}
-
\mathbf P_{\boldsymbol{\theta}^\star}
\right\|_F
\geq
\gamma_T
\|\boldsymbol{\delta}\|_2.$ Combining this inequality with
Eq.~\ref{eq:app_operator_perturbation_bound} yields
\(\gamma_T
\|\widehat{\boldsymbol{\theta}}
-\boldsymbol{\theta}^\star\|_2
\leq2\varepsilon\), which proves
Eq.~\ref{eq:app_coefficient_stability}.
\end{proof}

The theorem gives \(\gamma_T\) a direct interpretability meaning. If
\(\gamma_T\) is large, a small perturbation of the structural operator can
produce only a small perturbation of the recovered hop weights. If
\(\gamma_T\) is close to zero, the same operator-level error may correspond to
a much larger change in the explanation. Thus, two trained models may implement
nearly identical structural transformations while assigning noticeably
different importance to individual walk lengths.

The conditioning constant can be computed from the geometry of the random-walk
operators. Define the Gram matrix
\(\mathbf G_T\in\mathbb R^{T\times T}\) by
\([\mathbf G_T]_{s,t}
=
\langle\mathbf M^s,\mathbf M^t\rangle_F\).
For any coefficient perturbation
\(\boldsymbol{\delta}\),

\begin{equation}
\left\|
\sum_{t=0}^{T-1}\delta_t\mathbf M^t
\right\|_F^2
=
\boldsymbol{\delta}^{\top}
\mathbf G_T
\boldsymbol{\delta}.
\label{eq:app_walk_gram}
\end{equation}

If \(\mathbf Q\in\mathbb R^{T\times(T-1)}\) has orthonormal columns spanning
\(\mathcal T\), then $\gamma_T^2
=
\lambda_{\min}
\left(
\mathbf Q^\top\mathbf G_T\mathbf Q
\right).$ Hence approximate identifiability is controlled by the smallest eigenvalue of
the walk-power Gram matrix after removing the irrelevant all-ones direction.
When this eigenvalue is small, some normalized redistribution of hop weights
changes the resulting structural operator only weakly.

A spectral interpretation is especially transparent when \(\mathbf M\) is
normal. Let
\(\mathbf M=\mathbf U\boldsymbol{\Lambda}\mathbf U^\ast\), with eigenvalues
\(\lambda_1,\ldots,\lambda_n\). Then

\begin{equation}
\left\|
\sum_{t=0}^{T-1}
\delta_t\mathbf M^t
\right\|_F^2
=
\sum_{r=1}^{n}
\left|
\sum_{t=0}^{T-1}
\delta_t\lambda_r^t
\right|^2.
\label{eq:app_spectral_conditioning}
\end{equation}

Thus, recovering the hop coefficients is equivalent to distinguishing
polynomials from their values on the graph spectrum. If different powers
\(\lambda_r^t\) produce sufficiently distinct spectral responses, the
coefficient decomposition is well conditioned. If the relevant spectral
responses become nearly redundant, the smallest singular direction becomes
small and the hop explanation becomes unstable.

\paragraph{A concrete example.}
Consider the symmetric two-state lazy random walk

\[
\mathbf M
=
\begin{bmatrix}
1-\eta & \eta\\
\eta & 1-\eta
\end{bmatrix},
\qquad
0<\eta<\frac{1}{2}.
\]

Its eigenvalues are \(1\) and
\(\rho=1-2\eta\), with \(0<\rho<1\). Consider two valid hop-weight vectors
that differ only by moving an amount \(\delta>0\) of weight from hop \(t\) to
hop \(t+1\). Their coefficient difference is proportional to
\(\mathbf e_t-\mathbf e_{t+1}\), while the corresponding difference in
structural operators is

\[
\delta
\left(
\mathbf M^t-\mathbf M^{t+1}
\right).
\]

Along the stationary eigenvector, both powers have eigenvalue \(1\), so their
difference is zero. Along the remaining eigenvector, the difference is
\(\rho^t-\rho^{t+1}
=(1-\rho)\rho^t\). Therefore,

\begin{equation}
\left\|
\mathbf M^t-\mathbf M^{t+1}
\right\|_F
=
(1-\rho)\rho^t.
\label{eq:app_adjacent_hop_decay_example}
\end{equation}

The coefficient vectors may differ by the same amount regardless of \(t\), but
their effect on the structural operator becomes exponentially smaller as
\(t\) increases. For example, if \(\rho=0.8\), then the difference between
hop \(1\) and hop \(2\) is proportional to \(0.16\), while the difference
between hop \(10\) and hop \(11\) is only about \(0.021\). A model can therefore
distinguish early walk lengths much more reliably than later ones even though
all coefficients remain algebraically identifiable.

This example exposes the practical limitation of exact identifiability.
Linear independence answers whether two hop decompositions can be exactly
equal, whereas \(\gamma_T\) measures how far apart their induced operators
must be. As random-walk powers become increasingly similar, \(\gamma_T\)
decreases and hop-level explanations become sensitive to optimization noise,
finite data, or small model perturbations. This phenomenon motivates a more
direct question: beyond what walk length do additional random-walk scales cease
to provide meaningfully distinguishable structural information? The next
section formalizes this limit through the effective structural horizon.

\section{EFFECTIVE STRUCTURAL HORIZON OF RANDOM-WALK EXPLANATIONS}
\label{sec:app_structural_horizon}

The previous analysis provides an algebraic limit on the number of hop
coefficients that can be uniquely identified. Within this identifiable range,
however, uniqueness does not guarantee that different walk lengths continue
to provide sufficiently distinct structural signals. As the random walk
progresses, it gradually loses information about its starting node and
approaches its stationary behavior, causing consecutive operators
$\mathbf M^t$ and $\mathbf M^{t+1}$ to become increasingly similar. This
raises a more practical question of how far hop-level distinctions remain
structurally meaningful within the identifiable regime. We address this
question by studying how the difference between successive random-walk powers
decays with $t$, and use this decay to define an effective structural horizon
beyond which separating individual hop effects becomes increasingly unstable
and less informative.

For completeness, we restate Theorem~\ref{thm:structural_horizon} below.
Its two claims are established separately in
Lemma~\ref{thm:app_hop_decay} and~\ref{thm:app_structural_horizon}.

\spectraldecay*

\subsection{Decay of Hop Distinguishability}
\label{subsec:app_hop_decay}

To quantify when two neighboring walk lengths become difficult to distinguish,
we study how much the corresponding random-walk operators change from one step
to the next. Let \(\boldsymbol{\pi}\) denote the stationary distribution of
the random walk and assume that \(\mathbf M\) is ergodic and reversible with
respect to \(\boldsymbol{\pi}\). We use the weighted inner product
\(\langle \mathbf u,\mathbf v\rangle_{\pi}
=\sum_i \pi_i u_i v_i\), with induced vector norm
\(\|\cdot\|_{2,\pi}\) and operator norm denoted by the same symbol. Under
reversibility, \(\mathbf M\) is self-adjoint in this space.

Let
\(\mathbf \Pi=\mathbf 1\boldsymbol{\pi}^{\top}\) denote the stationary
projection. For any node signal \(\mathbf z\),
\(\mathbf \Pi\mathbf z\) replaces the signal by its stationary weighted
average at every node. We define the distinguishability between two
consecutive walk lengths as

\begin{equation}
D_t
:=
\left\|
\mathbf M^{t+1}-\mathbf M^t
\right\|_{2,\pi}.
\label{eq:app_hop_distinguishability}
\end{equation}

A large \(D_t\) means that extending the walk from \(t\) to \(t+1\) steps can
substantially change the propagated structural signal. A small \(D_t\), in
contrast, means that the two hop lengths behave almost identically for every
normalized input signal.

Let the eigenvalues of \(\mathbf M\) be
\(1=\lambda_1,\lambda_2,\ldots,\lambda_n\), and define

\[
\rho
:=
\max_{r\geq2}|\lambda_r|.
\]

Ergodicity gives \(\rho<1\). The quantity \(\rho\) controls how quickly the
random walk loses information about its initial state.

\begin{lemma}[Decay of Hop Distinguishability]
\label{thm:app_hop_decay}
Suppose that \(\mathbf M\) is an ergodic reversible random-walk operator. Then

\begin{equation}
\left\|
\mathbf M^t-\mathbf\Pi
\right\|_{2,\pi}
=
\rho^t,
\label{eq:app_mixing_decay}
\end{equation}

and the difference between consecutive walk lengths satisfies

\begin{equation}
D_t
=
\left\|
\mathbf M^{t+1}-\mathbf M^t
\right\|_{2,\pi}
=
\max_{r\geq2}
|\lambda_r|^t|1-\lambda_r|
\leq
(1+\rho)\rho^t.
\label{eq:app_adjacent_hop_decay}
\end{equation}

Hence the maximum structural distinction between hop \(t\) and hop \(t+1\)
decays at least geometrically with the nontrivial spectral radius
\(\rho\).
\end{lemma}

\begin{proof}
Because \(\mathbf M\) is reversible, it admits an orthonormal eigenbasis
\(\{\mathbf v_r\}_{r=1}^n\) under
\(\langle\cdot,\cdot\rangle_{\pi}\). The stationary eigenvector is
\(\mathbf v_1=\mathbf 1\) with eigenvalue \(\lambda_1=1\), while all remaining
eigenvalues satisfy \(|\lambda_r|<1\).

Any signal \(\mathbf z\) can therefore be decomposed as

\begin{equation}
\mathbf z
=
\langle\mathbf z,\mathbf v_1\rangle_{\pi}\mathbf v_1
+
\sum_{r=2}^{n}
\langle\mathbf z,\mathbf v_r\rangle_{\pi}\mathbf v_r.
\label{eq:app_signal_spectral_decomposition}
\end{equation}

Applying \(t\) random-walk steps gives

\begin{equation}
\mathbf M^t\mathbf z
=
\langle\mathbf z,\mathbf v_1\rangle_{\pi}\mathbf v_1
+
\sum_{r=2}^{n}
\lambda_r^t
\langle\mathbf z,\mathbf v_r\rangle_{\pi}\mathbf v_r.
\label{eq:app_signal_after_t_steps}
\end{equation}

The stationary projector retains only the first component,

\[
\mathbf\Pi\mathbf z
=
\langle\mathbf z,\mathbf v_1\rangle_{\pi}\mathbf v_1.
\]

Subtracting the stationary component from
Eq.~\ref{eq:app_signal_after_t_steps} yields

\begin{equation}
(\mathbf M^t-\mathbf\Pi)\mathbf z
=
\sum_{r=2}^{n}
\lambda_r^t
\langle\mathbf z,\mathbf v_r\rangle_{\pi}\mathbf v_r.
\end{equation}

Since the eigenvectors are orthonormal in
\(L_2(\boldsymbol{\pi})\),

\begin{equation}
\begin{aligned}
\left\|
(\mathbf M^t-\mathbf\Pi)\mathbf z
\right\|_{2,\pi}^2
&=
\sum_{r=2}^{n}
|\lambda_r|^{2t}
\left|
\langle\mathbf z,\mathbf v_r\rangle_{\pi}
\right|^2
\\
&\leq
\rho^{2t}
\sum_{r=2}^{n}
\left|
\langle\mathbf z,\mathbf v_r\rangle_{\pi}
\right|^2
\\
&\leq
\rho^{2t}\|\mathbf z\|_{2,\pi}^2.
\end{aligned}
\end{equation}

Equality is attained by choosing \(\mathbf z\) to be an eigenvector whose
eigenvalue has magnitude \(\rho\). Therefore
\(\|\mathbf M^t-\mathbf\Pi\|_{2,\pi}=\rho^t\).

We next compare two consecutive walk lengths. Using the same eigenbasis,

\begin{equation}
\begin{aligned}
(\mathbf M^{t+1}-\mathbf M^t)\mathbf z
&=
\sum_{r=2}^{n}
\left(
\lambda_r^{t+1}-\lambda_r^t
\right)
\langle\mathbf z,\mathbf v_r\rangle_{\pi}\mathbf v_r
\\
&=
\sum_{r=2}^{n}
\lambda_r^t(\lambda_r-1)
\langle\mathbf z,\mathbf v_r\rangle_{\pi}\mathbf v_r.
\end{aligned}
\label{eq:app_adjacent_hop_spectral}
\end{equation}

The stationary component disappears because
\(1^{t+1}-1^t=0\). Taking the induced operator norm therefore gives

\[
D_t
=
\max_{r\geq2}
|\lambda_r|^t|1-\lambda_r|.
\]

Finally, since
\(|\lambda_r|\leq\rho\) and
\(|1-\lambda_r|\leq1+|\lambda_r|\leq1+\rho\),

\[
D_t
\leq
(1+\rho)\rho^t.
\]

Because \(\rho<1\), this upper bound converges geometrically to zero as
\(t\) increases.
\end{proof}

The theorem gives a direct interpretation of random-walk mixing in terms of
hop-level explanations. Early powers of \(\mathbf M\) can encode noticeably
different neighborhoods because the walk still retains information about its
starting node. As \(t\) increases, the nonstationary spectral components are
multiplied repeatedly by \(\lambda_r^t\) and gradually disappear. Once these
components become small, taking one additional walk step changes little, even
though \(\mathbf M^t\) and \(\mathbf M^{t+1}\) may remain algebraically
distinct.

\paragraph{A simple example.}
Consider a graph whose largest nontrivial eigenvalue magnitude is
\(\rho=0.8\). The theorem gives

\[
D_t
\leq
1.8(0.8)^t.
\]

At \(t=1\), this upper bound is \(1.44\), so one-hop and two-hop propagation
can still differ substantially. At \(t=10\), it decreases to approximately
\(0.193\), and at \(t=20\) to approximately \(0.021\). Thus the walk powers
remain mathematically different, but the amount of new structural information
introduced by one additional hop becomes progressively smaller.

The spectral rate also explains why the useful range of hop lengths varies
across graphs. If \(\rho\) is small, the random walk mixes rapidly and
neighboring long-range hops become similar after only a few steps. If
\(\rho\) is close to one, information about the starting region persists
longer, so structurally distinct walk lengths can remain visible at larger
\(t\). The appropriate range of interpretable hops is therefore not determined
by \(T\) alone; it depends on the mixing properties of the underlying graph.
This observation motivates defining a graph-dependent effective structural
horizon, which we formalize next.

\subsection{Effective Structural Horizon and Spectral Dependence}
\label{subsec:app_structural_horizon}

It remains to prove the second claim of
Theorem~\ref{thm:structural_horizon}.
The decay result above suggests that the number of useful walk lengths should
depend on how finely we wish to distinguish their structural effects. Rather
than treating every power of the random walk as equally informative, we define
a tolerance-dependent horizon beyond which one additional walk step changes the
structural operator by at most a prescribed amount.

For a tolerance \(\varepsilon>0\), define the effective structural horizon as

\begin{equation}
H_{\varepsilon}
:=
\min
\left\{
t\geq0:
\left\|
\mathbf M^{s+1}-\mathbf M^s
\right\|_{2,\pi}
\leq\varepsilon
\quad
\text{for all } s\geq t
\right\}.
\label{eq:app_structural_horizon}
\end{equation}

Thus, before \(H_{\varepsilon}\), at least some consecutive walk lengths differ
by more than the tolerance \(\varepsilon\). Beyond this point, taking one more
random-walk step changes the propagation operator by at most
\(\varepsilon\). The horizon should therefore be interpreted as a resolution
scale: it marks the point after which separating neighboring hop effects
becomes increasingly difficult at the chosen level of precision.

The spectral decay established in
Theorem~\ref{thm:app_hop_decay} gives a direct bound on this horizon.

\begin{lemma}[Spectral Bound on the Effective Structural Horizon]
\label{thm:app_structural_horizon}
Suppose that \(\mathbf M\) is an ergodic reversible random-walk operator and
let
\(\rho=\max_{r\geq2}|\lambda_r|<1\).
Then, for any \(\varepsilon>0\),

\begin{equation}
H_{\varepsilon}
\leq
\left\lceil
\frac{
\log\!\left((1+\rho)/\varepsilon\right)
}{
-\log\rho
}
\right\rceil
\label{eq:app_horizon_upper_bound}
\end{equation}

whenever \(\varepsilon<1+\rho\). Hence the effective structural horizon grows
as the nontrivial spectral radius approaches one.
\end{lemma}

\begin{proof}
From Theorem~\ref{thm:app_hop_decay},

\[
\left\|
\mathbf M^{s+1}-\mathbf M^s
\right\|_{2,\pi}
\leq
(1+\rho)\rho^s.
\]

Because \(0<\rho<1\), the right-hand side decreases monotonically with \(s\).
It is therefore sufficient to choose \(t\) such that

\[
(1+\rho)\rho^t
\leq
\varepsilon.
\]

Dividing by \(1+\rho\) gives
\(\rho^t\leq\varepsilon/(1+\rho)\).
Taking logarithms and using
\(\log\rho<0\),

\begin{equation}
\begin{aligned}
t\log\rho
&\leq
\log\left(\frac{\varepsilon}{1+\rho}\right),
\\
t
&\geq
\frac{
\log\left(\varepsilon/(1+\rho)\right)
}{
\log\rho
}
=
\frac{
\log\left((1+\rho)/\varepsilon\right)
}{
-\log\rho
}.
\end{aligned}
\end{equation}

Taking the smallest integer satisfying this inequality yields
Eq.~\ref{eq:app_horizon_upper_bound}. Since the same bound holds for every
\(s\geq t\), this value is a valid upper bound on
\(H_{\varepsilon}\).
\end{proof}

The theorem makes the dependence on graph structure explicit. When
\(\rho\) is small, nonstationary information decays quickly and the random
walk reaches its long-range behavior after relatively few steps. In this
regime, later hop operators become similar early, producing a short structural
horizon. When \(\rho\) is close to one, the walk retains information about its
starting region for longer, and distinct walk lengths remain structurally
distinguishable over a larger range.

This dependence becomes especially clear for slowly mixing graphs. When
\(\rho\) is close to one,
\(-\log\rho\approx1-\rho\), so the bound behaves approximately as

\[
H_{\varepsilon}
\lesssim
\frac{
\log\left((1+\rho)/\varepsilon\right)
}{
1-\rho
}.
\]

Thus a small spectral gap \(1-\rho\) corresponds to a longer structural
horizon, while a large spectral gap leads to faster loss of hop-specific
information.

\paragraph{A concrete example.}
Consider two graphs evaluated at the same tolerance
\(\varepsilon=0.05\). Suppose the first graph has
\(\rho=0.5\). Then

\[
H_{0.05}
\leq
\left\lceil
\frac{\log(1.5/0.05)}{-\log 0.5}
\right\rceil
=
5.
\]

The bound therefore guarantees that after roughly five steps, the difference
between consecutive walk operators is below the chosen tolerance. Now consider
a more slowly mixing graph with \(\rho=0.9\). The corresponding bound becomes

\[
H_{0.05}
\leq
\left\lceil
\frac{\log(1.9/0.05)}{-\log 0.9}
\right\rceil
=
35.
\]

The second graph can therefore sustain distinguishable long-range walk
behavior for substantially more steps. The same RRWP truncation depth \(T\)
can consequently have very different interpretations across the two graphs:
\(T=10\) may already extend well beyond the useful structural range of the
first graph while remaining within the informative range of the second.

This also clarifies the role of the RRWP truncation parameter \(T\). Increasing
\(T\) always exposes the model to additional walk lengths, but those additional
terms need not provide equally distinct structural information. If
\(T\gg H_{\varepsilon}\), many of the included powers lie in a regime where
neighboring operators differ only weakly. Their coefficients may still be
algebraically identifiable, yet interpreting small differences among those
late-hop weights becomes increasingly sensitive to finite data and
optimization noise. In contrast, choosing \(T\) within the graph's effective
structural range concentrates the hop decomposition on walk lengths that remain
meaningfully distinguishable.

The effective horizon is therefore not a universal maximum number of hops, nor
does it imply that all structural information disappears beyond
\(H_{\varepsilon}\). It quantifies a more specific property: beyond this
scale, one additional random-walk step contributes at most
\(\varepsilon\) change to the propagation operator. The value of
\(H_{\varepsilon}\) depends jointly on the graph spectrum and the resolution
at which hop-level differences are considered meaningful.

\section{INTERPRETABILITY AND EXPRESSIVITY OF NORMALIZED RRWP}
\label{sec:app_rrwp_expressivity}

The previous analysis establishes when random-walk hop coefficients can be
interpreted uniquely and over what range different walk lengths remain
meaningfully distinguishable. A separate question concerns the effect of the
normalization itself. \HARMONIA constrains the RRWP coefficients to be
nonnegative and to sum to one, so each structural operator is a convex
combination of random-walk powers. This gives the hop weights a consistent
interpretation and prevents arbitrary rescaling of the structural component,
but it also restricts the class of graph filters that can be represented. In particular, normalized RRWP naturally describes diffusion-like structural
aggregation, where information from different walk lengths is combined with
nonnegative weights. It cannot directly represent signed combinations such as
\(\mathbf I-\mathbf M\), which emphasize differences between local and
propagated signals and are commonly associated with high-pass behavior. Thus,
the same constraint that makes the hop decomposition easier to interpret also
imposes a specific inductive bias on the structural operator. In this section,
we characterize that bias through the spectral response of normalized RRWP,
show how normalization provides a useful stability property, and clarify the
resulting trade-off between interpretability and structural expressivity.

\subsection{Spectral View of Simplex-Normalized RRWP}
\label{subsec:app_rrwp_spectral}

The hop coefficients of RRWP admit a complementary interpretation in the
spectral domain. For feature \(k\), recall the normalized structural operator
\[
\mathbf P_k
=
\sum_{t=0}^{T-1}
\theta_{t,k}\mathbf M^t,
\qquad
\theta_{t,k}\geq0,
\qquad
\sum_{t=0}^{T-1}\theta_{t,k}=1.
\]
Suppose that \(\mathbf M\) is reversible and let
\((\lambda_r,\mathbf v_r)\) be one of its eigenpairs, so that
\(\mathbf M\mathbf v_r=\lambda_r\mathbf v_r\). Applying
\(\mathbf P_k\) to this eigenvector gives

\begin{equation}
\begin{aligned}
\mathbf P_k\mathbf v_r
&=
\sum_{t=0}^{T-1}
\theta_{t,k}\mathbf M^t\mathbf v_r
\\
&=
\sum_{t=0}^{T-1}
\theta_{t,k}\lambda_r^t\mathbf v_r
\\
&=
g_k(\lambda_r)\mathbf v_r,
\end{aligned}
\label{eq:app_rrwp_spectral_response}
\end{equation}

where

\[
g_k(\lambda)
:=
\sum_{t=0}^{T-1}
\theta_{t,k}\lambda^t
\]

is the spectral response induced by the hop weights of feature \(k\).
Thus, the same coefficients that describe how much weight \HARMONIA assigns to
different walk lengths also determine how strongly each spectral mode of the
graph is preserved.

\begin{proposition}[Spectral Response of Normalized RRWP]
\label{prop:app_rrwp_spectral_response}
Let \(\mathbf P_k\) be a simplex-normalized RRWP operator. For every
\(\lambda\in[-1,1]\), $g_k(1)=1, 
|g_k(\lambda)|\leq1.$ If, in addition, the random walk is lazy so that
\(\lambda\in[0,1]\), then $0\leq g_k(\lambda)\leq1,$ and \(g_k(\lambda)\) is nondecreasing on \([0,1]\).
\end{proposition}

\begin{proof}
The normalization of the hop weights immediately gives

\[
g_k(1)
=
\sum_{t=0}^{T-1}\theta_{t,k}
=
1.
\]

For any \(\lambda\in[-1,1]\), the triangle inequality yields

\begin{equation}
\begin{aligned}
|g_k(\lambda)|
&=
\left|
\sum_{t=0}^{T-1}
\theta_{t,k}\lambda^t
\right|
\\
&\leq
\sum_{t=0}^{T-1}
\theta_{t,k}|\lambda|^t
\\
&\leq
\sum_{t=0}^{T-1}
\theta_{t,k}
=
1.
\end{aligned}
\end{equation}

If the walk is lazy, then \(\lambda\in[0,1]\), so every term
\(\theta_{t,k}\lambda^t\) is nonnegative and therefore
\(g_k(\lambda)\geq0\). Moreover,

\[
g_k'(\lambda)
=
\sum_{t=1}^{T-1}
t\,\theta_{t,k}\lambda^{t-1}
\geq0,
\]

which shows that the response is nondecreasing on \([0,1]\).
\end{proof}

The proposition clarifies the inductive bias introduced by simplex
normalization. For a lazy random walk, eigenvalues close to one correspond to
slowly varying graph modes that persist under repeated diffusion, whereas
smaller eigenvalues decay more rapidly. Since \(g_k(\lambda)\) is
nondecreasing and satisfies \(g_k(1)=1\), normalized RRWP tends to preserve
these slowly varying modes while attenuating modes associated with smaller
eigenvalues. Its spectral behavior is therefore naturally diffusion-like.

\paragraph{A concrete example.}
Suppose a feature uses three walk lengths with
\(\boldsymbol{\theta}_k=(0.2,0.3,0.5)\). Its spectral response is

\[
g_k(\lambda)
=
0.2+0.3\lambda+0.5\lambda^2.
\]

For the stationary mode,
\(g_k(1)=1\). A mode with eigenvalue \(0.5\) is scaled by
\(g_k(0.5)=0.475\), while a mode with eigenvalue \(0\) is scaled by only
\(g_k(0)=0.2\). The same hop mixture therefore preserves a slowly varying
mode much more strongly than a rapidly decaying one. Changing the hop weights
changes this spectral preference: placing more mass on larger \(t\) makes the
response decay more strongly away from \(\lambda=1\), corresponding to a
greater emphasis on longer-range diffusion.

The simplex constraint also reveals what normalized RRWP cannot represent.
Because all polynomial coefficients are nonnegative, it cannot directly form
arbitrary signed combinations of walk operators. For example, the structural
difference operator

\[
\mathbf I-\mathbf M
\]

has spectral response \(1-\lambda\) and requires a negative coefficient on
\(\mathbf M\). This response suppresses the stationary mode and emphasizes
variation between a node signal and its propagated version, which is the
opposite of the diffusion-like behavior encouraged by normalized RRWP.
Consequently, simplex normalization provides an interpretable and controlled
family of structural filters, but not the full class of polynomial graph
filters. The next subsection examines the stability gained from this
restriction and the corresponding cost in structural expressivity.

\subsection{Stability and the Interpretability--Expressivity Trade-off}
\label{subsec:app_rrwp_tradeoff}

The simplex constraint does more than make the hop coefficients easier to
interpret. Because a normalized RRWP operator is a convex combination of
random-walk powers, it also inherits the non-expansive behavior of Markov
propagation. This prevents the structural component from arbitrarily
amplifying perturbations in the propagated feature signal.

Let
\[
\mathbf P_k
=
\sum_{t=0}^{T-1}
\theta_{t,k}\mathbf M^t,
\qquad
\theta_{t,k}\geq0,
\qquad
\sum_{t=0}^{T-1}\theta_{t,k}=1.
\]
Since \(\mathbf M\) is row-stochastic, every power \(\mathbf M^t\) is
row-stochastic, and therefore so is \(\mathbf P_k\).

\begin{theorem}[Non-Expansiveness of Normalized RRWP]
\label{thm:app_rrwp_stability}
Let \(\mathbf M\) be a row-stochastic transition matrix and let
\(\mathbf P_k\) be a simplex-normalized RRWP operator. Then, for any two node
signals \(\mathbf z,\mathbf z'\in\mathbb R^n\),

\begin{equation}
\left\|
\mathbf P_k\mathbf z
-
\mathbf P_k\mathbf z'
\right\|_{\infty}
\leq
\left\|
\mathbf z-\mathbf z'
\right\|_{\infty}.
\label{eq:app_rrwp_nonexpansive}
\end{equation}

Hence normalized RRWP cannot increase the maximum node-wise magnitude of an
input perturbation.
\end{theorem}

\begin{proof}
Let
\(\mathbf e=\mathbf z-\mathbf z'\). Since
\(\mathbf P_k\) is row-stochastic, its entries are nonnegative and every row
sums to one. For each node \(i\),

\begin{equation}
\begin{aligned}
\left|
[\mathbf P_k\mathbf e]_i
\right|
&=
\left|
\sum_{j=1}^{n}
[\mathbf P_k]_{ij}e_j
\right|
\\
&\leq
\sum_{j=1}^{n}
[\mathbf P_k]_{ij}|e_j|
\\
&\leq
\|\mathbf e\|_{\infty}
\sum_{j=1}^{n}
[\mathbf P_k]_{ij}
\\
&=
\|\mathbf e\|_{\infty}.
\end{aligned}
\end{equation}

Taking the maximum over \(i\) gives
\(\|\mathbf P_k\mathbf e\|_{\infty}
\leq\|\mathbf e\|_{\infty}\), and substituting
\(\mathbf e=\mathbf z-\mathbf z'\) proves
Eq.~\ref{eq:app_rrwp_nonexpansive}.
\end{proof}

The theorem has a simple interpretation. Each output value produced by
\(\mathbf P_k\) is a convex average of propagated input values. Consequently,
a perturbation of magnitude at most \(\delta\) at the input cannot become a
perturbation larger than \(\delta\) solely through the normalized structural
operator. This property is particularly useful for an interpretable additive
model, since the magnitude attributed to structural aggregation cannot be
artificially increased through unconstrained hop coefficients.

\paragraph{A concrete example.}
Suppose the propagated values of one feature at three nodes are
\((2,4,6)\), and one row of the learned structural operator is
\((0.2,0.5,0.3)\). The aggregated value is
\(0.2(2)+0.5(4)+0.3(6)=4.2\), which lies within the range of the input values.
If every input is perturbed by at most \(0.1\), then the aggregated value can
also change by at most \(0.1\). In contrast, an unconstrained structural
combination with coefficients such as \((2,-1,1)\) can amplify small
differences through cancellation and rescaling. The simplex constraint rules
out this behavior by forcing structural aggregation to remain a convex
combination.

This stability comes with a corresponding restriction on expressivity. To make
the distinction explicit, consider the unrestricted polynomial family

\[
\mathcal F_T^{\mathrm{poly}}
=
\left\{
\sum_{t=0}^{T-1}\beta_t\mathbf M^t:
\beta_t\in\mathbb R
\right\},
\]

and the simplex-normalized RRWP family

\[
\mathcal F_T^{\mathrm{RRWP}}
=
\left\{
\sum_{t=0}^{T-1}\theta_t\mathbf M^t:
\boldsymbol{\theta}\in\Delta^{T-1}
\right\}.
\]

By construction,
\(\mathcal F_T^{\mathrm{RRWP}}
\subseteq
\mathcal F_T^{\mathrm{poly}}\).
The inclusion is generally strict because the normalized family permits only
nonnegative coefficients that sum to one.

For example, consider the difference operator
\(\mathbf I-\mathbf M\). Its action on a node signal compares the original
signal with its one-step propagated version and therefore emphasizes local
variation rather than smoothing it. Its polynomial coefficients are
\((1,-1,0,\ldots,0)\), which violate both nonnegativity and simplex
normalization. If
\(\mathbf I,\mathbf M,\ldots,\mathbf M^{T-1}\) are linearly independent,
this coefficient representation is unique, and no simplex coefficient vector
can produce the same operator. Thus normalized RRWP cannot directly realize
this high-pass structural response.

The same distinction is visible spectrally. An unrestricted polynomial filter
may have a response such as
\(g(\lambda)=1-\lambda\), which suppresses the stationary mode
\(g(1)=0\) while emphasizing modes associated with smaller eigenvalues.
Simplex-normalized RRWP instead satisfies \(g_k(1)=1\); for a lazy random walk,
its response is nonnegative and nondecreasing on \([0,1]\). The normalization
therefore favors diffusion-like aggregation in which slowly varying structural
modes are preserved rather than explicitly contrasted or canceled.

This restriction is intentional rather than a claim of universal graph-filter
expressivity. \HARMONIA uses RRWP primarily as an interpretable structural
weighting mechanism: \(\theta_{t,k}\) should describe how structural influence
is distributed across walk lengths. Allowing arbitrary signed coefficients
would enlarge the class of representable filters, but a coefficient could no
longer be interpreted directly as a nonnegative share of structural influence,
and the resulting operator would lose the convex-averaging stability established
above.

The simplex constraint therefore induces a clear trade-off. It gives each hop
coefficient a common scale, preserves the Markov interpretation of structural
aggregation, and guarantees non-expansive propagation in the
\(\ell_\infty\) norm. In exchange, it excludes signed combinations of walk
operators and therefore represents a narrower family than unrestricted
polynomial graph filters. This trade-off reflects the design goal of \HARMONIA:
the structural component is intended to remain directly inspectable and stable,
rather than to reproduce the full expressivity of an unconstrained spectral
graph filter.

\paragraph{Why linear RRWP?}
The restriction to a linear RRWP topology function is also important for both
interpretability and computation. With
\(\omega_k^{\sharp}(\mathbf p_{ij})
=\sum_{t=0}^{T-1}\theta_{t,k}p_{ij}^{(t)}\), the structural contribution can
be written as
\[
[h_i]_k
=
\sum_{t=0}^{T-1}
\theta_{t,k}[\mathbf M^t\mathbf z_k]_i,
\]
where \([\mathbf z_k]_j=f_k(x_{j,k})\). Each walk length therefore contributes
an explicit additive term, so the hop-level explanation follows directly from
the prediction mechanism. A nonlinear topology function could capture interactions among different
random-walk statistics, but these interactions would no longer admit the same
unique hop-wise decomposition. More importantly, the linear form allows the
sum over walk lengths to be separated from the sum over source nodes, which is
precisely the algebraic property used by Sparse RRWP Aggregation. \HARMONIA
therefore adopts linear RRWP not for maximal structural expressivity, but to
preserve exact hop-level interpretation and enable sparse computation.

\subsection{Exact Additive Explanations}
\label{subsec:app_exact_explanations}

The additive structure of \HARMONIA makes it possible to explain a prediction
using the same quantities that produce it. Consider the logit for class \(c\)
at target node \(i\),

\begin{equation}
\ell_{i,c}
=
\beta_c
+
\sum_{k=1}^{d}
w_{k,c}
\sum_{j\in\mathcal V}
\omega_k^{\sharp}(\mathbf p_{ij})
f_k(x_{j,k}),
\label{eq:app_exact_logit}
\end{equation}

where \(f_k(x_{j,k})\) is the response of feature \(k\) at source node \(j\),
\(\omega_k^{\sharp}(\mathbf p_{ij})\) determines how that response is transmitted
structurally to node \(i\), and \(w_{k,c}\) maps the resulting feature effect
to class \(c\). We let the contribution of source node \(j\) and feature \(k\) as
\[
C_{i\leftarrow j,k,c}
:=
w_{k,c}\,
\omega_k^{\sharp}(\mathbf p_{ij})\,
f_k(x_{j,k}).
\]
\begin{proposition}[Exact Contribution Completeness]
\label{prop:app_exact_completeness}
The quantities \(C_{i\leftarrow j,k,c}\) form an exact additive decomposition
of the class logit:
\begin{equation}
\ell_{i,c}-\beta_c
=
\sum_{k=1}^{d}
\sum_{j\in\mathcal V}
C_{i\leftarrow j,k,c}.
\label{eq:app_exact_completeness}
\end{equation}
Consequently, summing over source nodes gives a feature-level decomposition,
while summing over features gives a source-node decomposition, without changing
the total explained logit.
\end{proposition}

\begin{proof}
Substituting the definition of
\(C_{i\leftarrow j,k,c}\) gives

\begin{equation}
\begin{aligned}
\sum_{k=1}^{d}
\sum_{j\in\mathcal V}
C_{i\leftarrow j,k,c}
&=
\sum_{k=1}^{d}
\sum_{j\in\mathcal V}
w_{k,c}
\omega_k^{\sharp}(\mathbf p_{ij})
f_k(x_{j,k})
\\
&=
\sum_{k=1}^{d}
w_{k,c}
\left(
\sum_{j\in\mathcal V}
\omega_k^{\sharp}(\mathbf p_{ij})
f_k(x_{j,k})
\right)
\\
&=
\ell_{i,c}-\beta_c,
\end{aligned}
\end{equation}

which proves the result. The feature-level and source-level decompositions
follow by regrouping the same finite sum.
\end{proof}

This completeness property means that the explanation is not reconstructed by
a separate attribution method. After removing the bias term, every reported
feature--node contribution is an actual term used in the computation of the
logit. The contributions are signed: depending on
\(w_{k,c}\) and \(f_k(x_{j,k})\), a term may support or oppose class \(c\),
even though the normalized structural weights themselves remain nonnegative.

\paragraph{Feature--hop structural contributions.}
The linear RRWP parameterization provides a finer decomposition across
random-walk lengths. Since $\omega_k^{\sharp}(\mathbf p_{ij})
=
\sum_{t=0}^{T-1}
\theta_{t,k}(\mathbf M^t)_{ij},$ let define
\begin{equation}
    C_{i\leftarrow j,k,t,c}
:=
w_{k,c}\,
\theta_{t,k}\,
(\mathbf M^t)_{ij}\,
f_k(x_{j,k}).
\end{equation}
Then
\begin{equation}
\begin{aligned}
C_{i\leftarrow j,k,c}
&=
w_{k,c}
\left(
\sum_{t=0}^{T-1}
\theta_{t,k}(\mathbf M^t)_{ij}
\right)
f_k(x_{j,k})
\\
&=
\sum_{t=0}^{T-1}
C_{i\leftarrow j,k,t,c},
\end{aligned}
\label{eq:app_feature_hop_contribution}
\end{equation}
and therefore
\begin{equation}
\ell_{i,c}-\beta_c
=
\sum_{k=1}^{d}
\sum_{j\in\mathcal V}
\sum_{t=0}^{T-1}
C_{i\leftarrow j,k,t,c}.
\label{eq:app_feature_node_hop_completeness}
\end{equation}
This decomposition separates two notions that are easy to conflate.
The coefficient \(\theta_{t,k}\) describes how feature \(k\) allocates its
structural weighting across walk lengths, but it is not itself the realized
contribution of hop \(t\) to class \(c\). The latter also depends on the
feature response, the graph-specific propagation through
\((\mathbf M^t)_{ij}\), and the class coefficient \(w_{k,c}\). Equivalently, if
\([\mathbf z_k]_j=f_k(x_{j,k})\), the contribution of feature \(k\) through
hop \(t\), aggregated over all source nodes, is $C_{i,k,t,c}
=
w_{k,c}\theta_{t,k}
[\mathbf M^t\mathbf z_k]_i.$ Thus \HARMONIA supports exact explanations at multiple resolutions: by original
feature, by source node, and by random-walk length. These quantities are all
obtained by regrouping the same additive terms used by the predictor, rather
than by introducing an auxiliary explanation model. The completeness statement applies to the pre-activation logit. If the logits
are subsequently transformed by a sigmoid or softmax, the resulting class
probabilities are nonlinear functions of these additive terms and therefore do
not admit the same exact additive decomposition.

\section{SPARSE RRWP AGGREGATION}
\label{sec:app_sra}

RRWP provides a rich description of multi-hop and multi-path relationships, but
its conventional pairwise form is expensive to materialize on large graphs.
For \(n\) nodes and \(T\) walk lengths, storing the full collection of
descriptors \(\mathbf p_{ij}\) requires \(O(Tn^2)\) memory, even when the
underlying graph is sparse. Applying these pairwise structural weights across
all feature dimensions further makes the computation grow with the number of
node pairs rather than the number of observed edges. The linear RRWP parameterization used by \HARMONIA allows this cost to be
avoided exactly. Instead of first constructing every pairwise random-walk
descriptor and then aggregating over source nodes, the same computation can be
reordered into repeated sparse propagation with the transition matrix
\(\mathbf M\). This leads to \emph{Sparse RRWP Aggregation} (SRA), which
computes the identical linear RRWP operator without materializing dense
node-pair representations. The purpose of this section is to establish this
equivalence formally and to characterize the resulting reduction in time and
memory complexity. 
\subsection{Exact Equivalence of Dense RRWP and Sparse Aggregation}
\label{subsec:app_sra_equivalence}

The pairwise formulation of RRWP suggests that the full collection of
random-walk probabilities must first be constructed before feature aggregation.
For the linear topology function used by \HARMONIA, however, this is
unnecessary. The aggregation can be computed by propagating the feature
responses directly through the graph. For feature \(k\), define the node-wise response vector
\(\mathbf z_k\in\mathbb R^n\) by
\([\mathbf z_k]_j=f_k(x_{j,k})\). The dense RRWP aggregation at node \(i\) is

\begin{equation}
[h_i]_k^{\mathrm{dense}}
=
\sum_{j\in\mathcal V}
\left(
\sum_{t=0}^{T-1}
\theta_{t,k}(\mathbf M^t)_{ij}
\right)
[\mathbf z_k]_j .
\label{eq:app_dense_rrwp}
\end{equation}

Instead of explicitly constructing the matrices
\(\mathbf M^0,\ldots,\mathbf M^{T-1}\), SRA recursively propagates the
feature-response vector. Let

\begin{equation}
\mathbf z_k^{(0)}=\mathbf z_k,
\qquad
\mathbf z_k^{(t)}
=
\mathbf M\mathbf z_k^{(t-1)},
\quad
t=1,\ldots,T-1.
\label{eq:app_sra_recursion}
\end{equation}

The final structural representation is then obtained by combining the
propagated signals using the same hop coefficients,

\begin{equation}
[h_i]_k^{\mathrm{SRA}}
=
\sum_{t=0}^{T-1}
\theta_{t,k}
[\mathbf z_k^{(t)}]_i.
\label{eq:app_sra_output}
\end{equation}

The following result shows that this sparse computation is exactly equivalent
to the dense pairwise formulation.

\begin{theorem}[Exact Equivalence of SRA]
\label{thm:app_sra_equivalence}
For any transition matrix \(\mathbf M\), feature-response vector
\(\mathbf z_k\), and hop coefficients
\(\{\theta_{t,k}\}_{t=0}^{T-1}\), the recursion in
Eq.~\eqref{eq:app_sra_recursion} satisfies $\mathbf z_k^{(t)}
=
\mathbf M^t\mathbf z_k$ for every $t=0,\ldots,T-1.$ Consequently, for every node \(i\) and feature \(k\), $[h_i]_k^{\mathrm{SRA}}
=
[h_i]_k^{\mathrm{dense}}.$ \end{theorem}

\begin{proof}
We first show that the recursively propagated signal at step \(t\) is exactly
the result of applying the \(t\)-step random-walk operator. For \(t=0\), by definition, $\mathbf z_k^{(0)}
=
\mathbf z_k
=
\mathbf M^0\mathbf z_k.$ Now suppose that
\(\mathbf z_k^{(t-1)}
=
\mathbf M^{t-1}\mathbf z_k\).
Using the SRA recursion,

\begin{equation}
\begin{aligned}
\mathbf z_k^{(t)}
&=
\mathbf M\mathbf z_k^{(t-1)}
\\
&=
\mathbf M
\left(
\mathbf M^{t-1}\mathbf z_k
\right)
\\
&=
\mathbf M^t\mathbf z_k.
\end{aligned}
\end{equation}

Therefore,
\(\mathbf z_k^{(t)}=\mathbf M^t\mathbf z_k\)
for every \(t\) by induction. Substituting this identity into
Eq.~\eqref{eq:app_sra_output} gives

\begin{equation}
\begin{aligned}
[h_i]_k^{\mathrm{SRA}}
&=
\sum_{t=0}^{T-1}
\theta_{t,k}
[\mathbf M^t\mathbf z_k]_i
\\
&=
\sum_{t=0}^{T-1}
\theta_{t,k}
\sum_{j\in\mathcal V}
(\mathbf M^t)_{ij}
[\mathbf z_k]_j
\\
&=
\sum_{j\in\mathcal V}
\sum_{t=0}^{T-1}
\theta_{t,k}
(\mathbf M^t)_{ij}
[\mathbf z_k]_j
\\
&=
\sum_{j\in\mathcal V}
\left(
\sum_{t=0}^{T-1}
\theta_{t,k}
(\mathbf M^t)_{ij}
\right)
[\mathbf z_k]_j
\\
&=
[h_i]_k^{\mathrm{dense}}.
\end{aligned}
\end{equation}

The third equality only exchanges two finite sums, while the final equality
follows from the definition of dense RRWP aggregation in
Eq.~\eqref{eq:app_dense_rrwp}. Hence SRA and dense RRWP produce exactly the
same representation.
\end{proof}

\textbf{A Simple Example.} The equivalence is easiest to see in a small example. Suppose \(T=3\) and
feature \(k\) uses hop weights
\((\theta_{0,k},\theta_{1,k},\theta_{2,k})
=(0.2,0.3,0.5)\). The dense formulation applies the operator

\[
0.2\mathbf I
+
0.3\mathbf M
+
0.5\mathbf M^2
\]

to \(\mathbf z_k\). SRA instead computes
\(\mathbf z_k^{(0)}=\mathbf z_k\),
\(\mathbf z_k^{(1)}=\mathbf M\mathbf z_k\), and
\(\mathbf z_k^{(2)}=\mathbf M\mathbf z_k^{(1)}
=\mathbf M^2\mathbf z_k\), before forming

\[
0.2\mathbf z_k^{(0)}
+
0.3\mathbf z_k^{(1)}
+
0.5\mathbf z_k^{(2)}.
\]

The two expressions are identical. The difference lies only in how they are
evaluated: the dense formulation constructs pairwise walk operators, whereas
SRA applies the same operators implicitly through repeated propagation. This exact equivalence is important for the interpretation of \HARMONIA.
Replacing dense RRWP with SRA does not alter the learned hop coefficients,
the feature-wise structural contributions, or the resulting prediction. SRA
changes only the computational realization of the linear RRWP operator. The
next subsection quantifies how this reformulation changes the time and memory
requirements on sparse graphs.

\subsection{Time and Memory Complexity}
\label{subsec:app_sra_complexity}

The exact equivalence above substantially reduces the computational cost
when the underlying graph is sparse. Let \(n=|\mathcal V|\), let
\(\mathcal E\) denote the graph edge set, let \(d\) be the number of
input features, and let \(T\) be the number of random-walk orders.
Define
\[
\mathcal S_M
:=
\{(i,j):M_{ij}\neq0\},
\qquad
E_M
:=
|\mathcal S_M|
=
\operatorname{nnz}(\mathbf M).
\]
We compare only structural aggregation; evaluating the feature
functions \(f_k\) is common to both implementations.

In the explicit dense RRWP formulation, the structural descriptor
for every ordered node pair contains \(T\) entries,
\[
\bigl\{
(\mathbf M^t)_{ij}
:
i,j\in\mathcal V,\;
t=0,\ldots,T-1
\bigr\}.
\]
Materializing this tensor requires \(O(Tn^2)\) memory.
Applying its \(T\) pairwise channels to all \(d\) feature-response
vectors requires \(O(Tn^2d)\) arithmetic, excluding the cost of
constructing the matrix powers.

SRA avoids the node-pair representation entirely. Stack the
feature-response vectors into
\(\mathbf Z^{(0)}\in\mathbb R^{n\times d}\), where
\([\mathbf Z^{(0)}]_{i,k}=f_k(x_{i,k})\).
The propagated representations are computed recursively as
\begin{equation}
\mathbf Z^{(t)}
=
\mathbf M\mathbf Z^{(t-1)},
\qquad
t=1,\ldots,T-1.
\label{eq:app_sra_matrix_recursion}
\end{equation}
Each sparse propagation requires \(O(E_Md)\) arithmetic,
while accumulating the feature-specific weighted outputs requires
\(O(nd)\) arithmetic per walk order. The total structural
aggregation cost is therefore \(O(T(E_M+n)d)\), which reduces
to \(O(TE_Md)\) when \(E_M\ge n\), as for a row-stochastic
transition matrix with no zero rows.

\begin{proposition}[Full version of Proposition~\ref{prop:sra_exactness}]
\label{prop:app_sra_complexity}
Let \(E_M=\operatorname{nnz}(\mathbf M)\).
For \(T\) random-walk orders and \(d\) feature-response channels,
explicit dense pairwise RRWP materialization requires
\(O(Tn^2)\) memory, and its direct application requires
\(O(Tn^2d)\) arithmetic, excluding matrix-power construction.

Sparse RRWP Aggregation computes the same linear operator in
\(O(T(E_M+n)d)\) time without materializing any matrix power.
When \(E_M\ge n\), this simplifies to \(O(TE_Md)\);
if additionally \(E_M=\Theta(|\mathcal E|)\), it becomes
\(O(T|\mathcal E|d)\).

A streaming forward pass uses \(O(E_M+nd)\) working memory,
plus \(O(Td)\) storage for the hop coefficients.
If all propagated states are retained, these requirements become
\(O(E_M+Tnd)\), plus the same coefficient storage.
These bounds exclude other model parameters, feature-network
activations, and optimizer state.
Hence, for fixed \(T\) and \(d\), SRA has linear structural
time and memory complexity in \(n\) whenever \(E_M=O(n)\).
\end{proposition}

\begin{proof}
At walk order \(t\), the direct RRWP formulation requires
\(\mathbf M^t\mathbf Z^{(0)}\).
An explicit dense pairwise realization stores
\(\mathbf M^0,\ldots,\mathbf M^{T-1}\) as \(T\) dense
\(n\times n\) arrays, requiring \(O(Tn^2)\) memory.
Even when \(\mathbf M\) is sparse, its higher powers can become
dense because multi-step paths connect nodes that are not
directly adjacent.
Multiplying each dense operator by
\(\mathbf Z^{(0)}\in\mathbb R^{n\times d}\) requires
\(O(n^2d)\) arithmetic, giving \(O(Tn^2d)\) over all walk orders.

SRA instead computes
\[
\mathbf Z^{(t)}
=
\mathbf M\mathbf Z^{(t-1)}
=
\mathbf M^t\mathbf Z^{(0)}.
\]
The second equality follows by induction on \(t\), so the
propagated representations agree exactly at every walk order.
Applying the same hop coefficients therefore gives the same
aggregated output:
\[
[\mathbf H]_{i,k}
=
\sum_{t=0}^{T-1}
\theta_{t,k}[\mathbf Z^{(t)}]_{i,k}.
\]

For each target node \(i\), sparse propagation evaluates
\[
[\mathbf Z^{(t)}]_{i,:}
=
\sum_{j:(i,j)\in\mathcal S_M}
M_{ij}[\mathbf Z^{(t-1)}]_{j,:}.
\]
Each of the \(E_M\) nonzero transitions contributes one
scalar--vector multiplication and accumulation over \(d\)
channels. Thus, the \(T-1\) propagation steps require
\(O((T-1)E_Md)\) arithmetic.
Accumulating the weighted representations over all \(T\)
orders requires an additional \(O(Tnd)\), yielding
\(O(T(E_M+n)d)\) total time.
When \(E_M\ge n\), this is \(O(TE_Md)\).
The further simplification to \(O(T|\mathcal E|d)\) holds
when \(E_M=\Theta(|\mathcal E|)\); any added self-transitions
are included in \(E_M\).

A sparse representation of \(\mathbf M\) requires
\(O(E_M+n)\) storage for its values and indices.
A streaming forward pass additionally retains current and next
propagated states and an output accumulator, each of size
\(n\times d\).
Since \(nd\) dominates \(n\), the working memory is
\(O(E_M+nd)\).
Including the \(T\times d\) hop coefficients gives
\(O(E_M+nd+Td)\) total structural storage.

If all states
\(\mathbf Z^{(0)},\ldots,\mathbf Z^{(T-1)}\)
are retained for backpropagation or hop-wise inspection,
their storage is \(O(Tnd)\), giving
\(O(E_M+Tnd+Td)\) total structural storage.
For fixed \(T\) and \(d\), both time and memory are therefore
\(O(n)\) whenever \(E_M=O(n)\).
\end{proof}

\paragraph{Example of the memory difference.}
Consider a graph with \(n=100{,}000\) nodes and \(T=8\)
walk lengths. An explicit dense RRWP tensor contains
\(8\times10^{10}\) scalar entries.
Using 32-bit floating-point values, this tensor alone requires
approximately \(320\) GB, before accounting for features,
gradients, or other model states.
SRA stores the original sparse transition matrix and propagates
the \(n\times d\) feature-response matrix without constructing
this pairwise tensor.

\section{EXPERIMENTAL SETUP AND REPRODUCIBILITY}
\label{sec:app_model_exp_details}

\subsection{Hardware Settings}
\label{subsec:hardware_setting}

All experiments were conducted on a Linux server equipped with an Intel
Xeon CPU and a single NVIDIA GeForce RTX 4090 GPU with 24\,GB of VRAM.
Unless stated otherwise, an out-of-memory result (OOM) denotes that the
evaluated implementation could not complete under this GPU-memory budget.

\subsection{Datasets and Evaluation Splits}
\label{subsec:app_datasets_splits}
\label{app:datasets}
\label{app:real_world}

We evaluate \HARMONIA on twelve real-world datasets. Cora and CiteSeer are citation networks~\citep{sen2008collective}; Tolokers is a crowdsourcing collaboration network~\citep{platonov2023critical}; ogbn-arxiv and ogbn-products are Open Graph Benchmark datasets~\citep{hu2020open}; and IGB-Small is from the Illinois Graph Benchmark~\citep{khatua2023igb}. For graph classification, MUTAG, PROTEINS, DD, AIDS, and Mutagenicity are obtained from TUDataset~\citep{morris2020tudataset}, while hERG-Karim is a molecular hERG-blockade dataset~\citep{karim2021cardiotox,huang2021therapeutics}. Dataset statistics for the node- and graph-classification benchmarks are reported in Tables~\ref{tab:node-dataset-stats} and~\ref{tab:graph-dataset-stats}, respectively.

For node classification, we report the mean and standard deviation over five independent random seeds. For graph classification, we use 10-fold cross-validation repeated with three random seeds, and aggregate results over all folds and repetitions. Synthetic recovery experiments are conducted over ten random seeds, while the seed-stability analysis in Fig.~\ref{fig:seed_stability} reports results from five independent seeds.

\begin{table}[t]
\centering
\footnotesize
\setlength{\tabcolsep}{3.5pt}
\renewcommand{\arraystretch}{1.05}
\caption{Statistics of node-classification datasets.}
\label{tab:node-dataset-stats}
\begin{tabular}{lrrrr}
\toprule
Dataset & \# Nodes & \# Edges & \# Features & \# Classes \\
\midrule
Cora & 2,708     & 10,556     & 1,433 & 7  \\
CiteSeer & 3,327     & 9,104      & 3,703 & 6  \\
Tolokers  & 11,758    & 519,000    & 10    & 2  \\
ogbn-arxiv  & 169,343   & 1,166,243  & 128   & 40 \\
IGB-Small    & 1,000,000 & 12,070,502 & 1,024 & 19 \\
ogbn-products & 2,449,029 & 61,859,140 & 100   & 47 \\
\bottomrule
\end{tabular}
\end{table}

\begin{table}[t]
\centering
\scriptsize
\setlength{\tabcolsep}{2.5pt}
\renewcommand{\arraystretch}{0.95}
\caption{Statistics of graph-classification datasets.}
\label{tab:graph-dataset-stats}

\resizebox{0.7\linewidth}{!}{%
\begin{tabular}{lrrrrr}
\toprule
Dataset & \# Graphs & Avg. Nodes & Avg. Edges & \# Features & \# Classes \\
\midrule
MUTAG          & 188    & 17.93  & 19.79  & 7  & 2 \\
PROTEINS       & 1,113  & 39.06  & 72.82  & 3  & 2 \\
DD             & 1,178  & 284.32 & 715.66 & 89 & 2 \\
AIDS           & 2,000  & 15.69  & 16.20  & 23 & 2 \\
Mutagenicity   & 4,337  & 30.32  & 30.77  & 14 & 2 \\
hERG-Karim     & 13,445 & 30.42  & 32.56  & 65 & 2 \\
\bottomrule
\end{tabular}%
}
\end{table}

\subsection{Baseline Implementations}
\label{subsec:app_baselines}

The comparison contains three groups of methods. The black-box graph models
are GCN~\citep{kipf2016semi}, GAT~\citep{velickovic2018graph},
GraphSAGE~\citep{hamilton2017inductive}, and Graph
Transformer~\citep{dwivedi2021generalization}. The non-graph additive baselines
are NAM~\citep{agarwal2021neural} and
GP-NAM~\citep{zhang2024gaussian}. The graph-additive baselines are
GNAN~\citep{bechler2024intelligible} and
G-NAMRFF~\citep{reddy2025interpretable}.

\subsection{Hyperparameter Settings}
\label{subsec:app_hyperparameters}

Tables~\ref{tab:harmonia_arch_hparams} and
\ref{tab:harmonia_optim_hparams} are dataset-level templates for the final
selected \HARMONIA configurations. No numerical value is inserted unless it is
stated in the current manuscript or supplied experimental record. This avoids
presenting plausible default values as settings of the reported runs.

\begin{table*}[t]
\centering
\scriptsize
\setlength{\tabcolsep}{3.5pt}
\resizebox{\textwidth}{!}{%
\begin{tabular}{lrrrrrlll}
\toprule
Dataset & $C$ & $B$ & $q$ & $m$ & $T$ & Expert depth / widths & Activation & Dropout / norm. \\
\midrule
Cora
& 5 & 8 & 32 & 2 & 8
& 2 / [16, 8]
& ReLU hidden; linear output
& 0.4 / none \\
CiteSeer
& 5 & 8 & 32 & 2 & 8
& 2 / [16, 8]
& ReLU hidden; linear output
& 0.5 / none \\
Tolokers
& 3 & 8 & 8 & 1 & 4
& 2 / [16, 8]
& ReLU hidden; linear output
& 0.2 / none \\
ogbn-arxiv
& 4 & 8 & 16 & 2 & 8
& 2 / [16, 8]
& ReLU hidden; linear output
& 0.1 / none \\
IGB-Small
& 6 & 8 & 32 & 2 & 4
& 2 / [16, 8]
& ReLU hidden; linear output
& 0.1 / none \\
ogbn-products
& 4 & 8 & 16 & 2 & 4
& 2 / [16, 8]
& ReLU hidden; linear output
& 0.1 / none \\
\midrule
MUTAG
& 3 & 8 & 8 & 1 & 4
& 2 / [16, 8]
& ReLU hidden; linear output
& 0.2 / none \\
PROTEINS
& 3 & 8 & 8 & 1 & 8
& 2 / [16, 8]
& ReLU hidden; linear output
& 0.2 / none \\
DD
& 4 & 8 & 16 & 2 & 8
& 2 / [16, 8]
& ReLU hidden; linear output
& 0.1 / none \\
AIDS
& 3 & 8 & 8 & 1 & 4
& 2 / [16, 8]
& ReLU hidden; linear output
& 0.1 / none \\
Mutagenicity
& 4 & 8 & 16 & 2 & 8
& 2 / [16, 8]
& ReLU hidden; linear output
& 0.2 / none \\
hERG-Karim
& 4 & 8 & 16 & 2 & 8
& 2 / [16, 8]
& ReLU hidden; linear output
& 0.2 / none \\
\bottomrule
\end{tabular}}
\caption{Selected \HARMONIA architecture by dataset. Here $C$ is the number of
experts, $B$ the bases per expert, $q$ the feature-embedding dimension, $m$
the number of active experts, and $T$ the number of RRWP channels including
$t=0$.}
\label{tab:harmonia_arch_hparams}
\end{table*}

\begin{table*}[t]
\centering
\scriptsize
\setlength{\tabcolsep}{3.5pt}
\resizebox{\textwidth}{!}{%
\begin{tabular}{llllllll}
\toprule
Dataset & Optimizer & Learning rate & Weight decay & Batch size & Max. epochs & Patience & Selection metric \\
\midrule
Cora
& AdamW & $5\times10^{-4}$ & $5\times10^{-4}$ & full graph & 800 & 150
& validation accuracy (\texttt{val\_acc}) \\

CiteSeer
& AdamW & $1\times10^{-3}$ & $1\times10^{-4}$ & full graph & 800 & 150
& validation accuracy (\texttt{val\_acc}) \\

Tolokers
& AdamW & $1\times10^{-3}$ & $5\times10^{-5}$ & full graph & 300 & 50
& validation accuracy (\texttt{val\_auc}) \\

ogbn-arxiv
& Adam & $2\times10^{-2}$ & $5\times10^{-5}$ & full graph & 600 & 100
& validation accuracy (\texttt{val\_acc}) \\

IGB-Small
& AdamW & $1\times10^{-3}$ & $5\times10^{-5}$ & 16384 & 500 & 50
& validation accuracy (\texttt{val\_acc}) \\

ogbn-products
& AdamW & $1\times10^{-3}$ & $5\times10^{-5}$ & 3584 & 500 & 50
& validation accuracy (\texttt{val\_acc}) \\

MUTAG
& Adam & $1\times10^{-3}$ & $1\times10^{-5}$ & 32 & 600 & 100
& validation ROC-AUC (\texttt{val\_auc}) \\

PROTEINS
& Adam & $1\times10^{-3}$ & $1\times10^{-5}$ & 32 & 600 & 100
& validation ROC-AUC (\texttt{val\_auc}) \\

DD
& Adam & $1\times10^{-3}$ & $1\times10^{-5}$ & 16 & 800 & 100
& validation ROC-AUC (\texttt{val\_auc}) \\

AIDS
& Adam & $1\times10^{-3}$ & $1\times10^{-5}$ & 256 & 800 & 100
& validation ROC-AUC (\texttt{val\_auc}) \\

Mutagenicity
& Adam & $1\times10^{-3}$ & $1\times10^{-5}$ & 32 & 600 & 100
& validation ROC-AUC (\texttt{val\_auc}) \\

hERG
& Adam & $1\times10^{-3}$ & $1\times10^{-5}$ & 32 & 600 & 100
& validation ROC-AUC (\texttt{val\_auc}) \\
\bottomrule
\end{tabular}}
\caption{Selected \HARMONIA optimization settings by dataset. For full-batch
training, report the batch-size entry as ``full graph.''}
\label{tab:harmonia_optim_hparams}
\end{table*}

\subsection{Mixture-of-Neural-Basis Routing}
\label{subsec:app_monb_routing}

\HARMONIA uses $C$ neural-basis experts, each of which maps a scalar input to $B$ basis values,

\begin{equation}
\boldsymbol\psi_{\varphi_c}(x)
=\left[\psi^{(c)}_1(x),\ldots,\psi^{(c)}_B(x)\right]^{\top}
\in\mathbb R^B,
\qquad c=1,\ldots,C.
\label{eq:app_expert_basis}
\end{equation}

Each feature $k$ has a learned global embedding
$\mathbf e_k\in\mathbb R^{d_e}$. Following noisy top-$m$ routing, its expert-score
vector is

\begin{equation}
\mathbf Q_k
=\mathbf W_g^{\top}\mathbf e_k
+\boldsymbol\epsilon_k\odot
  \sigma\!\left(\mathbf W_n^{\top}\mathbf e_k\right),
\qquad
\boldsymbol\epsilon_k\sim\mathcal N(\mathbf 0,\mathbf I_C),
\label{eq:app_noisy_routing_scores}
\end{equation}

where $\mathbf W_g,\mathbf W_n\in\mathbb R^{{d_e}\times C}$ are learned
router parameters. Let
$\mathcal I_k=\operatorname{TopM}(\mathbf Q_k,m)$ denote the indices of the
$m$ largest scores. \HARMONIA uses masked sigmoid gates,

\begin{equation}
\tilde{\pi}_{k,c}
=\mathbf 1\!\left\{c\in\mathcal I_k\right\}\sigma(Q_{k,c}),
\qquad c=1,\ldots,C,
\label{eq:app_masked_sigmoid_gate}
\end{equation}

rather than a softmax over the selected experts. Consequently, the active
weights are independently bounded in $(0,1)$ but are not constrained to sum
to one. The response of feature $k$ is

\begin{equation}
f_k(x)
=\mathbf a_k^{\top}
  \sum_{c\in\mathcal I_k}\tilde{\pi}_{k,c}\boldsymbol\psi_{\varphi_c}(x),
\qquad
\mathbf a_k\in\mathbb R^B.
\label{eq:app_monb_response}
\end{equation}

Thus, each feature estimates only $B$ basis coefficients, even though the complete expert dictionary contains up to $CB$ basis functions. Because the router depends on feature identity rather than on a node's full feature vector, $f_k$ remains a globally inspectable univariate function. Under the stated bias-free parameterization, if each expert has $P_{\psi}$ parameters, the feature module contains $CP_{\psi}+d(B+{d_e})+2{d_e}C$ parameters.

\subsection{Prediction Heads and Graph Readout}
\label{subsec:app_prediction_readout}

Let $\mathbf h_i\in\mathbb R^d$ be the feature-wise representation produced by MoNB and SRA for node $v_i$. For class $c$, the node-level logit is

\begin{equation}
\ell_{i,c}
=\beta_c+\sum_{k=1}^{d}w_{k,c}[\mathbf h_i]_k.
\label{eq:app_node_prediction_head}
\end{equation}

The contribution of feature $k$ to this logit is therefore exactly $w_{k,c}[\mathbf h_i]_k$. For graph classification, \HARMONIA uses sum readout over the node-wise additive representations,

\begin{equation}
\ell_{G,c}
=\beta_c+
\sum_{i\in\mathcal V_G}\sum_{k=1}^{d}w_{k,c}[\mathbf h_i]_k.
\label{eq:app_graph_prediction_head}
\end{equation}

This preserves the same decomposition at graph level: summing all node--feature terms reconstructs the graph logit up to the intercept. For a multi-label task, independent class probabilities are $\widehat y_{i,c}=\sigma(\ell_{i,c})$ or $\widehat y_{G,c}=\sigma(\ell_{G,c})$. The output loss and decision rule used for each reported benchmark must be recorded explicitly because the current manuscript specifies the logits but does not fully document the task-specific training loss.

\subsection{Complete Algorithm and Training Procedure}
\label{subsec:app_algorithm_training}

Algorithm~\ref{alg:harmonia_training} summarizes the
\HARMONIA forward pass with linear RRWP aggregation.
Here, $\sigma$ denotes the sigmoid function, and
$g_{\mathrm{out}}$ is the task-specific output link:
softmax for single-label multiclass classification,
or sigmoid for single-logit binary and multi-label classification.
The product
$\mathbf Z\odot(\mathbf 1_n\boldsymbol\theta_t^\top)$
scales each feature channel by its hop coefficient and is
implemented by broadcasting.
\begin{algorithm}[tb]
\caption{\HARMONIA Forward Pass}
\label{alg:harmonia_training}
\textbf{Input}:
$\mathbf X\in\mathbb R^{n\times d}$,
$\mathbf A,\mathbf D$,
$T,C,B,m,L$;
training flag, prediction level, and output link $g_{\mathrm{out}}$\\
\textbf{Parameters}:
$\{\varphi_c\}_{c=1}^{C}$,
$\{\mathbf a_k\in\mathbb R^B,\,
   \mathbf e_k\in\mathbb R^{d_e}\}_{k=1}^{d}$,
$\mathbf W_g,\mathbf W_n\in\mathbb R^{d_e\times C}$,
$\{\alpha_{t,k}\}$,
$\mathbf W_o\in\mathbb R^{d\times L}$,
$\boldsymbol\beta\in\mathbb R^L$\\
\textbf{Conditions}:
$1\le m\le C$, $T\ge1$,
$D_{ii}=\sum_j A_{ij}>0$, and
$\sum_{s=0}^{T-1}\alpha_{s,k}^{2}>0$ for every $k$\\
\textbf{Output}:
Node predictions $\{\hat{\mathbf y}_i\}_{i=1}^{n}$
or graph prediction $\hat{\mathbf y}_G$

\begin{algorithmic}[1]

\STATE Compute and store sparsely
$\mathbf M\leftarrow\mathbf D^{-1}\mathbf A$

\FOR{$k=1,\ldots,d$}

    \IF{training}
        \STATE Sample
        $\boldsymbol\epsilon_k
        \sim\mathcal N(\mathbf 0,\mathbf I_C)$
    \ELSE
        \STATE Set
        $\boldsymbol\epsilon_k\leftarrow\mathbf 0$
    \ENDIF

    \STATE Compute routing scores
    $\mathbf Q_k
    \leftarrow
    \mathbf W_g^\top\mathbf e_k
    +
    \boldsymbol\epsilon_k
    \odot\sigma(\mathbf W_n^\top\mathbf e_k)$

    \STATE Select experts
    $\mathcal I_k
    \leftarrow\operatorname{TopM}(\mathbf Q_k,m)$

    \STATE Set
    $\tilde{\pi}_{k,c}
    \leftarrow
    \mathbf 1\{c\in\mathcal I_k\}\sigma(Q_{k,c})$
    for $c=1,\ldots,C$

    \FOR{$i=1,\ldots,n$}
        \STATE Compute feature response
        $F_{i,k}
        \leftarrow
        \mathbf a_k^\top
        \left(
        \sum_{c\in\mathcal I_k}
        \tilde{\pi}_{k,c}
        \boldsymbol\psi_{\varphi_c}(x_{i,k})
        \right)$
    \ENDFOR

    \STATE Normalize hop coefficients
    $\displaystyle
    \theta_{t,k}
    \leftarrow
    \frac{\alpha_{t,k}^{2}}
    {\sum_{s=0}^{T-1}\alpha_{s,k}^{2}}$,
    for $t=0,\ldots,T-1$

\ENDFOR

\STATE Initialize
$\mathbf H\leftarrow\mathbf 0_{n\times d}$
and $\mathbf Z\leftarrow\mathbf F$

\FOR{$t=0,\ldots,T-1$}

    \STATE Set
    $\boldsymbol\theta_t
    \leftarrow
    (\theta_{t,1},\ldots,\theta_{t,d})^\top$

    \STATE Accumulate
    $\mathbf H
    \leftarrow
    \mathbf H
    +
    \mathbf Z\odot
    (\mathbf 1_n\boldsymbol\theta_t^\top)$

    \IF{$t<T-1$}
        \STATE Propagate
        $\mathbf Z\leftarrow\mathbf M\mathbf Z$
        using sparse matrix multiplication
    \ENDIF

\ENDFOR

\STATE Set
$\mathbf h_i\leftarrow(\mathbf H_{i,:})^\top$
for $i=1,\ldots,n$

\IF{node-level prediction}

    \STATE Compute
    $\boldsymbol\ell_i
    \leftarrow
    \boldsymbol\beta+\mathbf W_o^\top\mathbf h_i$
    for $i=1,\ldots,n$

    \STATE Set
    $\hat{\mathbf y}_i
    \leftarrow g_{\mathrm{out}}(\boldsymbol\ell_i)$
    for $i=1,\ldots,n$

    \STATE \textbf{return}
    $\{\hat{\mathbf y}_i\}_{i=1}^{n}$

\ELSE

    \STATE Set
    $\mathbf h_G
    \leftarrow\sum_{i=1}^{n}\mathbf h_i$

    \STATE Compute
    $\boldsymbol\ell_G
    \leftarrow
    \boldsymbol\beta+\mathbf W_o^\top\mathbf h_G$

    \STATE Set
    $\hat{\mathbf y}_G
    \leftarrow g_{\mathrm{out}}(\boldsymbol\ell_G)$

    \STATE \textbf{return} $\hat{\mathbf y}_G$

\ENDIF

\end{algorithmic}
\end{algorithm}

Here $[\boldsymbol\theta_t]_k=\theta_{t,k}$ and $\odot$ denotes
element-wise multiplication with row-wise broadcasting. The normalization of
$\boldsymbol\theta_k$ requires
$\sum_{s=0}^{T-1}\alpha_{s,k}^{2}>0$; the initialization or numerical
safeguard used to maintain this condition must be reported with the
implementation. The complete optimization and model-selection settings are
listed as placeholders in Tables~\ref{tab:harmonia_arch_hparams}
and~\ref{tab:harmonia_optim_hparams}.

\subsection{Compute and Reproducibility Details}
\label{subsec:app_compute_reproducibility}

The reported real-world protocol uses five independent seeds for node classification and 10-fold cross-validation repeated over three seeds for graph classification. The additional synthetic explanation-recovery results are reported over ten random seeds, whereas the seed-stability visualization contains five displayed seeds. Mean and standard deviation should be computed over the evaluation units stated for each experiment; folds, repetitions, synthetic scenarios, informative features, and explained nodes must not be interchanged as independent runs.

\paragraph{Scope of the scalability evidence.} The large-graph results show that the evaluated \HARMONIA implementation can complete training and prediction within the reported hardware budget. They do not, by themselves, provide a controlled measurement of peak memory, training time, or the separate gains attributable to MoNB and SRA. The structural computation nevertheless has an explicit complexity characterization. For $N$ nodes and $d$ feature channels, SRA recursively evaluates $\mathbf Z^{(t)}=\mathbf M\mathbf Z^{(t-1)}$ and accumulates the weighted channels. This costs $O\!\left(T\,\operatorname{nnz}(\mathbf M)d+TNd\right)$ arithmetic operations. The forward computation can be streamed with $O\!\left(\operatorname{nnz}(\mathbf M)+Nd+Td\right)$ structural working memory, excluding feature-network activations and optimizer state. Ordinary reverse-mode training may retain $O(TNd)$ propagated activations; therefore, the streaming forward bound is not a bound on total training memory. Neither case requires the explicit $O(TN^2)$ RRWP tensor.

\subsection{Graph-Classification Results}
\label{subsec:graph_classification}

\begin{table*}[b]
\centering
\small
\setlength{\tabcolsep}{5pt}
\resizebox{\textwidth}{!}{%
\begin{tabular}{lrrrrrr}
\toprule
Method & MUTAG & PROTEINS & DD & AIDS & Mutagenicity & hERG \\
\midrule
GCN & $68.07 \pm 6.34$ & $\mathbf{70.97 \pm 4.67}$ & $75.66 \pm 2.36$
    & $\mathbf{99.25 \pm 0.63}$ & $\mathbf{75.69 \pm 0.94}$ & $73.06 \pm 2.08$ \\
GAT & $67.20 \pm 3.45$ & $69.92 \pm 4.03$ & $\mathbf{77.30 \pm 3.68}$
    & $99.00 \pm 0.75$ & $69.40 \pm 1.26$ & $65.04 \pm 1.07$ \\
GraphSAGE & $64.12 \pm 2.46$ & $67.35 \pm 2.38$ & $75.78 \pm 3.98$
    & $98.20 \pm 1.05$ & $69.25 \pm 3.94$ & $71.08 \pm 2.05$ \\
Graph Transformer & $\mathbf{73.30 \pm 5.36}$ & $69.76 \pm 3.27$ & $76.24 \pm 3.26$
    & $99.10 \pm 0.60$ & $73.10 \pm 0.93$ & $\mathbf{74.58 \pm 1.84}$ \\
\midrule
NAM & $63.12 \pm 9.15$ & $62.45 \pm 4.28$ & $58.07 \pm 3.56$
    & $72.51 \pm 2.53$ & $67.35 \pm 2.57$ & $64.57 \pm 2.54$ \\
GP-NAM & $64.30 \pm 8.46$ & $65.68 \pm 4.13$ & $62.07 \pm 3.24$
    & $79.35 \pm 2.28$ & $65.46 \pm 2.25$ & $68.06 \pm 2.37$ \\
GNAN & $67.35 \pm 3.92$ & $59.64 \pm 2.45$ & $60.57 \pm 5.26$
    & $74.87 \pm 1.43$ & $66.64 \pm 4.76$ & $66.12 \pm 2.24$ \\
G-NAMRFF & $66.81 \pm 5.38$ & $67.94 \pm 3.74$ & $63.91 \pm 3.47$
    & $81.10 \pm 1.76$ & $71.70 \pm 2.03$ & $68.41 \pm 1.86$ \\
\midrule
\HARMONIA & $\underline{69.50 \pm 1.54}$ & $\underline{70.76 \pm 1.27}$
    & $\underline{75.38 \pm 1.85}$ & $\underline{97.75 \pm 1.16}$
    & $\underline{74.70 \pm 1.04}$ & $\underline{71.78 \pm 1.57}$ \\
\bottomrule
\end{tabular}}
\caption{Graph-classification accuracy (\%).
Bold denotes the highest mean among all compared methods; underlining
denotes the highest mean among interpretable methods, including \HARMONIA.
The markings indicate mean rankings, not statistical significance.}
\label{tab:graph-classification}
\end{table*}

Table~\ref{tab:graph-classification} reports the complete graph-classification results. \HARMONIA has the highest reported mean among the evaluated interpretable methods on all six datasets. The comparison with black-box models is more mixed: \HARMONIA is close to the strongest reported result on PROTEINS and Mutagenicity, while larger gaps remain on several other datasets. These comparisons concern mean predictive accuracy and do not constitute tests of statistical significance.

\subsection{Synthetic Data Construction and Evaluation Protocol}
\label{app:synthetic_setup}

\paragraph{Graphs, features, and splits.} The synthetic benchmark is a binary node-classification task with known feature responses and feature-specific structural profiles. Each of five independently generated data seeds specifies a 64-node undirected graph with 78 or 79 edges and 384 feature/label draws on that graph. The draws are divided into 256 training, 64 validation, and 64 test instances. Thus, instances within one data seed share the topology and differ in their features and sampled labels; topology varies across data seeds. Each node has 20 continuous raw features in $[-1,1]$. Four columns are informative and the remaining 16 have zero contribution to the generator. The informative column locations are permuted separately for each data seed.

\paragraph{Feature responses and structural profiles.}
The four response templates are

\begin{equation}
\phi_1(x)=2x,\qquad
\phi_2(x)=1.5(x^2-1),\qquad
\phi_3(x)=\sin(\pi x),\qquad
\phi_4(x)=\tanh(2x).
\label{eq:app_function_library}
\end{equation}

We denote the generator's corresponding centered responses by $f_k^{\star}$. An additive constant in a response is absorbed into the intercept in Eq.~\eqref{eq:app_synthetic_contributions}; the templates specify the functional forms and amplitudes. Each data seed contains all four structural profiles in Table~\ref{tab:synthetic_hop_profiles} and assigns them to the four response functions using a predeclared permutation. Consequently, structural recovery is evaluated feature by feature: the generator has no single shared dominant hop.

\begin{table}[t]
\centering
\small
\begin{tabular}{lccccc}
\toprule
Profile & $\theta_{0,k}^{\star}$ & $\theta_{1,k}^{\star}$
& $\theta_{2,k}^{\star}$ & $\theta_{3,k}^{\star}$ & Dominant hop\\
\midrule
A & $0.70$ & $0.20$ & $0.08$ & $0.02$ & $0$\\
B & $0.05$ & $0.15$ & $0.65$ & $0.15$ & $2$\\
C & $0.05$ & $0.10$ & $0.20$ & $0.65$ & $3$\\
D & $0.10$ & $0.55$ & $0.25$ & $0.10$ & $1$\\
\bottomrule
\end{tabular}
\caption{Ground-truth feature-specific structural profiles. Every profile
is nonnegative and sums to one. The four profiles occur in every data seed.}
\label{tab:synthetic_hop_profiles}
\end{table}

Let $\mathbf M$ be the graph's row-normalized adjacency. The exact
contribution of feature $k$ to instance $g$ is

\begin{equation}
\mathbf c_{g,k}^{\star}
=\sum_{t=0}^{3}\theta_{t,k}^{\star}
\mathbf M^t f_k^{\star}(\mathbf X_{g,:,k}),
\qquad
\boldsymbol\eta_g^{\star}
=\beta^{\star}\mathbf 1+\sum_{k\in\mathcal S}\mathbf c_{g,k}^{\star}.
\label{eq:app_synthetic_contributions}
\end{equation}

Labels are sampled from the Bernoulli model in Eq.~\eqref{eq:app_synthetic_contributions}. The functions, informative columns, and hop profiles define evaluation ground truth and are hidden during fitting.

\paragraph{Matched targets and repetitions.} Each data seed is paired with three model seeds, giving 15 runs per method. We explain nodes $0$ and $32$ in each of the 64 test instances, giving 128 deterministic targets per run and 1,920 target evaluations across all runs. Feature-explainer comparisons use these matched targets. The edge experiment reuses the cohort's fixed training hyperparameters and validation-selected checkpoints, with no additional hyperparameter search. Post-hoc GCN explainers operate on the common frozen GCN checkpoints.

The feature--hop table reports means over the 15 runs. For edge metrics, we first average over the 128 targets in each run, then average the three model seeds within a data seed. If $m_{d,s,q}$ is a target-level metric, the reported edge summary is

\begin{equation}
\bar m_d=\frac{1}{3}\sum_{s=1}^{3}
\frac{1}{128}\sum_{q=1}^{128}m_{d,s,q},\qquad
\bar m=\frac{1}{5}\sum_{d=1}^{5}\bar m_d,\qquad
s_m=\left[\frac{1}{4}\sum_{d=1}^{5}(\bar m_d-\bar m)^2\right]^{1/2}.
\label{eq:app_synthetic_aggregation}
\end{equation}

Edge entries are $\bar m\pm s_m$. The standard deviation therefore describes variation across five data-seed means; targets and model seeds are not counted as additional independent data replicates.

\subsection{Feature--Hop Recovery: Explanations, Metrics, and Results}
\label{app:recovery_metrics}

\paragraph{Feature explanations.}
For additive predictors, feature importance is based on their explicit
contributions to the scalar prediction score, including the final readout.
For \HARMONIA, the feature contribution is

\begin{equation}
\widehat{\mathbf c}_{g,k}
=\widehat w_k\sum_{t=0}^{3}\widehat\theta_{t,k}
\mathbf M^t\widehat f_k(\mathbf X_{g,:,k}).
\label{eq:app_synthetic_learned_contribution}
\end{equation}

For a binary predictor implemented with two logits, the scalar score is their difference, and the corresponding feature contributions include the difference of the class-specific readout weights. GNAN and G-NAMRFF use their own additive decompositions. Ground-truth contribution importance is obtained from Eq.~\eqref{eq:app_synthetic_contributions}. GNNExplainer~\citep{ying2019gnnexplainer} supplies feature-mask scores averaged over the matched targets.

The additional feature baseline is GCN + UCExplainer~\cite{giorgi2026unified}. It uses the existing GCN checkpoints without retraining the classifier or performing a new hyperparameter search. For target $q$, its score for raw feature $k$ is

\begin{equation}
s_{q,k}^{\mathrm{UC}}
=\frac{1}{64}\sum_{v=1}^{64}
\left|X^{\mathrm{cf},q}_{v,k}-X^{q}_{v,k}\right|,
\qquad
\widehat I_k^{\mathrm{UC}}
=\frac{1}{128}\sum_{q=1}^{128}s_{q,k}^{\mathrm{UC}}.
\label{eq:app_uc_feature_score}
\end{equation}

The adapter freezes the edge parameters, optimizes continuous feature perturbations with Adam at learning rate $0.1$ for 500 epochs per target, and clips perturbed values to $[-1,1]$. Its implementation uses \texttt{beta=1} to enable the feature path and \texttt{scheduler.initial\_alpha=0} to freeze the edge path. This differs from the upstream \texttt{beta=0} setting, which disables feature perturbation and cannot yield the raw-feature ranking evaluated here.

\paragraph{Feature-ranking metrics.}
Let $\widehat I_k$ denote a method's run-level feature score and let
$I_k^{\star}$ denote the ground-truth contribution importance. With
$\mathcal S$ the four informative columns,

\begin{equation}
\operatorname{Precision@4}
=\frac{|\operatorname{Top}_4(\widehat{\mathbf I})\cap\mathcal S|}{4}.
\label{eq:app_synthetic_precision}
\end{equation}

NDCG@4 measures the discounted ranking quality of these four positions, using continuous ground-truth contribution importance as relevance and normalizing by the ideal ordering. Feature Spearman measures rank correlation between predicted and ground-truth importance over all 20 raw features. The metrics assess different aspects of recovery: Precision@4 measures support identification, NDCG@4 also assesses ordering among the selected features, and Spearman evaluates the full ranking. The 16 null features have tied ground-truth relevance; consequently, Feature Spearman need not equal one when the informative top four are recovered exactly.

\paragraph{Effective response recovery.}
Effective NRMSE summarizes the normalized error of an explicit learned effective response curve relative to its ground-truth counterpart. The effective response includes the predictor's output scaling. This metric evaluates response-curve recovery; node-wise propagated contribution errors and edge-intervention effects are different evaluation quantities. It is unavailable for the two post-hoc GCN feature explainers because their masks or perturbation scores do not define additive response functions.

\paragraph{Common structural probe.}
Hop L1, top-hop accuracy, and probe coverage are obtained from a common signed structural probe of the frozen predictive model. Hop L1 measures discrepancy from the feature-specific ground-truth hop profile, and top-hop accuracy measures agreement with its dominant order. Probe coverage records the proportion of probes yielding a usable structural estimate and accompanies the recovery scores to show the extent of the evaluation. These quantities assess the probe's recovered structural behavior; they are not obtained by comparing GNAN distance weights, G-NAMRFF Laplacian coefficients, and \HARMONIA random-walk coefficients as though they were interchangeable parameters. In particular, both GCN feature explainers have identical structural results because the probe acts on their shared GCN predictor independently of the attached feature explanation.

\begin{table*}[t]
\centering
\small
\setlength{\tabcolsep}{5pt}

\resizebox{\textwidth}{!}{%
\begin{tabular}{lrrrrrrr}
\toprule
Method
& \shortstack[c]{Precision\\@4 $\uparrow$}
& \shortstack[c]{NDCG\\@4 $\uparrow$}
& \shortstack[c]{Feature\\Spearman $\uparrow$}
& \shortstack[c]{Effective\\NRMSE $\downarrow$}
& \shortstack[c]{Hop\\L1 $\downarrow$}
& \shortstack[c]{Top-hop\\Acc. $\uparrow$}
& \shortstack[c]{Probe\\Coverage $\uparrow$} \\
\midrule

\HARMONIA
& $\mathbf{1.000}$
& $\mathbf{1.000}$
& $\mathbf{0.699}$
& $\mathbf{0.193}$
& $\mathbf{0.468}$
& $\mathbf{0.950}$
& $\mathbf{1.000}$ \\

GNAN
& $\mathbf{1.000}$
& $0.996$
& $0.698$
& $0.435$
& $0.822$
& $0.250$
& $0.983$ \\

G-NAMRFF
& $0.783$
& $0.897$
& $0.568$
& $0.735$
& $1.129$
& $0.289$
& $0.640$ \\

GCN + GNNExplainer
& $0.850$
& $0.934$
& $0.619$
& N/A
& $0.766$
& $0.493$
& $0.547$ \\

GCN + GOAt
& $0.950$
& $0.967$
& $0.669$
& N/A
& $0.766$
& $0.493$
& $0.547$ \\

GCN + UCExplainer
& $0.767$
& $0.912$
& $0.586$
& N/A
& $0.766$
& $0.493$
& $0.547$ \\

\bottomrule
\end{tabular}%
}

\caption{
Complete feature--hop recovery results, averaged over five data seeds and three model seeds per data seed. Feature explainers use the same 128 test targets per run. N/A indicates that the explainer does not define an explicit additive response curve. The GCN rows share structural-probe results by construction. Bold denotes the best reported mean, including ties; no dispersion estimates are available for this table.
}
\label{tab:synthetic_feature_full}
\end{table*}

\paragraph{Results.}
\HARMONIA and GNAN both recover the informative feature set, with Precision@4 of $1.000$. \HARMONIA additionally attains NDCG@4 of $1.000$ at the reported precision, compared with $0.996$ for GNAN. Their full-feature rank correlations are nearly identical ($0.699$ and $0.698$), whereas their Effective NRMSE differs substantially ($0.193$ and $0.435$). \HARMONIA therefore improves response recovery beyond the support identification already achieved by GNAN. Its top-hop accuracy is $0.950$ with coverage $1.000$, while the other predictors have both lower accuracy and lower coverage. Its nonzero Hop L1 of $0.468$ also shows that correct dominant-hop identification does not imply exact recovery of the full profile. The differing probe coverage limits interpretation of the aggregate hop scores as comparisons on identical sets of valid probes.

For the shared GCN checkpoints, GNNExplainer exceeds the feature-only UCExplainer adapter on all three feature-ranking metrics. UCExplainer changes the predicted label for 1,588 of 1,920 targets, a validity rate of $0.827$. Nevertheless, all 20 raw features change on average at a $10^{-6}$ threshold. Counterfactual validity and feature sparsity are therefore separate properties in this experiment; this adapter's results do not demonstrate sparse feature explanations.

\subsection{Instance-Level Structural Explanations}
\label{app:synthetic_edges}

\paragraph{Candidate edges and budget.}
For each target $q=(g,i)$, let $E_q$ contain the undirected
candidate edges in its three-hop neighborhood. Candidate sets contain
7--68 edges, depending on the data seed and target. Every method receives
the same candidates and selects
\begin{equation}
K_q=\left\lceil0.20|E_q|\right\rceil
\label{eq:app_synthetic_edge_budget}
\end{equation}
edges. The candidate set and budget are fixed before evaluating any
method's ranking.

\paragraph{Exact edge-intervention ground truth.}
For candidate edge $e$, its relevance is
\begin{equation}
I_{q,e}^{\star}
=\left|\eta_{g,i}^{\star}(G)
-\eta_{g,i}^{\star}(G\setminus e)\right|.
\label{eq:app_edge_truth}
\end{equation}
Deleting an undirected edge removes both directed adjacency entries. We
then recompute the row-normalized transition matrix and its powers through
order three and reevaluate Eq.~\eqref{eq:app_synthetic_contributions}. Features and
labels remain fixed. Thus, relevance includes changes in transition
normalization as well as the removal of walks through the edge. It is the
exact single-edge effect under the specified synthetic generator, rather
than a shortest-path label or a sum of edge masks. Effects of jointly
deleting several edges need not equal the sum of their individual effects.

\paragraph{Method-specific edge scores.}
For an additive predictor with scalar logit $\widehat\eta$, we use
\begin{equation}
\widehat I_{q,e}
=\left|\widehat\eta_{g,i}(G)
-\widehat\eta_{g,i}(G\setminus e)\right|,
\label{eq:app_edge_sensitivity}
\end{equation}
with all learned parameters frozen. Each model rebuilds its own structural
operator after deletion. GNAN recomputes shortest-path distances and
distance-shell normalization. G-NAMRFF recomputes powers of the symmetric
normalized Laplacian through order three while retaining its learned FIR
coefficients and random Fourier feature functions. \HARMONIA recomputes
the row-normalized random-walk powers while retaining its learned MoNB
responses and linear topology parameters. These are model edge-sensitivity
scores; they are distinct from the additive source-node and hop
contributions exposed by \HARMONIA.

GNNExplainer, PGExplainer~\citep{luo2020parameterized}, and IDEA score edges
of the same frozen GCN checkpoints. For each undirected edge, the scores
of its two directed copies are averaged before ranking. This protocol
compares realized edge effects or edge explanations against a common
ground truth without identifying the different models' structural
coefficient systems with one another.

\paragraph{Ground-truth recovery metrics.}
Edge Spearman is the rank correlation between predicted scores and
$I_{q,e}^{\star}$ over all candidate edges. Edge NDCG@20\% evaluates the
first $K_q$ positions using the continuous relevance in
Eq.~\eqref{eq:app_edge_truth}, normalized by the ideal ordering. Let
$\widehat S_q$ and $S_q^{\star}$ be the predicted and ground-truth
top-$K_q$ edge sets. Then
\begin{equation}
\operatorname{Precision@20\%}_q
=\frac{|\widehat S_q\cap S_q^{\star}|}{K_q}.
\label{eq:app_edge_precision}
\end{equation}
These three metrics evaluate agreement with the known generator, with
larger values indicating better recovery.

\begin{table*}[t]
\centering
\scriptsize
\setlength{\tabcolsep}{5pt}
\resizebox{\textwidth}{!}{%
\begin{tabular}{lrrr}
\toprule
Method 
& Edge Spearman $\uparrow$ 
& Edge NDCG@20\% $\uparrow$
& Precision@20\% $\uparrow$ \\
\midrule

\HARMONIA 
& $\mathbf{0.7822\pm0.0489}$ 
& $\mathbf{0.9425\pm0.0102}$
& $\mathbf{0.8162\pm0.0402}$ \\

GNAN
& $0.4546\pm0.0581$ 
& $0.6305\pm0.0805$
& $0.5005\pm0.0766$ \\

G-NAMRFF
& $0.6006\pm0.0282$ 
& $0.7589\pm0.0262$
& $0.6430\pm0.0478$ \\

GCN + GNNExplainer
& $0.5959\pm0.0187$ 
& $0.6259\pm0.0527$
& $0.5708\pm0.0530$ \\

GCN + GOAt
& $0.5871\pm0.0371$ 
& $0.7790\pm0.0424$
& $0.6597\pm0.0712$ \\

GCN + UCExplainer
& $0.0194\pm0.0559$ 
& $0.2970\pm0.0517$
& $0.2594\pm0.0294$ \\

\bottomrule
\end{tabular}}
\caption{
Complete instance-level structural recovery results. Entries are mean
$\pm$ sample standard deviation across five data-seed means; each mean
averages three model seeds, each evaluated on the same 128 targets.
All methods share candidate edges, exact edge-deletion ground truth,
and the rounded-up top-20\% budget. Bold denotes the best reported mean.
}
\label{tab:synthetic_edge_full}
\end{table*}

\paragraph{Results and interpretation.}
\HARMONIA achieves the highest mean on all three edge-recovery metrics:
Spearman $0.7822$, NDCG $0.9425$, and precision $0.8162$.
G-NAMRFF is the strongest baseline, with corresponding values of
$0.6006$, $0.7589$, and $0.6430$. GNAN has weaker recovery, while
PGExplainer and IDEA have near-zero rank correlations with the generator
and precision near $0.20$, indicating little alignment in this setting.
These comparisons describe the reported mean rankings, rather than
establishing statistical significance from aggregate summaries.

\paragraph{Scope of the evidence.}
The benchmark tests recovery under a known additive random-walk
generator, which matches \HARMONIA's structural representation. It
supports the reported feature-response and edge-intervention comparisons
within this setting. It does not isolate the contribution of MoNB from
that of structural basis alignment, establish coefficient identifiability
from the observed samples, or establish recovery under other graph
generators. Edge-sensitivity scores also require reevaluating a model
after interventions; this table is not a comparison of explanation
runtime. Predictive accuracy and calibration are separate from the
recovery and fidelity quantities reported here.

\subsection{Source-Node Contribution Recovery}
\label{app:synthetic_nodes}

\paragraph{Objective and matched cohort.}
This experiment evaluates whether an explanation identifies the source
nodes that most strongly influence a target prediction under the known
synthetic generator. A source-node effect is defined by replacing that
node's complete feature vector with the feature-wise training mean and
measuring the resulting change in the target logit. All 64 source nodes,
including the target itself, are candidates, and the ranking cutoff is
$K=4$. This intervention differs from the edge deletion evaluated in
Appendix~\ref{app:synthetic_edges}: the graph remains fixed, while one
node's features change.

The primary comparison uses data-seed indices $0,1,2$ and model-seed
indices $0,1,2$. Every run explains nodes $0$ and $32$ in each of 64 held-out
instances, giving 128 targets per run and 1,152 target evaluations per
method across nine runs. Both post-hoc explainers completed all nine runs.
The direct-intervention methods were originally evaluated on five data
seeds; we restrict their reported results to the same three data seeds as
the explainers. No predictor is retrained for this evaluation: every
method uses its existing checkpoint selected by validation log loss.

\paragraph{Generator and reproducible seed assignments.}
We use the four-response, four-hop construction in
Appendix~\ref{app:synthetic_setup}. Each data seed defines a connected
64-node topology and 384 independent feature/label instances, split into
256 training, 64 validation, and 64 test instances. All 20 raw features are
sampled independently from $\operatorname{Uniform}[-1,1]$; four are signals
and 16 are null. Each connected graph passes the predefined training-design
checks on rank, condition number, and cross-hop correlation. The
row-normalized random-walk powers $\mathbf M^0,\ldots,\mathbf M^3$ are
stored with the dataset.

For the response template $\phi_k$ assigned to informative column $k$,
centering uses only training covariates:
\begin{equation}
f_k^\star(x)=\phi_k(x)
-\frac{1}{N|\mathcal D_{\mathrm{train}}|}
\sum_{g\in\mathcal D_{\mathrm{train}}}\sum_{j=1}^{N}
\phi_k(X_{g,j,k}),\qquad N=64.
\label{eq:app_node_response_centering}
\end{equation}
The templates are $2x$, $1.5(x^2-1)$, $\sin(\pi x)$, and $\tanh(2x)$;
the four hop profiles are those in Table~\ref{tab:synthetic_hop_profiles}.
The generator's intercept is calibrated using training covariates so that
expected training prevalence is $0.5$, and labels follow
$y_{g,i}\sim\operatorname{Bernoulli}(\sigma(\eta_{g,i}^\star))$ with
$\eta_{g,i}^\star$ defined in Eq.~\eqref{eq:app_synthetic_contributions}.
Feature, label, and signal-column-permutation random-number streams are
separate. Dataset SHA-256 hashes are checked against the saved checkpoints
before evaluation.

\begin{table}[t]
\centering
\small
\setlength{\tabcolsep}{4pt}
\begin{tabular}{ccccl}
\toprule
Data index & Feature seed & Graph seed & Signal columns & Profile assignment\\
\midrule
$0$ & $20260916$ & $123$ & $19,17,5,2$ & A, B, C, D\\
$1$ & $20260917$ & $223$ & $18,3,7,1$ & B, C, D, A\\
$2$ & $20260918$ & $323$ & $1,14,18,0$ & C, D, A, B\\
\bottomrule
\end{tabular}
\caption{Source-node evaluation cohort. Signal-column indices are
zero-based. Profile assignments follow the response order
$(2x,\text{quadratic},\sin,\tanh)$. Each data seed is paired with model
seeds $0,1,2$.}
\label{tab:synthetic_node_seeds}
\end{table}

\paragraph{Predictors and original training protocol.}
\HARMONIA uses the direct-linear MoNB model with a basis setting of 64,
four experts, top-2 routing, hidden widths $(64,32)$, and four RRWP orders.
GNAN uses two layers with 16 hidden channels and learned distance weights.
G-NAMRFF uses the RFF setting $M=64$ and four filter powers. The GCN uses
three 32-channel layers and a direct linear path from the target's own
features. These architectures use the same training protocol: binary
cross-entropy with logits, AdamW with learning rate $5\times10^{-4}$ and
zero weight decay, a cosine learning-rate schedule ending at $10^{-6}$,
batch size 16, and gradient-norm clipping at 1. Training runs for at most
400 epochs, with early stopping after 100 validation epochs without
improvement and checkpoint selection by validation log loss.

\paragraph{Exact source-node ground truth.}
The replacement vector is computed from raw training features:
\begin{equation}
\bar x_{\mathrm{train},k}
=\frac{1}{N|\mathcal D_{\mathrm{train}}|}
\sum_{g\in\mathcal D_{\mathrm{train}}}\sum_{v=1}^{N}X_{g,v,k}.
\label{eq:app_node_baseline}
\end{equation}
Let $\mathbf X_g^{(j\leftarrow\bar{\mathbf x})}$ replace all 20 features
of source node $j$ with $\bar{\mathbf x}_{\mathrm{train}}$, leaving every
other node and the topology unchanged. The signed oracle effect is
\begin{align}
\Delta_{g,i\leftarrow j}^{\star}
&=\eta_{g,i}^{\star}(\mathbf X_g)
-\eta_{g,i}^{\star}(\mathbf X_g^{(j\leftarrow\bar{\mathbf x})})
\label{eq:app_node_oracle_delta}\\
&=\sum_{k\in\mathcal S}\sum_{t=0}^{3}
\theta_{t,k}^{\star}(\mathbf M^t)_{i,j}
\left[f_k^\star(X_{g,j,k})
-f_k^\star(\bar x_{\mathrm{train},k})\right].
\nonumber
\end{align}
The second equality follows from the additive generator and the fixed
topology. The effect retains both its direction and logit-scale magnitude.
It is a baseline-dependent intervention effect, rather than a Shapley value
or an arbitrary decomposition of a nonlinear prediction. In particular,
centering $f_k^\star$ over training observations does not generally make
$f_k^\star(\bar x_{\mathrm{train},k})$ zero. The oracle evaluation reuses
the stored walk powers, response functions, hop weights, and intercept;
the reported implementation check has zero residual against the generator.

\paragraph{Model and post-hoc explanation scores.}
For \HARMONIA, GNAN, G-NAMRFF, and GCN, the signed model effect is
\begin{equation}
\widehat\Delta_{g,i\leftarrow j}
=\widehat\eta_{g,i}(\mathbf X_g)
-\widehat\eta_{g,i}(\mathbf X_g^{(j\leftarrow\bar{\mathbf x})}),
\label{eq:app_node_model_delta}
\end{equation}
with model parameters frozen. Every target is evaluated under all 64
source-node interventions. The GCN counterfactual row denotes this direct
model response, separately from either post-hoc explainer.

GNNExplainer and UCExplainer use the same frozen GCN checkpoints.
GNNExplainer is run for 100 optimization epochs with learning rate $0.01$
in model-level binary-node explanation mode. It learns one nonnegative
\texttt{object} mask value per source node, with edge masking disabled.
UCExplainer uses the COMBINEX feature-only counterfactual adapter for 500
epochs with Adam learning rate $0.1$, \texttt{beta=1},
\texttt{initial\_alpha=0}, and \texttt{p\_lambda=0}. It targets the class
opposite to the GCN's original prediction and scores node $j$ by
\begin{equation}
s_{g,i,j}^{\mathrm{UC}}
=\frac{1}{20}\sum_{k=1}^{20}
\left|X_{g,j,k}^{\mathrm{CF}(i)}-X_{g,j,k}\right|.
\label{eq:app_node_uc_score}
\end{equation}
Both post-hoc explanations yield unsigned importance scores, not signed
effects in logit units.

\paragraph{Ranking and signed recovery metrics.}
For one target, write $\Delta_j^\star=\Delta_{g,i\leftarrow j}^\star$
and $a_j=|\Delta_j^\star|$. Direct-intervention methods use
$s_j=|\widehat\Delta_j|$ for ranking; post-hoc methods use their node
importance scores. Let $T_4(\mathbf v)$ denote the four highest-scoring
nodes under $\mathbf v$, and let $\pi_r$ denote the node at predicted
rank $r$. We report
\begin{align}
\operatorname{NodeSpearman}
&=\operatorname{Spearman}(\mathbf s,\mathbf a),
\label{eq:app_node_spearman}\\
\operatorname{Precision@4}
&=\frac{|T_4(\mathbf s)\cap T_4(\mathbf a)|}{4},
\label{eq:app_node_precision}\\
\operatorname{NDCG@4}
&=\frac{\sum_{r=1}^{4}a_{\pi_r}/\log_2(r+1)}
{\operatorname{IDCG@4}}.
\label{eq:app_node_ndcg}
\end{align}
Spearman evaluates the complete ranking of 64 source nodes. Precision@4
measures overlap with the oracle's top four, whereas NDCG@4 uses the
continuous oracle magnitudes as linear gains and normalizes by their
ideal ordering. A uniformly random selection of four out of 64 candidates
has expected Precision@4 of $4/64=0.0625$.

For methods with signed model effects, we additionally report
\begin{align}
\operatorname{NRMSE}_{\mathrm{signed}}
&=\frac{
\left[\frac{1}{64}\sum_{j=1}^{64}
(\widehat\Delta_j-\Delta_j^\star)^2\right]^{1/2}}
{\operatorname{std}_{j}(\Delta_j^\star)+10^{-12}},
\label{eq:app_node_signed_nrmse}\\
\operatorname{SignAgree@4}
&=\frac{1}{4}\sum_{j\in T_4(\mathbf a)}
\mathbf 1\!\left\{
\operatorname{sign}(\widehat\Delta_j)
=\operatorname{sign}(\Delta_j^\star)\right\}.
\label{eq:app_node_sign_agreement}
\end{align}
Here, $\operatorname{std}_j$ is the standard deviation across the 64
oracle source-node effects for that target. Signed NRMSE measures effect
magnitude and sign errors, while sign agreement is evaluated on the
\emph{oracle} top-four nodes. These metrics are unavailable for the
unsigned post-hoc scores. Signed source-node NRMSE evaluates the
intervention vector and is distinct from the effective response-curve
NRMSE in Table~\ref{tab:synthetic_feature_full}.

\paragraph{Aggregation.}
Each metric is computed per target and averaged across the 128 targets
within a run. The three model-seed means are then averaged within each
data seed. If $m_{d,s,q}$ is a target-level metric, the reported mean and
sample standard deviation are
\begin{equation}
\bar m_d=\frac{1}{3}\sum_{s=0}^{2}
\frac{1}{128}\sum_{q=1}^{128}m_{d,s,q},\qquad
\bar m=\frac{1}{3}\sum_{d=0}^{2}\bar m_d,\qquad
s_m=\left[\frac{1}{2}\sum_{d=0}^{2}(\bar m_d-\bar m)^2\right]^{1/2}.
\label{eq:app_node_aggregation}
\end{equation}
Thus, uncertainty is measured across three independently generated
datasets, after averaging model initializations within each dataset.
The 1,152 target evaluations are not treated as independent data replicates
\begin{table*}[t]
\centering
\small
\setlength{\tabcolsep}{3pt}
\resizebox{\textwidth}{!}{%
\begin{tabular}{lccccc}
\toprule
Method 
& Spearman $\uparrow$ 
& Precision@4 $\uparrow$ 
& NDCG@4 $\uparrow$
& \shortstack{Signed\\NRMSE $\downarrow$}
& SignAgree@4 $\uparrow$ \\
\midrule

\HARMONIA
& $\mathbf{0.9601\pm0.0464}$ 
& $\mathbf{0.8487\pm0.0254}$
& $\mathbf{0.9658\pm0.0122}$ 
& $\mathbf{0.2992\pm0.0626}$
& $\mathbf{0.9792\pm0.0066}$ \\

GNAN
& $0.5319\pm0.0730$ 
& $0.7057\pm0.0344$
& $0.8861\pm0.0409$ 
& $0.5853\pm0.1196$
& $0.8487\pm0.0478$ \\

G-NAMRFF
& $0.9309\pm0.0761$ 
& $0.6413\pm0.0412$
& $0.7975\pm0.0334$ 
& $0.8433\pm0.0888$
& $0.7157\pm0.0436$ \\

GCN + GNNExplainer
& $0.8893\pm0.1106$ 
& $0.5312\pm0.0584$
& $0.6914\pm0.0812$ 
& N/A 
& N/A \\

GCN + GOAt
& $0.9363\pm0.0669$ 
& $0.6732\pm0.0384$
& $0.8294\pm0.0507$ 
& N/A 
& N/A \\

GCN + UCExplainer
& $-0.0264\pm0.0285$ 
& $0.1120\pm0.0098$
& $0.2076\pm0.0100$ 
& N/A 
& N/A \\

\bottomrule
\end{tabular}}
\caption{
Source-node contribution recovery.
Entries are mean $\pm$ sample standard deviation across three data-seed
means, each averaging three model seeds and 128 targets per run.
All methods rank the same 64 candidates against the same oracle
intervention effects. N/A denotes unavailable signed logit-scale scores.
Bold marks the best mean in each recovery column.
}
\label{tab:synthetic_node_full}
\end{table*}

\begin{table}[t]
\centering
\small
\begin{tabular}{lc}
\toprule
Explainer & Nodes with nonzero score/change\\
\midrule
GCN + GNNExplainer & $25.83\pm8.89$\\
GCN + UCExplainer & $64.00\pm0.00$\\
\bottomrule
\end{tabular}
\caption{Auxiliary source-node explanation diagnostics on the same
cohort. UCExplainer changes at least one feature at every node under the
$10^{-6}$ threshold. These counts describe the full explanation before
selecting its top-four nodes.}
\label{tab:synthetic_node_diagnostics}
\end{table}

\paragraph{Results and interpretation.}
\HARMONIA has the best reported mean on all five recovery metrics, including
NDCG@4 of $0.9658$, signed NRMSE of $0.2992$, and top-four sign agreement
of $0.9792$. G-NAMRFF has a high full-ranking correlation ($0.9309$), but
lower top-four precision ($0.6413$) and substantially higher signed NRMSE
($0.8433$). Broad ranking agreement therefore does not ensure accurate
recovery of the strongest effects or their signed magnitudes. Conversely,
GNAN's NDCG@4 of $0.8861$ indicates useful recovery near the top of the
ranking despite its lower overall Spearman correlation ($0.5319$).

For the same frozen GCN, GNNExplainer improves Spearman relative to direct
counterfactual effects ($0.8893$ versus $0.8327$), but reduces Precision@4
($0.5312$ versus $0.6771$) and NDCG@4 ($0.6914$ versus $0.8280$).
UCExplainer achieves counterfactual validity of $0.8993$, yet its node
ranking has near-zero Spearman correlation ($-0.0264$), NDCG@4 of
$0.2076$, and changes at all 64 nodes. Successfully flipping the fitted
model's class is therefore insufficient evidence of correct or sparse
source-node attribution.

The conclusions concern training-mean replacement effects under this
controlled additive random-walk generator and its graph-selection
protocol. The reported intervention scores do not establish a unique
attribution under other baselines, and no statistical-significance claim
is made from three independently generated datasets. The model-intervention
rows also require repeated forward evaluations; these results do not
compare explanation runtime.

\subsection{Additional Univariate Response Recovery}
\label{app:synthetic_function_recovery}

\paragraph{Purpose.}
We additionally examine \HARMONIA's learned univariate responses for the
twelve ground-truth functions in Figure~\ref{fig:synthetic_recovery}.
This study complements the feature--hop and edge-recovery experiments in
Appendices~\ref{app:recovery_metrics} and~\ref{app:synthetic_edges} by
examining a broader collection of functional forms, including polynomial,
periodic, saturating, nonsmooth, and localized responses. It is a separate
function-recovery study from the four-informative-feature benchmark.

\paragraph{Curve recovery metric.}
For response $\ell$, let $g_\ell^\star$ be the ground-truth univariate effect and let $\widehat g_\ell=\widehat w_\ell\widehat f_\ell$ be the learned response including its output coefficient. Both are evaluated on the same prediction scale and common grid $\mathcal U=\{u_j\}_{j=1}^{J}$. We center each curve over this grid:

\begin{equation}
\widetilde{\widehat g}_\ell(u_j)
=\widehat g_\ell(u_j)-\frac{1}{J}\sum_{a=1}^{J}\widehat g_\ell(u_a),
\qquad
\widetilde g_\ell^\star(u_j)
=g_\ell^\star(u_j)-\frac{1}{J}\sum_{a=1}^{J}g_\ell^\star(u_a).
\label{eq:app_curve_centering}
\end{equation}

For a nonconstant ground-truth response, the curve error is

\begin{equation}
\operatorname{CurveNRMSE}_\ell
=\frac{
\left\{\sum_{j=1}^{J}
\left[\widetilde{\widehat g}_\ell(u_j)
-\widetilde g_\ell^\star(u_j)\right]^2\right\}^{1/2}}
{\left\{\sum_{j=1}^{J}
\left[\widetilde g_\ell^\star(u_j)\right]^2\right\}^{1/2}}.
\label{eq:app_curve_nrmse}
\end{equation}

Centering removes constant offsets while retaining amplitude and sign
differences. No fitted sign reversal or multiplicative alignment is
included in this definition. The normalization uses the ground-truth
curve on the evaluation grid. These curve-level summaries are reported
separately from the aggregate Effective NRMSE of the feature--hop
benchmark in Table~\ref{tab:synthetic_feature_full}.

\paragraph{Results.}
Figure~\ref{fig:synthetic_recovery} shows that recovery error varies
across response functions. The reported mean curve errors are small for
the linear ($0.023$), sigmoid ($0.033$), hyperbolic tangent ($0.034$),
and arctangent ($0.034$) responses. The sine response has a larger error
($0.364$), showing that accurate recovery of the simpler responses does
not extend uniformly to every functional form in this study.

\begin{figure*}[t]
\centering

\IfFileExists{images/synthetic_recovery/synthetic_1.png}{%
\begin{subfigure}[t]{0.235\textwidth}\centering
\includegraphics[width=\linewidth]{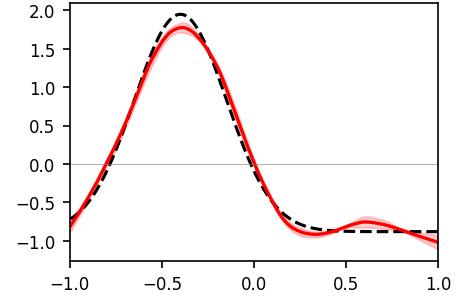}
\caption{Localized response: ${0.103}\pm{0.032}$}
\label{fig:syn_exp_local_neg}
\end{subfigure}\hfill
\begin{subfigure}[t]{0.235\textwidth}\centering
\includegraphics[width=\linewidth]{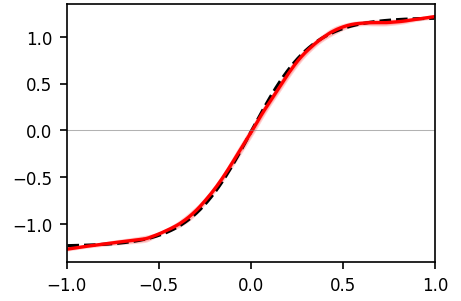}
\caption{Sigmoid: ${0.033}\pm{0.011}$}
\label{fig:syn_sigmoid_6x_center}
\end{subfigure}\hfill
\begin{subfigure}[t]{0.235\textwidth}\centering
\includegraphics[width=\linewidth]{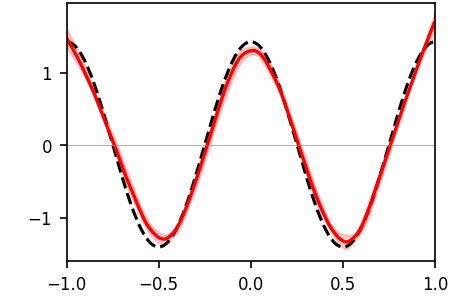}
\caption{Cosine: ${0.136}\pm{0.037}$}
\label{fig:syn_cos_2pi}
\end{subfigure}\hfill
\begin{subfigure}[t]{0.235\textwidth}\centering
\includegraphics[width=\linewidth]{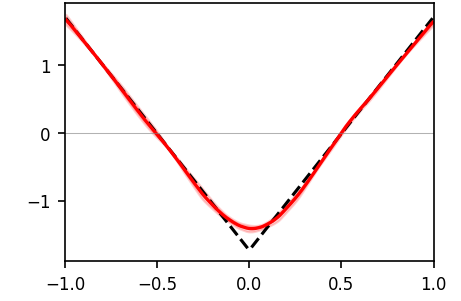}
\caption{Absolute value: ${0.084}\pm{0.013}$}
\label{fig:syn_abs}
\end{subfigure}
\par\medskip
\begin{subfigure}[t]{0.235\textwidth}\centering
\includegraphics[width=\linewidth]{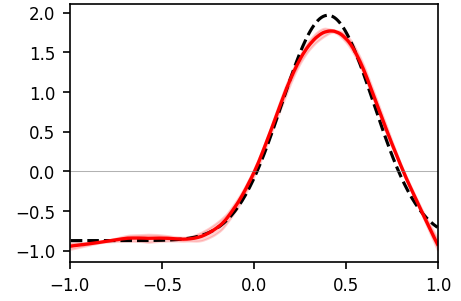}
\caption{Localized response: ${0.094}\pm{0.016}$}
\label{fig:syn_exp_local_pos}
\end{subfigure}\hfill
\begin{subfigure}[t]{0.235\textwidth}\centering
\includegraphics[width=\linewidth]{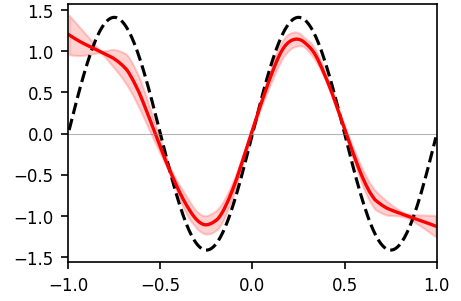}
\caption{Sine: ${0.364}\pm{0.069}$}
\label{fig:syn_sin_2pi}
\end{subfigure}\hfill
\begin{subfigure}[t]{0.235\textwidth}\centering
\includegraphics[width=\linewidth]{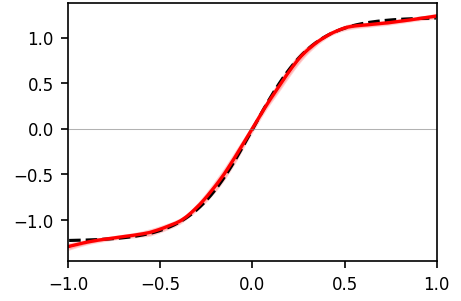}
\caption{Hyperbolic tangent: ${0.034}\pm{0.008}$}
\label{fig:syn_tanh_3x}
\end{subfigure}\hfill
\begin{subfigure}[t]{0.235\textwidth}\centering
\includegraphics[width=\linewidth]{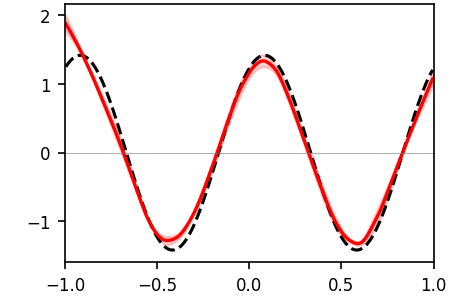}
\caption{Phase-shifted sine: ${0.157}\pm{0.030}$}
\label{fig:syn_sin_2pi_phase}
\end{subfigure}
\par\medskip
\begin{subfigure}[t]{0.235\textwidth}\centering
\includegraphics[width=\linewidth]{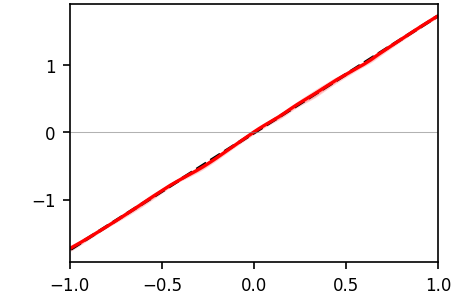}
\caption{Linear: ${0.023}\pm{0.008}$}
\label{fig:syn_x}
\end{subfigure}\hfill
\begin{subfigure}[t]{0.235\textwidth}\centering
\includegraphics[width=\linewidth]{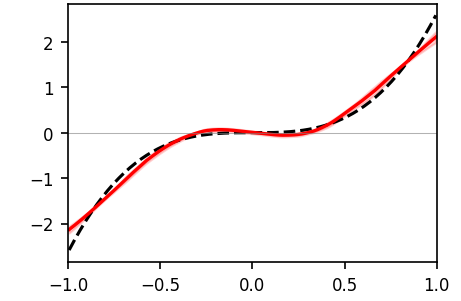}
\caption{Cubic: ${0.139}\pm{0.032}$}
\label{fig:syn_x_cube}
\end{subfigure}\hfill
\begin{subfigure}[t]{0.235\textwidth}\centering
\includegraphics[width=\linewidth]{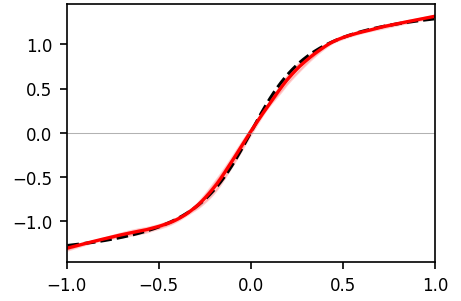}
\caption{Arctangent: ${0.034}\pm{0.016}$}
\label{fig:syn_arctan_4x}
\end{subfigure}\hfill
\begin{subfigure}[t]{0.235\textwidth}\centering
\includegraphics[width=\linewidth]{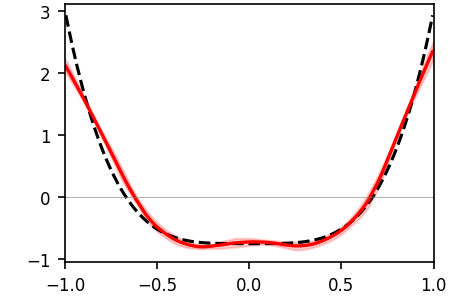}
\caption{Quartic: ${0.183}\pm{0.025}$}
\label{fig:syn_x_fourth}
\end{subfigure}
\par\medskip
}{%
\includegraphics[width=0.90\textwidth,trim=0 31 0 0,clip]{figures/synthetic_function_recovery.png}
}
\caption{\HARMONIA univariate response recovery for twelve ground-truth
functions. Panel annotations show the reported curve NRMSE summaries.
This additional function-recovery study is separate from the feature--hop
benchmark in Table~\ref{tab:synthetic_feature_full}; its curve errors are
defined in Eq.~\eqref{eq:app_curve_nrmse}.}
\label{fig:synthetic_recovery}
\end{figure*}
\section{ADDITIONAL EXPERIMENTAL RESULTS AND ABLATIONS}
\label{sec:app_additional_analysis}
This section empirically examines the main mechanisms underlying HARMONIA. We first study when functional specialization improves over global basis sharing, how this benefit depends on expert capacity and routing quality, and whether MoNB retains its parameter advantage as feature dimension grows. We then examine when RRWP hop coefficients provide reliable structural explanations by separating operator recovery from coefficient recovery and measuring the graph-dependent effective structural horizon. Finally, we evaluate the roles of linearity and simplex normalization in preserving sparse computation and faithful hop-level attribution. Additional robustness and qualitative analyses are presented at the end of the section.

\subsection{Functional Specialization and Heterogeneity}
\label{subsec:app_exp_specialization}

We examine whether the benefit of MoNB increases as feature-response functions
become more heterogeneous. Specifically, the synthetic generator uses four
functional families: polynomial, periodic, saturating, and localized. We
construct low-, medium-, and high-heterogeneity regimes by varying the number
of functional families present in the data, and compare Global NBM, MoNB, and
independent per-feature networks under the same training protocol.

\begin{figure*}[t]
\centering

\begin{subfigure}[t]{0.55\textwidth}
    \centering
    \includegraphics[width=\linewidth]{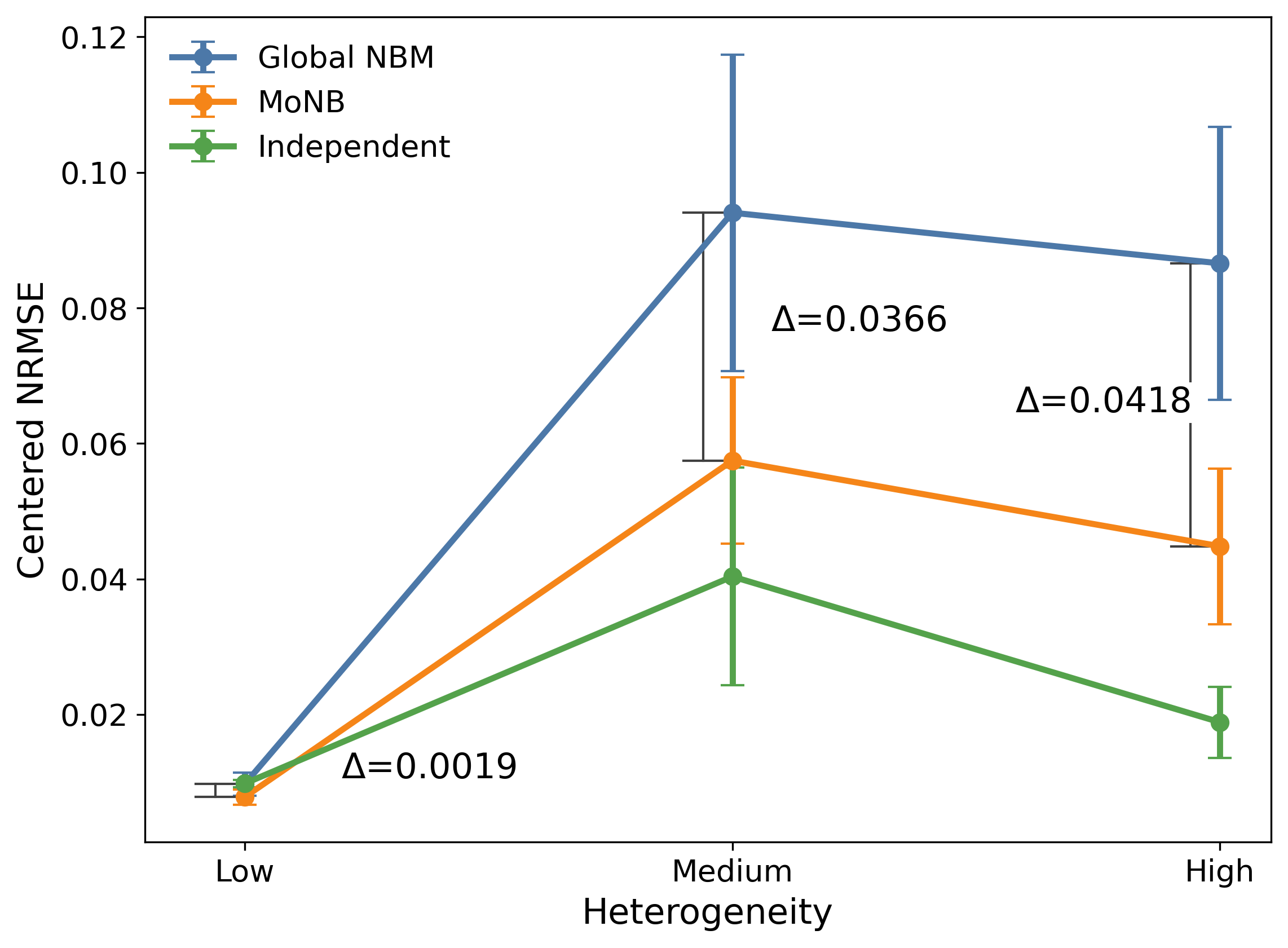}
    \caption{Approximation error under increasing heterogeneity.}
    \label{fig:h1_specialization_performance}
\end{subfigure}
\hfill
\begin{subfigure}[t]{0.41\textwidth}
    \centering
    \includegraphics[width=\linewidth]{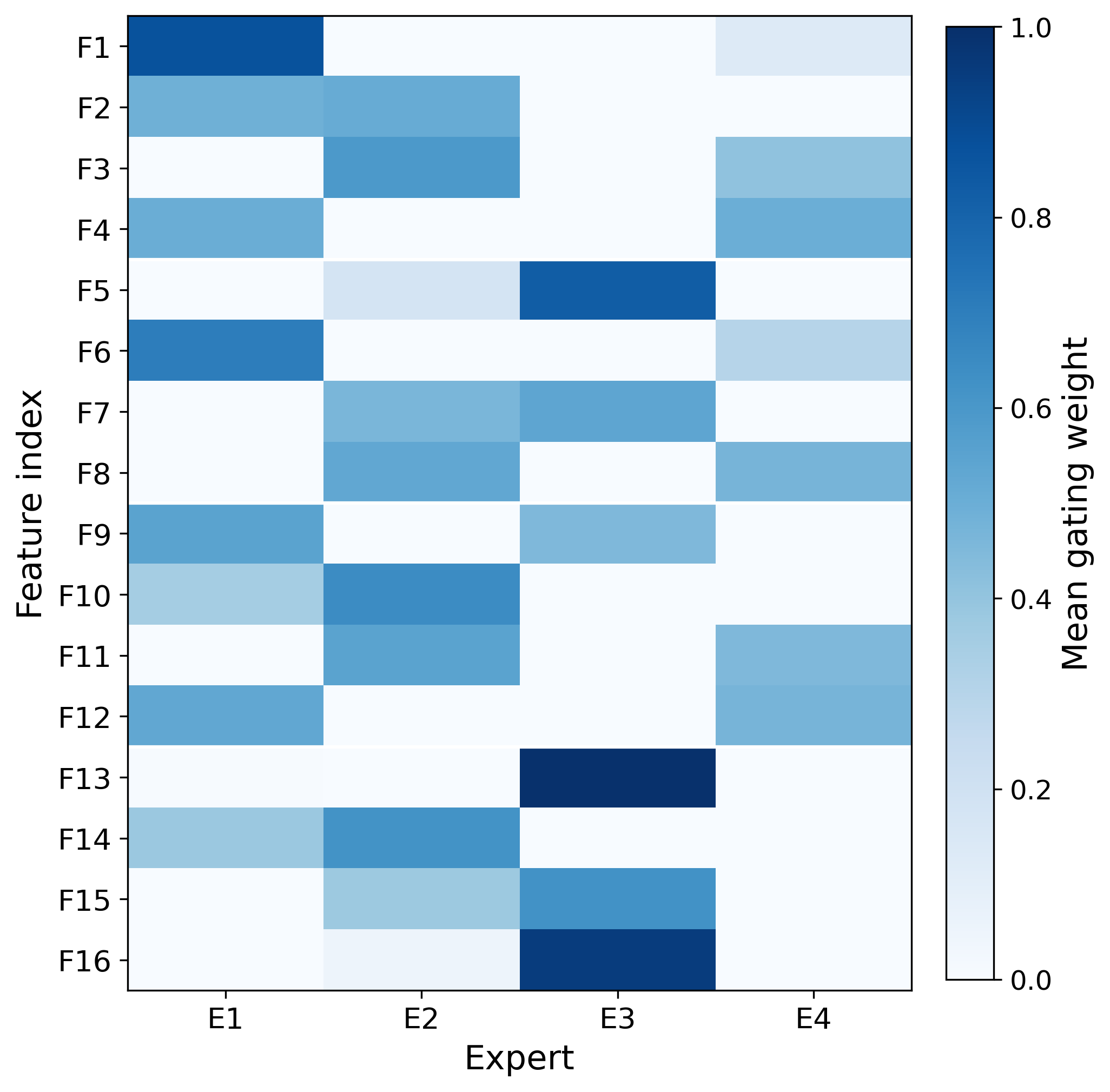}
    \caption{Feature-to-expert routing under high heterogeneity.}
    \label{fig:h1_specialization_routing}
\end{subfigure}

\caption{
Functional specialization under increasing heterogeneity.
\textbf{(a)} Centered NRMSE of Global NBM, MoNB, and independent
per-feature networks. The annotated values
$\Delta = E_{\mathrm{Global}} - E_{\mathrm{MoNB}}$ measure the approximation
gain of MoNB over global basis sharing.
\textbf{(b)} Mean feature-to-expert gating weights in the high-heterogeneity
regime.
}
\label{fig:h1_specialization}
\end{figure*}

Figure~\ref{fig:h1_specialization}(a) shows that the advantage of MoNB grows
with functional heterogeneity. Global NBM and MoNB are nearly indistinguishable
under low heterogeneity, with $\Delta=0.0019$, while the gap increases to
$0.0366$ and $0.0418$ under medium and high heterogeneity. A compact global
basis is therefore sufficient when feature responses share similar functional
structure, but becomes increasingly restrictive as the required functional
directions diversify. MoNB reduces this restriction by allowing different
features to access different expert spaces while retaining parameter sharing.

The independent per-feature model remains more flexible because it removes
parameter sharing entirely. The relevant advantage of MoNB is therefore not
unrestricted expressivity, but its ability to reduce the approximation
limitation of global sharing without assigning an independent nonlinear model
to every feature.

Figure~\ref{fig:h1_specialization}(b) further shows that different features
concentrate their gating weights on different experts rather than following a
single common routing pattern. The routing therefore remains feature-dependent
and does not collapse all features onto one shared expert in this setting.

\paragraph{Insight.}
Specialization provides little benefit when one shared basis already captures
the feature responses. Its advantage emerges as functional heterogeneity
increases, precisely when forcing all features into the same functional space
becomes restrictive.

\subsection{Expert Count and Basis Capacity}
\label{subsec:app_exp_num_experts_bases}

Having established when specialization becomes useful, we next examine how
much specialization capacity is needed to realize this benefit. We vary the
number of experts \(C\in\{1,3,5,7\}\) and the number of neural basis functions
per expert \(B\in\{4,8,16\}\), with sparse routing budget
\(m=\min(2,C)\). These parameters control complementary aspects of the model.
\(C\) determines how many expert spaces are available for specialization,
whereas \(B\) determines the representational capacity within each expert.
Importantly, increasing \(C\) expands the set of candidate functional spaces
without increasing the \(B\)-dimensional coefficient vector used by each
feature.

\begin{figure*}[h]
    \centering
    \begin{subfigure}[t]{0.32\textwidth}
        \centering
        \includegraphics[width=\linewidth]{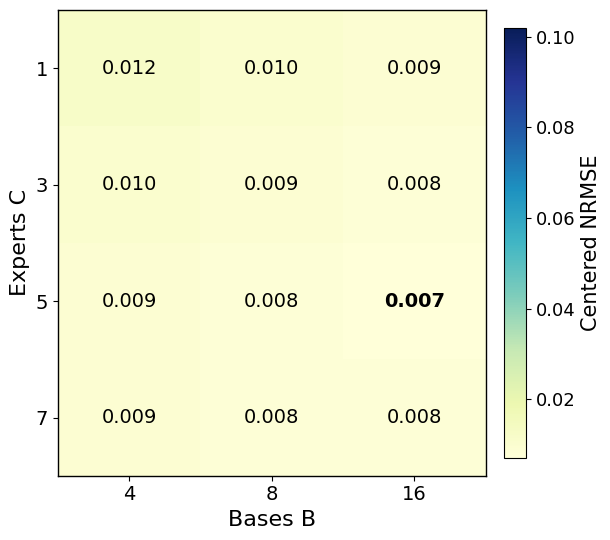}
        \caption{Low heterogeneity}
        \label{fig:h3_capacity_low}
    \end{subfigure}
    \hfill
    \begin{subfigure}[t]{0.32\textwidth}
        \centering
        \includegraphics[width=\linewidth]{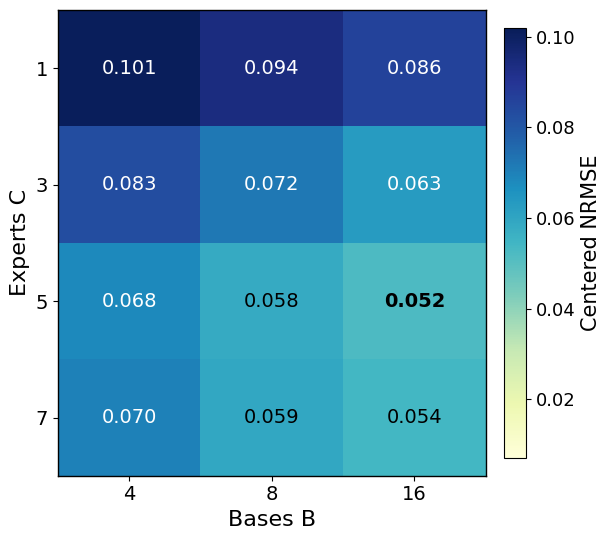}
        \caption{Medium heterogeneity}
        \label{fig:h3_capacity_medium}
    \end{subfigure}
    \hfill
    \begin{subfigure}[t]{0.32\textwidth}
        \centering
        \includegraphics[width=\linewidth]{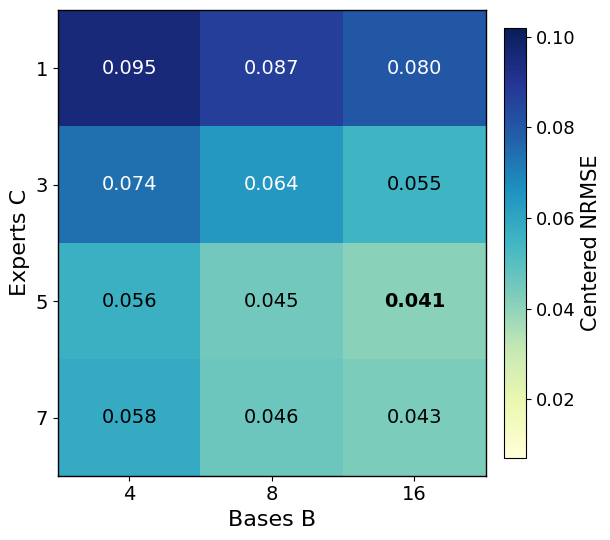}
        \caption{High heterogeneity}
        \label{fig:h3_capacity_high}
    \end{subfigure}

    \vspace{-4pt}
    \caption{
    Effect of expert count and basis size under low, medium, and high
    functional heterogeneity. Each cell reports centered NRMSE with
    \(m=\min(2,C)\), where lower is better. Additional experts provide larger
    gains in more heterogeneous regimes, while increasing \(C\) beyond \(5\)
    provides no further improvement in this sweep.
    }
    \label{fig:h3_capacity}
\end{figure*}

Figure~\ref{fig:h3_capacity} shows that increasing the number of experts is
more beneficial when feature responses are more heterogeneous. At \(B=16\),
increasing \(C\) from \(1\) to \(5\) reduces centered NRMSE only from \(0.009\)
to \(0.007\) under low heterogeneity, but from \(0.086\) to \(0.052\) under
medium heterogeneity and from \(0.080\) to \(0.041\) under high heterogeneity.
Thus, additional expert spaces become useful when a single shared functional
space is no longer sufficient to represent the diverse feature responses.

Basis size controls a different bottleneck. At \(C=5\) under high
heterogeneity, increasing \(B\) from \(4\) to \(8\) and \(16\) reduces the
error from \(0.056\) to \(0.045\) and \(0.041\), respectively. Hence, adding
more expert spaces cannot compensate for insufficient representational
capacity within each expert.

\paragraph{More experts do not guarantee better approximation.}
All three regimes achieve their lowest reported error at
\((C,B)=(5,16)\). Increasing \(C\) to \(7\) changes the errors from
\(0.007\), \(0.052\), and \(0.041\) to \(0.008\), \(0.054\), and
\(0.043\), respectively. The sweep therefore exhibits saturation rather than
a monotonic benefit from increasing the number of experts. This result alone
does not identify the cause, since learned performance can also depend on
routing, basis estimation, and optimization. It nevertheless shows that
additional expert capacity does not automatically translate into better
learned approximation.

\paragraph{Insight.}
\(C\) and \(B\) address different capacity bottlenecks. \(C\) controls the
number of functional spaces available for specialization, while \(B\)
controls how much each space can represent. Too few expert spaces restrict
specialization, whereas increasing their number beyond the useful range does
not necessarily provide additional approximation benefit.

\subsection{Routing Quality and Specialization Gain}
\label{subsec:app_exp_routing}

Having shown that increasing the number of experts does not always improve
approximation, we next examine when routing preserves the benefit of
specialization. To isolate routing from basis learning, we fix the expert
spaces and perturb only the feature-to-expert assignments. We evaluate
\(99\) controlled corruption settings across three heterogeneity regimes
using nearest-wrong, random-wrong, and farthest-wrong routing. These policies
assign a corrupted feature to the closest, a random, or the most distant
incorrect expert in functional space.

\begin{figure}[h]
    \centering

    \begin{subfigure}[t]{0.32\linewidth}
        \centering
        \includegraphics[width=\linewidth]
        {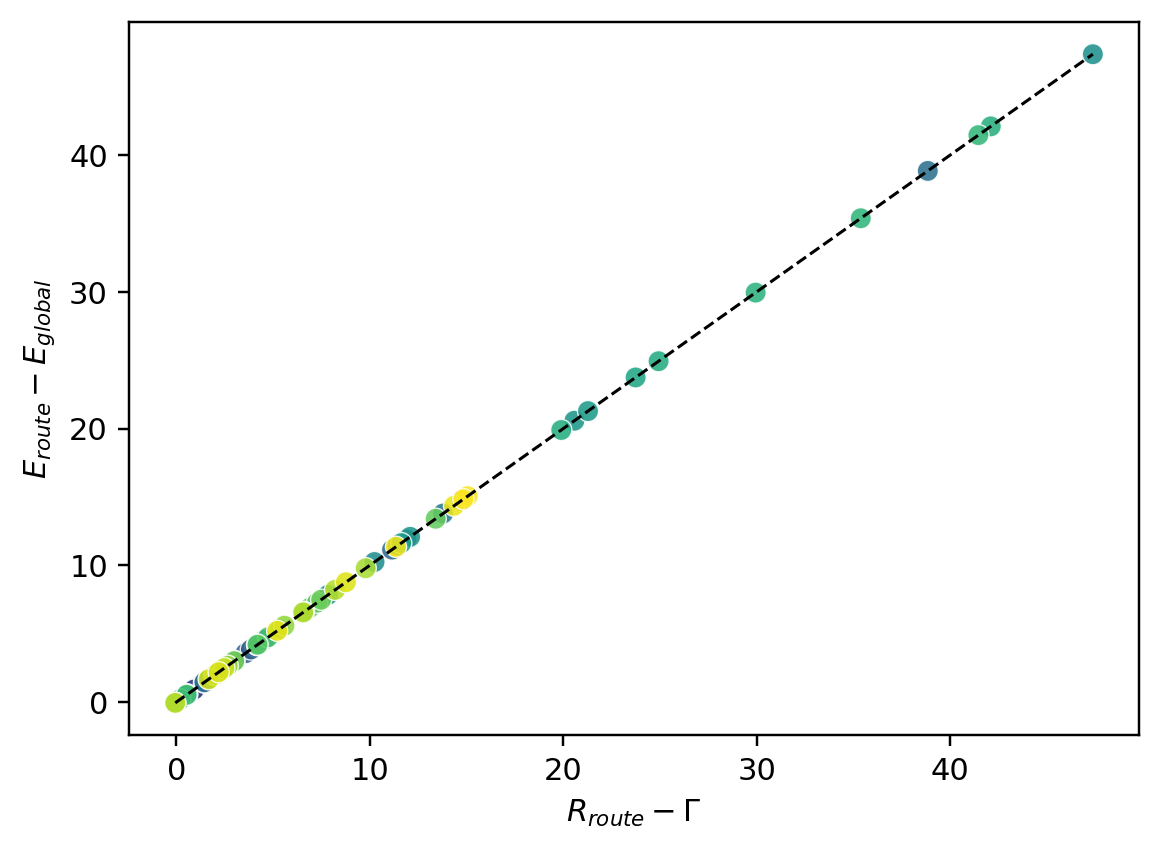}
        \caption{}
        \label{fig:h3_identity}
    \end{subfigure}
    \hfill
    \begin{subfigure}[t]{0.32\linewidth}
        \centering
        \includegraphics[width=\linewidth]
        {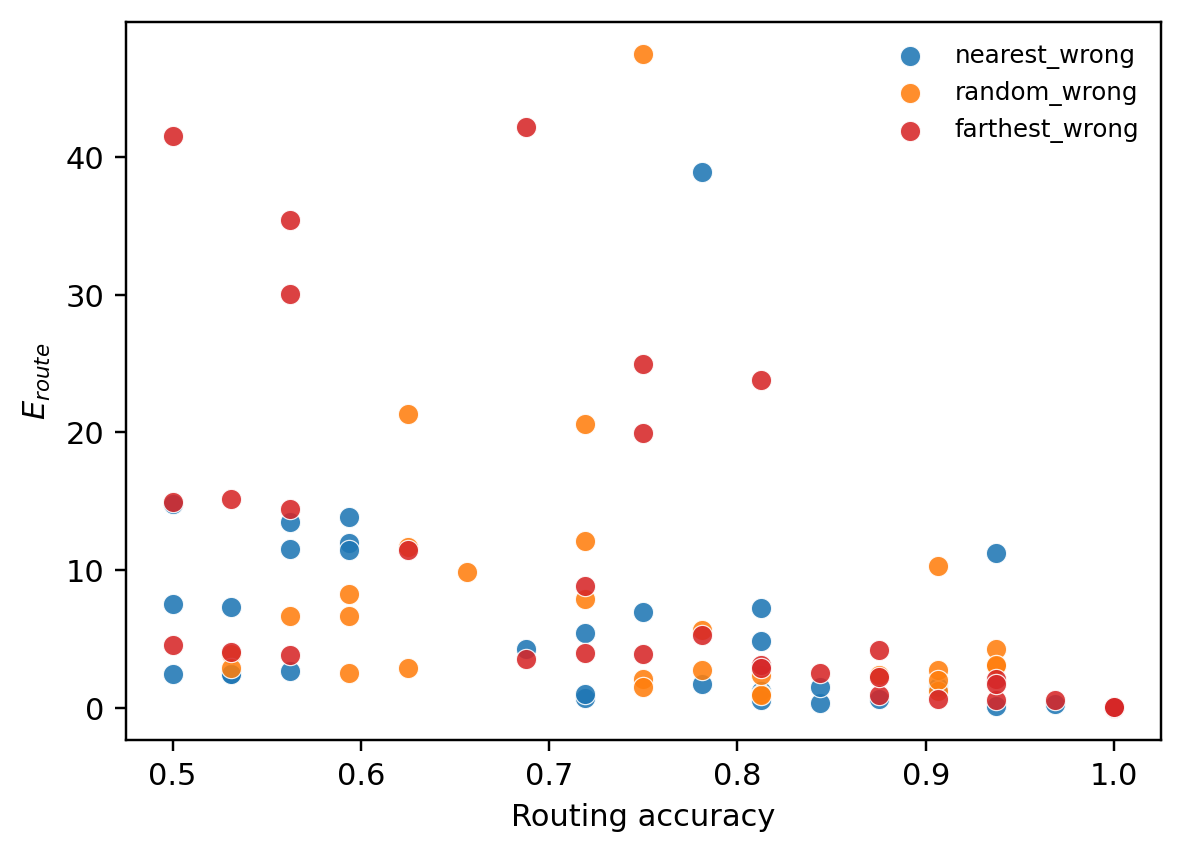}
        \caption{}
        \label{fig:h3_error_accuracy}
    \end{subfigure}
    \hfill
    \begin{subfigure}[t]{0.32\linewidth}
        \centering
        \includegraphics[width=\linewidth]
        {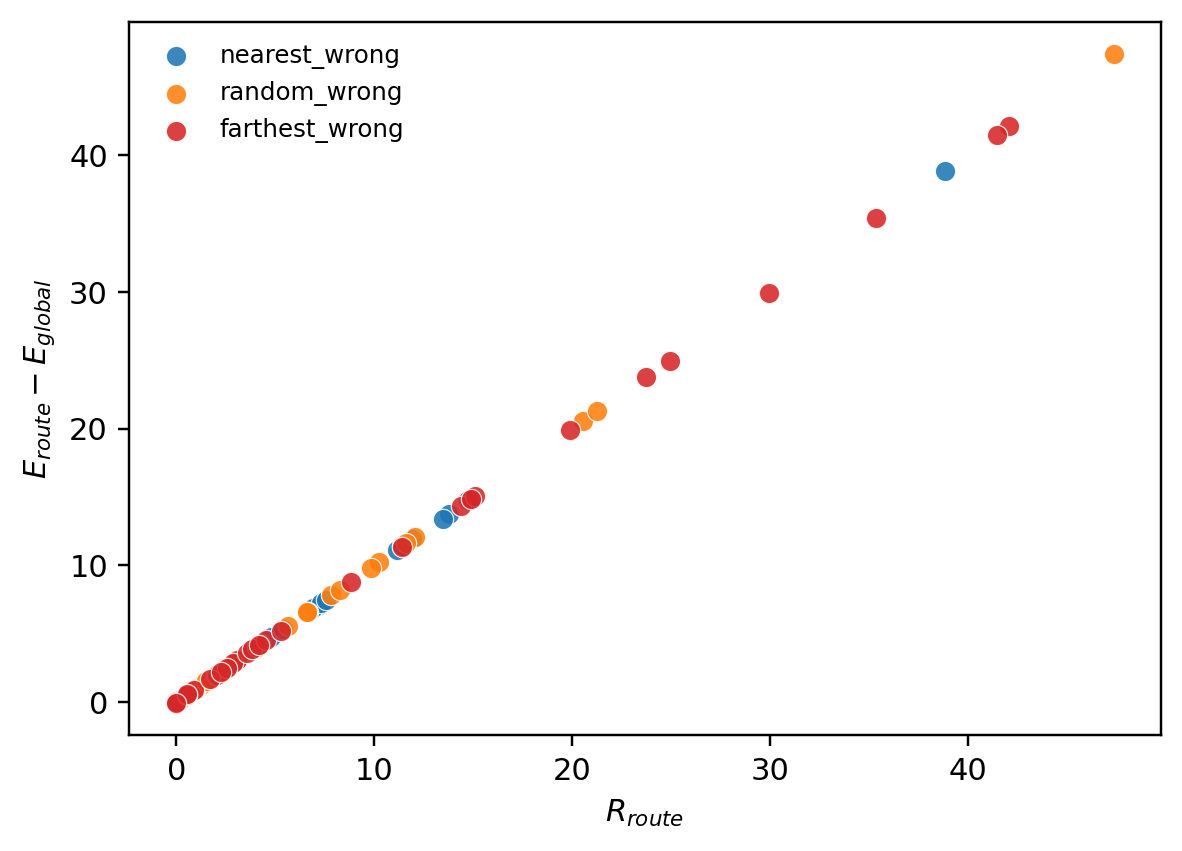}
        \caption{}
        \label{fig:h3_routing_cost}
    \end{subfigure}

    \caption{
    Routing quality and specialization gain.
    \textbf{(a)} Numerical validation of
    \(E_{\mathrm{route}}-E_{\mathrm{global}}
    =
    R_{\mathrm{route}}-\Gamma_{B,C}\).
    \textbf{(b)} Approximation error as a function of routing accuracy under
    nearest-wrong, random-wrong, and farthest-wrong corruption policies.
    Similar routing accuracies can produce substantially different errors.
    \textbf{(c)} Approximation damage as a function of the cost-weighted
    routing penalty \(R_{\mathrm{route}}\). Different corruption policies
    follow nearly the same relationship when routing errors are measured by
    their functional cost.
    }
    \label{fig:h3_routing}
\end{figure}

Figure~\ref{fig:h3_routing}(a) verifies the specialization--routing
decomposition in the controlled setting. Across all corruption settings,
\(E_{\mathrm{route}}-E_{\mathrm{global}}\) agrees with
\(R_{\mathrm{route}}-\Gamma_{B,C}\) up to numerical precision. The same
relation also identifies the point at which specialization loses its
advantage. When \(R_{\mathrm{route}}<\Gamma_{B,C}\), we have
\(E_{\mathrm{route}}<E_{\mathrm{global}}\); once the routing penalty reaches
the specialization gain, this advantage disappears.

More importantly, Figure~\ref{fig:h3_routing}(b) shows that routing accuracy
alone does not characterize routing quality. Settings with similar routing
accuracy can produce substantially different approximation errors. Routing a
feature to an expert close to its oracle expert is generally less damaging
than routing it to a functionally distant expert, even when the number of
incorrect assignments is comparable. Thus, what matters is not only how often
the router makes an incorrect assignment, but how much functional damage each
routing error introduces.

Figure~\ref{fig:h3_routing}(c) makes this distinction explicit through the
cost-weighted routing penalty
\[
R_{\mathrm{route}}
=
\frac{1}{d}
\sum_{k=1}^{d}
\left[
e_k(\widehat z_k)-e_k(z_k^\star)
\right],
\qquad
e_k(c)
=
\left\|
f_k^\star-\Pi_{\mathcal S_c}f_k^\star
\right\|_{\mathcal H}^{2}.
\]
When approximation damage is expressed in terms of \(R_{\mathrm{route}}\),
the different corruption policies follow nearly the same relationship. This
shows that \(R_{\mathrm{route}}\) captures the severity of routing errors more
directly than raw routing accuracy.

The specialization gain \(\Gamma_{B,C}\) can therefore be interpreted as a
routing-error budget:
\[
R_{\mathrm{route}} < \Gamma_{B,C}.
\]
Specialization creates an approximation gain, while routing errors consume
that gain. Once their functional cost reaches or exceeds
\(\Gamma_{B,C}\), specialization no longer improves over global sharing.

\paragraph{Insight.}
Not all routing errors are equally harmful. Routing quality depends on the
functional cost of the selected expert rather than only on the fraction of
correct assignments. The benefit of MoNB is therefore governed by the balance
between specialization gain and routing penalty.

\subsection{Hop-Level Explanation Reliability}
\label{subsec:app_exp_hop_identifiability}

Having examined when routing preserves functional specialization, we now turn
to the reliability of structural explanations. \HARMONIA exposes structural
influence through the RRWP hop coefficients $\theta$. For a node pair
$(i,j)$ with RRWP descriptor
$\mathbf p_{i,j}=[(M^0)_{i,j},\ldots,(M^{T-1})_{i,j}]^\top$, the structural
weight is $\theta^\top\mathbf p_{i,j}$. Across all node pairs, the same
coefficients induce the propagation operator
\[
P_\theta=\sum_{t=0}^{T-1}\theta_t M^t.
\]
Thus, $P_\theta$ describes the structural transformation implemented by the
model, whereas $\theta$ describes how that transformation is attributed across
walk lengths. Appendix~\ref{sec:app_identifiability} shows that a reliable
hop-level explanation requires not only uniqueness of this decomposition but
also sufficient conditioning.

To measure this stability, Appendix~\ref{sec:app_identifiability} defines the
restricted conditioning constant
\[
\gamma_T
=
\inf_{\substack{\delta^\top\mathbf 1=0\\ \|\delta\|_2=1}}
\left\|
\sum_{t=0}^{T-1}\delta_tM^t
\right\|_F.
\]
Here, $\delta$ represents a normalized redistribution of hop weights.
Therefore, $\gamma_T$ measures the smallest observable change in $P_\theta$
caused by a unit change in $\theta$. A large $\gamma_T$ means that substantially
different hop profiles must induce substantially different operators. A small
$\gamma_T$, in contrast, means that the hop coefficients can change
considerably while the resulting structural transformation changes very
little. Hence, even when the coefficients remain uniquely identifiable,
their hop-level interpretation can become unstable as $\gamma_T$ approaches
zero.

We test this phenomenon using a ground-truth operator
\[
P^\star=\sum_{t=0}^{T-1}\theta_t^\star M^t,
\qquad
\theta^\star\in\Delta^{T-1},
\]
and recover $\hat{\theta}$ from either $P^\star$ or a perturbed observation
$\widetilde P=P^\star+E$. Across $520$ controlled recovery settings, we vary
the graph family, walk horizon $T$, conditioning regime, and perturbation
level. We measure coefficient recovery by
\[
e_\theta=\|\hat{\theta}-\theta^\star\|_2
\]
and operator recovery by
\[
e_{\mathrm{op}}
=
\left\|
P_{\hat{\theta}}-P_{\theta^\star}
\right\|_F.
\]

\begin{figure}[t]
    \centering
    \includegraphics[width=0.5\linewidth]
    {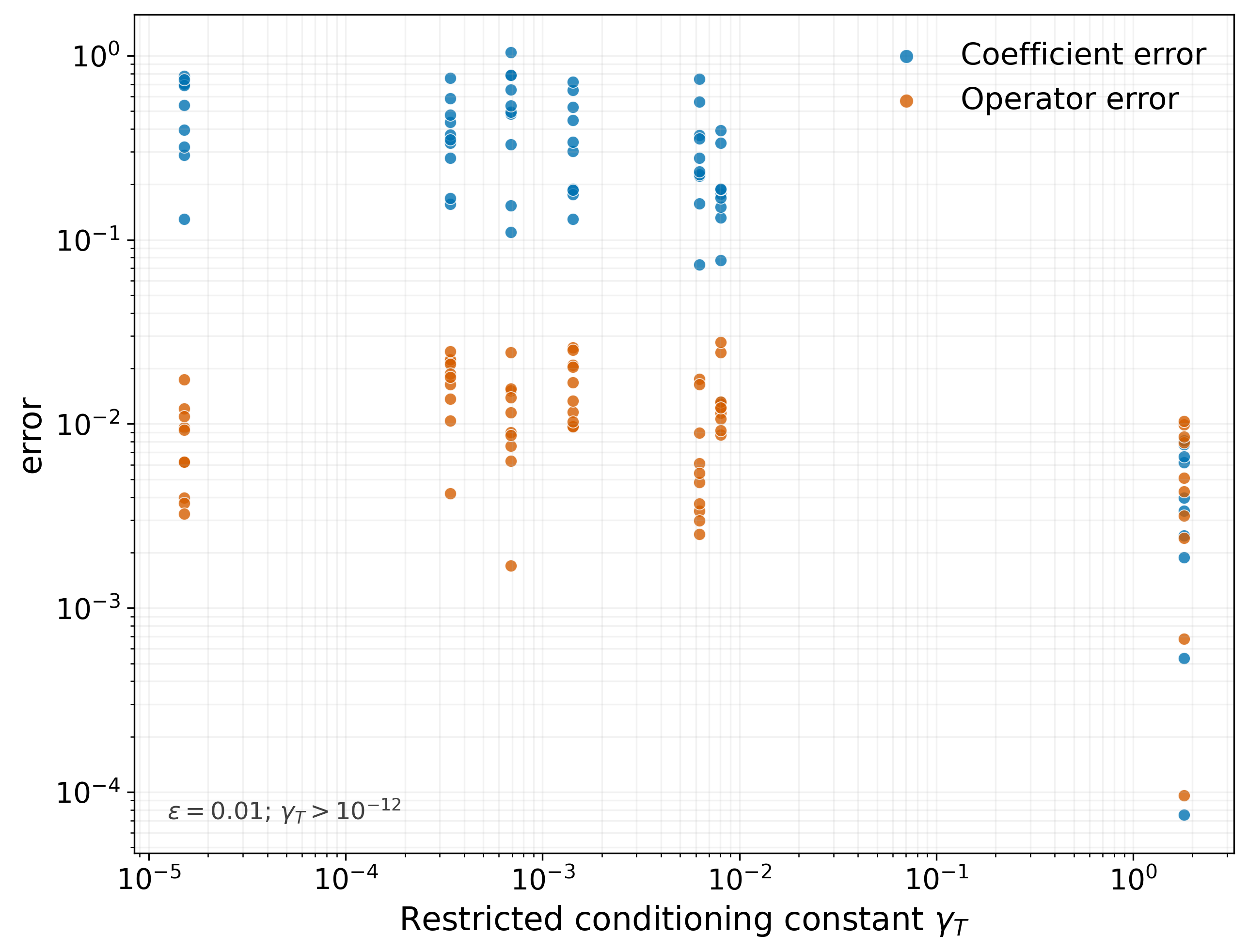}
    \caption{
    Hop recovery under varying structural conditioning at
    $\epsilon=0.01$. Coefficient error $e_\theta$ measures recovery of the
    hop allocation, while operator error $e_{\mathrm{op}}$ measures recovery
    of the induced structural transformation. Only numerically non-singular
    settings with $\gamma_T>10^{-12}$ are shown. As $\gamma_T$ decreases,
    operator recovery remains accurate while coefficient recovery becomes
    increasingly unstable.
    }
    \label{fig:hop_identifiability}
\end{figure}

Figure~\ref{fig:hop_identifiability} shows a clear separation between
structural correctness and explanation reliability. When $\gamma_T$ is large,
both $e_{\mathrm{op}}$ and $e_\theta$ are small, so recovering the propagation
operator also recovers its hop allocation. As $\gamma_T$ decreases, however,
the operator error remains low while the coefficient error increases by orders
of magnitude. Representative ill-conditioned settings under
$\epsilon=0.01$ produce coefficient errors of approximately
$0.26$--$0.54$, while the corresponding operator errors remain only
$0.007$--$0.018$. Thus, nearly identical structural transformations can
support substantially different allocations of importance across walk
lengths.

This behavior directly matches the stability result in
Appendix~\ref{sec:app_identifiability},
\[
\|\hat{\theta}-\theta^\star\|_2
\le
\frac{2\epsilon}{\gamma_T}.
\]
For the same operator-level perturbation $\epsilon$, smaller $\gamma_T$
permits a larger deviation in the recovered hop coefficients. Importantly,
Figure~\ref{fig:hop_identifiability} contains only non-singular settings, so
the observed instability arises before exact non-identifiability. Exact
uniqueness is therefore necessary but not sufficient for a reliable
hop-level explanation.

\paragraph{Insight.}
$P_\theta$ determines the structural behavior of the model, while $\theta$
determines how that behavior is explained across walk lengths. The conditioning
constant $\gamma_T$ controls the stability of this connection. Consequently,
accurate operator recovery does not by itself guarantee reliable hop recovery.

\subsection{Empirical Effective Structural Horizon}
\label{subsec:app_exp_horizon}

Increasing the RRWP truncation length $T$ makes more propagation depths
available to the model, but it does not guarantee that each additional walk
length provides a genuinely distinct structural signal. We therefore ask how
far the random walk can propagate before neighboring walk lengths become
effectively indistinguishable, and whether this range is determined by graph
mixing rather than by the chosen truncation length alone.

To expose this dependence, we consider connected undirected synthetic graphs
with deliberately different mixing regimes. A \emph{ring} provides a
slow-mixing reference, where information spreads only
\begin{figure}[t]
    \centering
    \begin{minipage}[t]{0.485\linewidth}
        \centering
        \includegraphics[width=\linewidth]
        {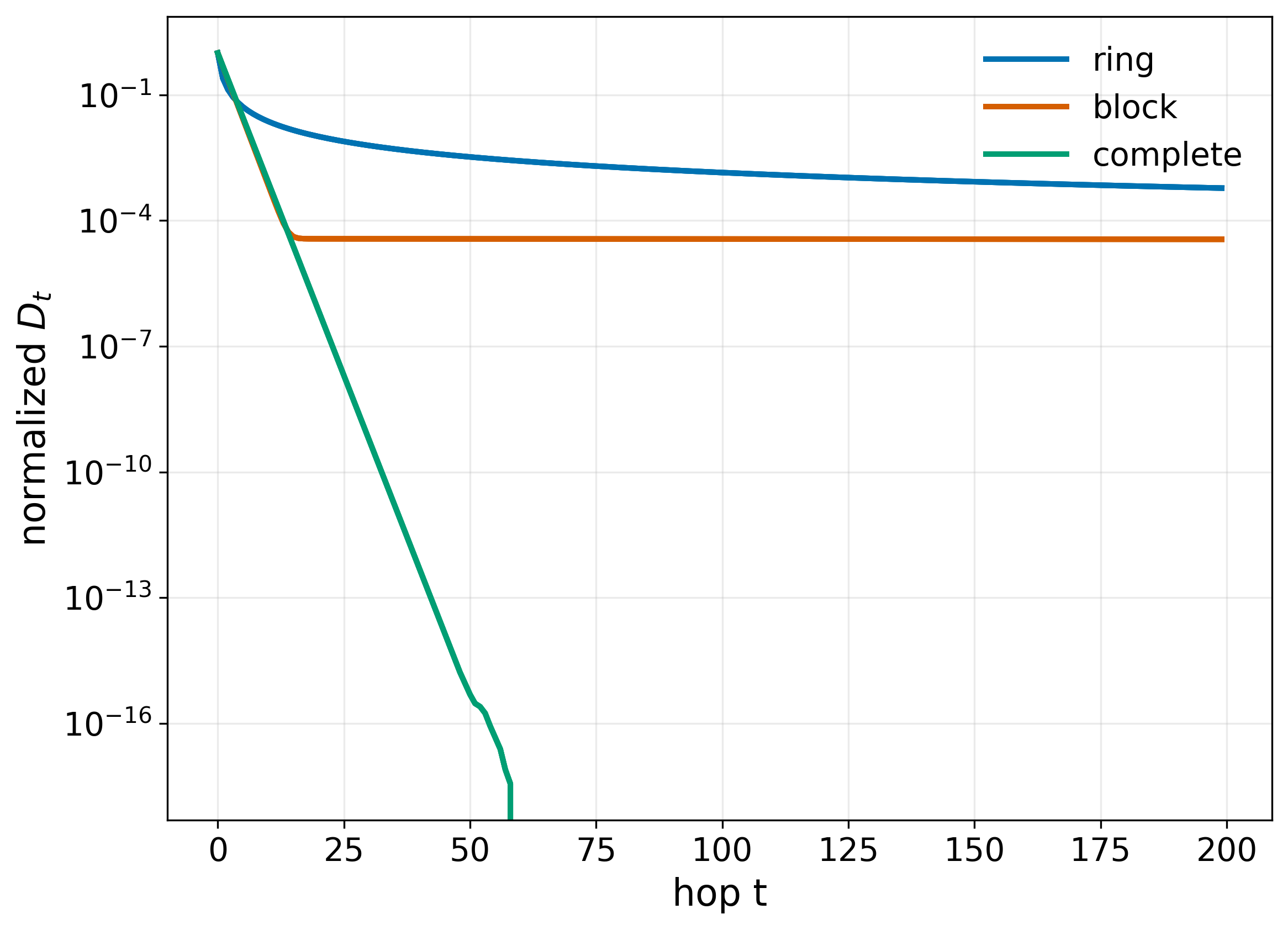}
        \vspace{1pt}
        \textbf{(a)}
    \end{minipage}
    \hfill
    \begin{minipage}[t]{0.485\linewidth}
        \centering
        \includegraphics[width=\linewidth]
        {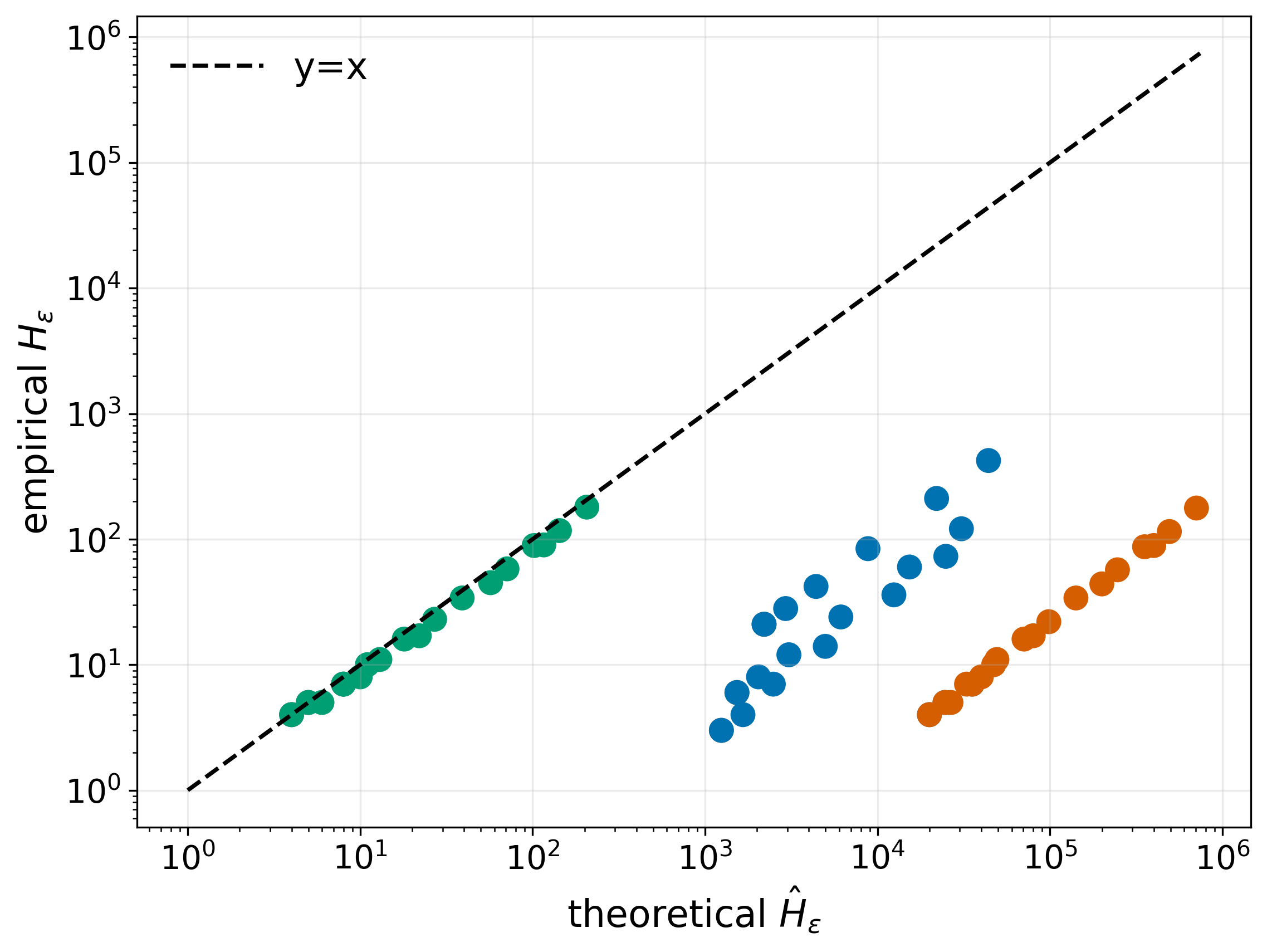}
        \vspace{1pt}
        \textbf{(b)}
    \end{minipage}
    \caption{
    Empirical validation of Theorem~\ref{thm:structural_horizon}.
    \textbf{(a)} Consecutive-hop difference $D_t$ across graph families with
    different mixing regimes.
    \textbf{(b)} Empirical horizon $H_\epsilon$ versus the spectral upper
    bound $\widehat{H}_\epsilon$; all settings remain below the identity line.
    }
    \label{fig:h5_horizon}
\end{figure}gradually through local
neighbors. A \emph{block-structured} graph introduces community structure,
so propagation mixes readily within communities but more slowly across them.
A \emph{complete} graph provides a rapidly mixing reference, since every node
is directly connected to every other node. For all graph families, we use
the lazy random-walk operator
\[
M
=
\frac{1}{2}\left(I+D^{-1}A\right),
\]
which satisfies the spectral assumptions of Theorem~\ref{thm:structural_horizon}.

The first claim of Theorem~\ref{thm:structural_horizon} concerns how quickly neighboring propagation
depths lose their distinction. Let
\[
\rho=\max_{r\geq2}|\lambda_r|<1
\]
denote the largest non-stationary eigenvalue magnitude, and define
\[
D_t
=
\left\|M^{t+1}-M^t\right\|_{2,\pi}.
\]
Theorem~\ref{thm:structural_horizon} gives
\[
D_t
=
\max_{r\geq2}|\lambda_r|^t|1-\lambda_r|
\leq
(1+\rho)\rho^t.
\]

Thus, consecutive walk lengths should become progressively less
distinguishable as $t$ increases. Importantly, this decay is graph-dependent:
a smaller $\rho$ implies faster mixing and faster loss of hop-specific
structure, whereas $\rho$ closer to one allows distinct propagation scales to
persist for longer. This behavior is exactly what is observed in
Figure~\ref{fig:h5_horizon}(a): the complete graph rapidly loses
consecutive-hop distinction, the ring preserves it over much larger depths,
and the block-structured graph exhibits a distinct intermediate regime.
Hence, the same numerical truncation $T$ can correspond to very different
amounts of usable structural resolution across graphs.

The second claim turns this decay into a graph-dependent structural horizon.
For a chosen resolution $\epsilon>0$, we define
\[
H_\epsilon
=
\min
\left\{
t\geq0:
D_s\leq\epsilon
\text{ for all }s\geq t
\right\}.
\]
Beyond $H_\epsilon$, taking one additional random-walk step changes the
propagation operator by at most the prescribed resolution. Theorem~\ref{thm:structural_horizon} further
gives
\[
H_\epsilon
\leq
\widehat{H}_\epsilon
:=
\left\lceil
\frac{\log((1+\rho)/\epsilon)}
{-\log\rho}
\right\rceil.
\]
The dependence on $\rho$ therefore predicts that rapidly mixing graphs should
have short effective horizons, whereas slowly mixing graphs can preserve
distinguishable propagation scales for substantially longer. In
Figure~\ref{fig:h5_horizon}(b), every empirically measured $H_\epsilon$
remains below its corresponding spectral bound $\widehat{H}_\epsilon$.
The theoretical expression is therefore conservative rather than an exact
predictor, but it provides a graph-dependent upper estimate of the range of
walk lengths that remain structurally distinguishable. In this sense, $T$
specifies how many propagation depths are exposed to RRWP, whereas
$H_\epsilon$ characterizes how many of those depths remain distinguishable at
the chosen structural resolution.

\subsection{RRWP Normalization and Linear--Nonlinear Ablations}
\label{subsec:app_exp_rrwp_ablation}
The linear RRWP topology function is interpretable and scalable, but it can
only assign additive weights to individual random-walk channels. To improve
expressiveness, we also consider a DNN-based topology function that directly
maps the RRWP vector $\mathbf{p}_{i,j}$ to feature-specific topology scores. Specifically, we use an MLP $g_{\phi}:\mathbb{R}^{T}\rightarrow\mathbb{R}^{d}$ ,
where the $k$-th output dimension gives the topology score for feature $k$:
\begin{equation}
    \omega_k^{\sharp}(\mathbf{p}_{i,j})
    =
    [g_{\phi}(\mathbf{p}_{i,j})]_k.
\end{equation}

Substituting this topology function into
Eq.~\eqref{eq:problem_representation}, we obtain
\begin{equation}
    [\mathbf{h}_i]_k
    =
    \sum_{j=1}^{|\mathcal{V}|}
    [g_{\phi}(\mathbf{p}_{i,j})]_k
    f_k([\mathbf{x}_j]_k)
\end{equation}

This formulation allows \HARMONIA to learn nonlinear relationships among RRWP channels, making the topology function more expressive than the linear variant. However, this comes at the cost of scalability and interpretability. Since $g_{\phi}(\cdot)$ nonlinearly mixes the RRWP channels, the
resulting structural contribution can no longer be decomposed into the explicit walk-length terms
$\theta_{t,k}z^{(t)}_{i,k}$. Thus, the DNN-based topology function
offers greater expressiveness, whereas the linear variant supports
sparse computation and explicit attribution to individual
$t$-hop channels

\begin{table}[t]
\centering
\small
\renewcommand{\arraystretch}{1.15}
\setlength{\tabcolsep}{4pt}

\caption{Comparison between \HARMONIA (Linear RRWP) and 
\HARMONIA (DNN RRWP) across graph and node classification benchmarks. 
OOM indicates Out-Of-Memory ($>24$\,GB VRAM).}

\label{tab:dnn_vs_linear_rrwp}

\resizebox{\linewidth}{!}{%
\begin{tabular}{lcccccc}
\toprule

\multirow{2}{*}{\textbf{Model Variant}} 
& \multicolumn{3}{c}{\textbf{Graph Classification (Acc. \%)}} 
& \multicolumn{3}{c}{\textbf{Node Classification (Acc. \%)}} \\

\cmidrule(lr){2-4}
\cmidrule(lr){5-7}

& \textbf{Mutag} 
& \textbf{Proteins} 
& \textbf{AIDS} 
& \textbf{Cora} 
& \textbf{Citeseer} 
& \textbf{ogbn-arxiv} \\

\midrule

\HARMONIA (DNN RRWP) 
& \textbf{71.25 $\pm$ 1.8} 
& \textbf{72.10 $\pm$ 1.4} 
& \textbf{98.15 $\pm$ 0.9} 
& \textbf{80.52 $\pm$ 0.9} 
& \textbf{73.64 $\pm$ 0.8} 
& OOM \\

\HARMONIA (Linear RRWP) 
& 69.50 $\pm$ 1.5 
& 70.76 $\pm$ 1.2 
& 97.75 $\pm$ 1.1 
& 79.60 $\pm$ 1.3 
& 72.80 $\pm$ 1.1 
& \textbf{78.14 $\pm$ 2.3} \\

\bottomrule
\end{tabular}%
}

\end{table}

\subsection{Explanation Stability Across Random Seeds}
\label{subsec:app_explanation_stability}
\subsubsection{Remove-and-Retrain Evaluation}

The feature-removal experiment follows the rationale of
remove-and-retrain evaluation~\citep{hooker2019benchmark}.
Let $R_{\mathrm{top}}$, $R_{\mathrm{bottom}}$, and
$R_{\mathrm{random}}$ contain the same fraction ($20\%$) of input
features, chosen from the top, bottom, or a random part of the ranking.
The chosen columns are removed consistently from the training,
validation, and test inputs, and a new model is trained from scratch.
The diagnostic is
\begin{equation}
\Delta\operatorname{Acc}(R)
=\operatorname{Acc}_{\mathrm{test}}(F_{-R})
-\operatorname{Acc}_{\mathrm{test}}(F_{\mathrm{full}}).
\end{equation}
A more negative value indicates a larger loss of predictive information
under removal and retraining.
Feature selection must not use test labels; the available importance
plot identifies validation-margin importance.

\begin{figure}[t]
\centering
\includegraphics[width=0.50\linewidth]{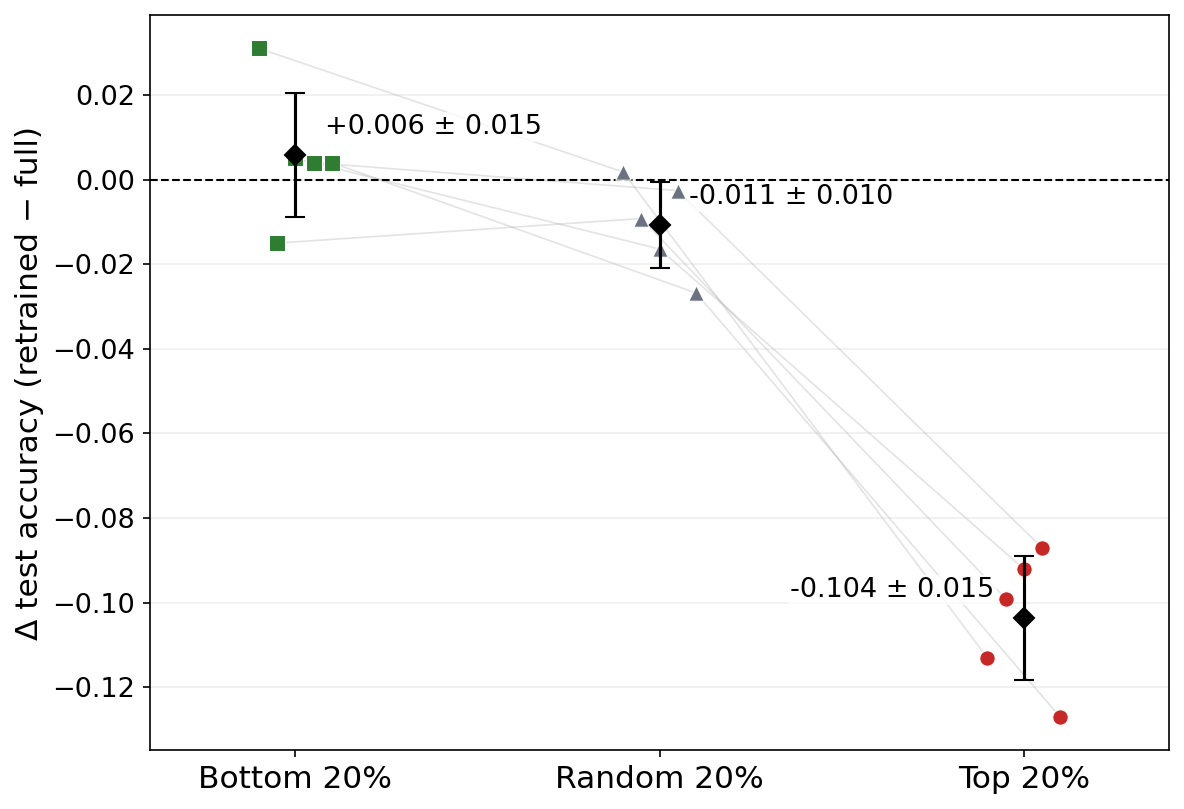}
\caption{Change in test accuracy after removal and retraining.
The plot uses accuracy fractions; $-0.104$ corresponds to a decrease
of $10.4$ percentage points. Printed summaries are reproduced from
the experiment.}
\label{fig:remove_retrain}
\end{figure}

Removing the top-ranked features produces a much larger mean
performance decrease than the two controls
(Figure~\ref{fig:remove_retrain}).
The small positive mean after bottom-feature removal does not by
itself establish a statistically significant improvement.
This diagnostic assesses predictive relevance under retraining,
not a causal effect in the data-generating system.

\subsubsection{Importance Stability Across Seeds}

For seed $s$, let $\mathbf I^{(s)}$ be the importance vector measured
on the same reference data and let
$\bar I_k=S^{-1}\sum_{s=1}^{S}I_k^{(s)}$.
The scatter plot~\ref{fig:seed_stability} compares
$(\bar I_k,I_k^{(s)})$ using logarithmic axes.
Its diagonal is a reference for agreement in magnitude.
Agreement with a mean that contains the plotted seed is a descriptive
visualization, not an independent reliability estimate.

A rank-based summary, when computed from the run outputs, is
\begin{equation}
\operatorname{Stability}_{\mathrm{rank}}
=\frac{2}{S(S-1)}
\sum_{a<b}\operatorname{Spearman}
       \bigl(\mathbf I^{(a)},\mathbf I^{(b)}\bigr).
\end{equation}
Spearman correlation is unchanged by positive normalization of each
importance vector. Constant vectors require explicit handling.
No numerical Spearman value is inferred from the scatter plot.

\begin{figure*}[t]
\centering
\includegraphics[width=\textwidth]{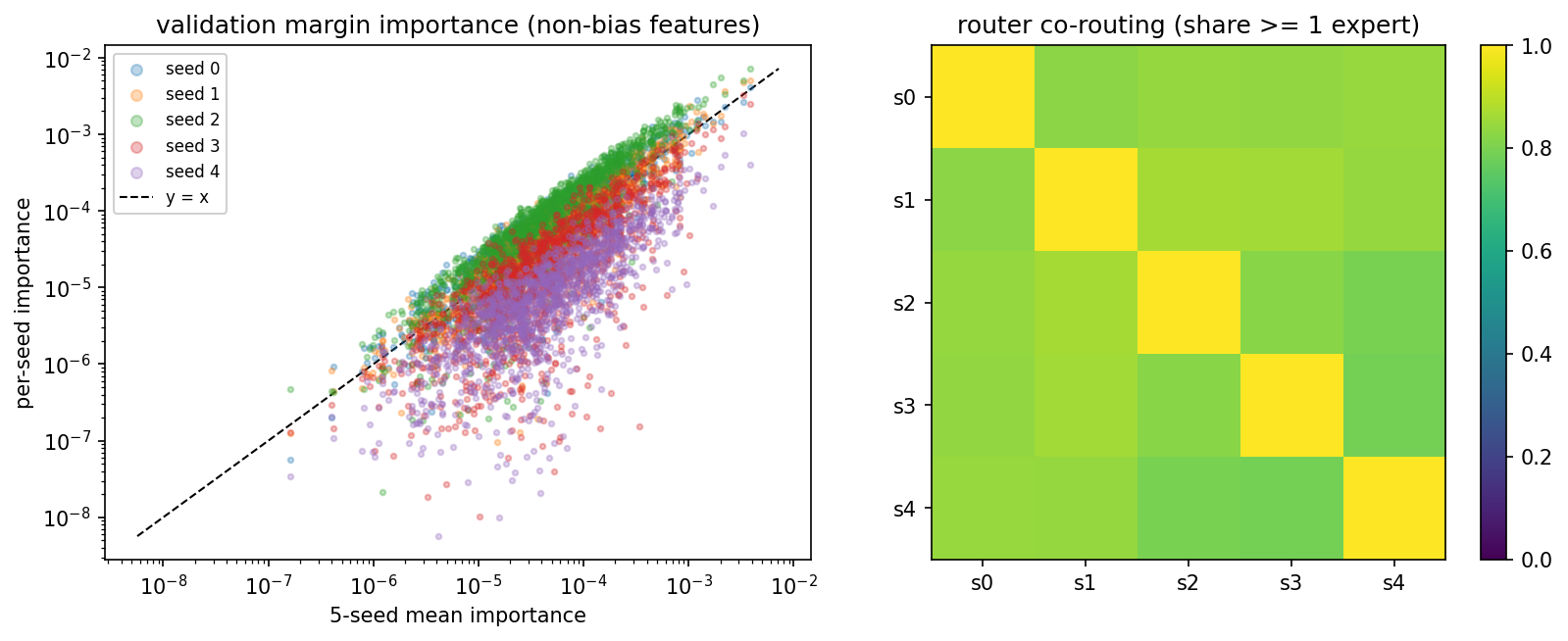}
\caption{Five-seed visualization of feature importance.
The left panel displays validation-margin importance for non-bias
features
}
\label{fig:seed_stability}
\end{figure*}

\subsection{Parameter Complexity}
\label{subsec:app_parameter_complexity}

To evaluate parameter scaling under a controlled setting, we construct a
synthetic multi-class setup in which the input feature dimension \(D\) is varied
while the number of output classes is fixed to \(K=47\). All other
architecture-specific hyperparameters are kept fixed throughout the comparison,
so that the experiment isolates the effect of increasing feature dimensionality
on the number of learnable parameters.

For GNAN, each input feature is modeled by a separate MLP, together with an
additional network for topology. For an MLP with \(L\) hidden layers of width
\(H_u\), extended to \(K\) output classes, the parameter count is

\[
P_{\mathrm{GNAN}}
=
(D+1)
\left[
H_u^2(L-1)
+
(L+K+1)H_u
+
K
\right].
\]

For G-NAMRFF, each feature learns an \(M\)-dimensional RFF representation for
each output class, while the FIR graph filter contributes \(R+1\) shared
parameters, giving

\[
P_{\mathrm{G\text{-}NAMRFF}}
=
DMK + R + 1.
\]

\[
P_{\text{HARMONIA}}
=
E P_{\psi}
+
2qE
+
D(B+q+T+K)
+
K.
\]

The key difference is therefore in the coefficient multiplying \(D\).
GNAN replicates an entire neural network for every feature, while G-NAMRFF
requires \(MK\) learnable weights per feature. In contrast, \HARMONIA shares
the nonlinear expert parameters across features, so increasing \(D\) introduces
only the feature-specific basis coefficients, routing embeddings, RRWP
coefficients, and linear output weights.

Figure~\ref{fig:parameter_complexity} compares the number of learnable parameters as the input feature dimension \(D\) increases. Although the parameter count grows with \(D\) for all methods, \HARMONIA exhibits substantially lower parameter complexity than GNAN and G-NAMRFF across the entire range. At \(D=5{,}000\), \HARMONIA requires only on the order of \(10^{5}\) learnable parameters, whereas GNAN and G-NAMRFF reach the order of \(10^{7}\), resulting in a gap of approximately one to two orders of magnitude.

This difference becomes increasingly important for high-dimensional inputs, where feature-wise parameterization can lead to a large model size. \HARMONIA mitigates this growth by sharing a compact collection of neural basis experts across features while learning only lightweight feature-specific coefficients and routing parameters. These results show that the proposed MoNB parameterization substantially reduces the parameter overhead associated with increasing feature dimensionality while preserving feature-specific functional specialization.

\begin{figure}[h]
    \centering
    \vspace{-5pt}
    \includegraphics[width=0.5\linewidth]{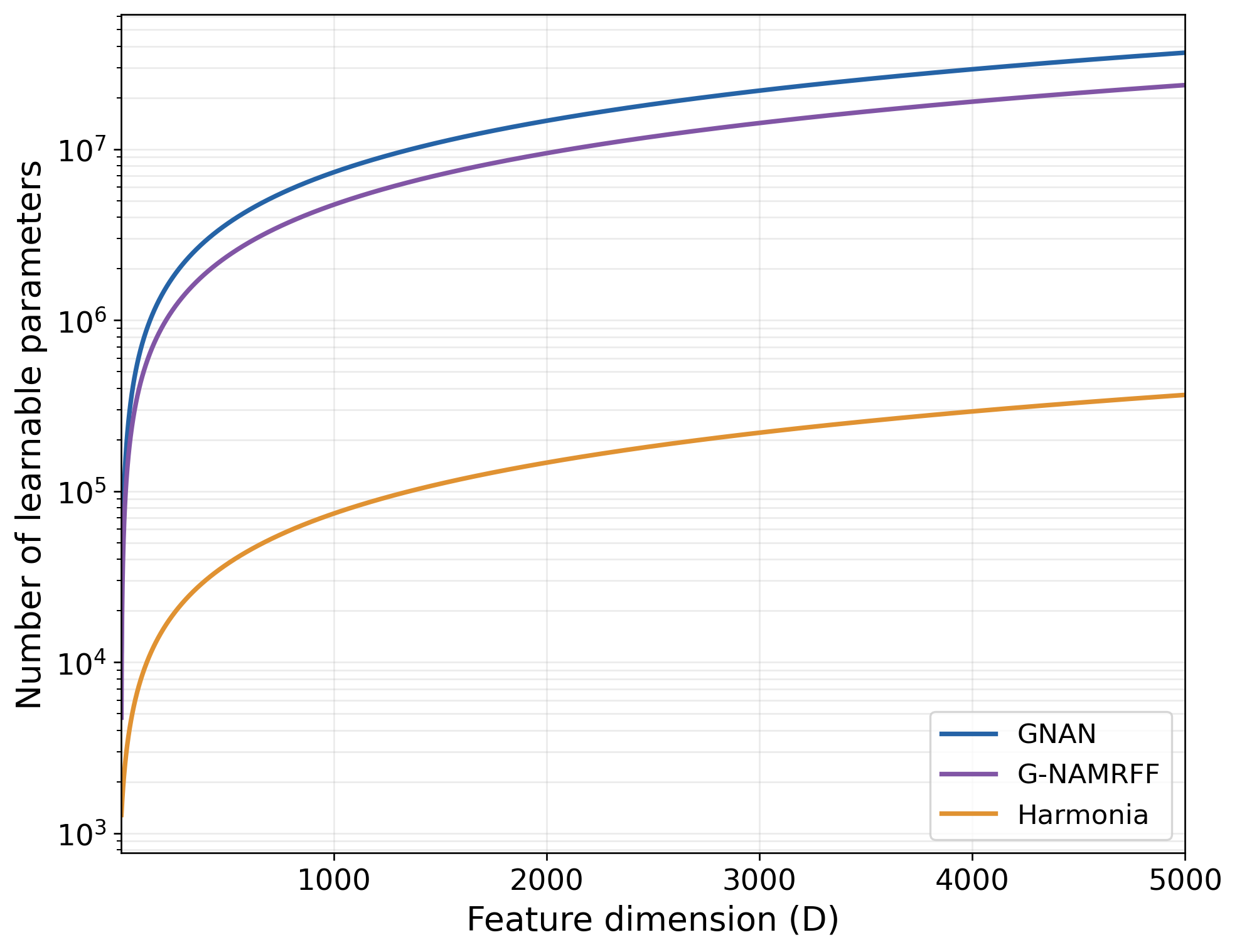}
    \vspace{-5pt}
    \caption{Parameter complexity in the synthetic multi-class setting with
    \(K=47\) output classes as the input feature dimension \(D\) increases.}
    \label{fig:parameter_complexity}
    \vspace{-5pt}
\end{figure}

\subsection{Learned Feature Shape Functions}
\label{subsec:app_feature_shapes}

We first inspect the feature-response functions learned by the additive
component of \HARMONIA on Mutagenicity. Because the prediction decomposes into
feature-wise contributions, each learned function can be evaluated directly
over the observed feature range without fitting a separate post-hoc explainer.
Figure~\ref{fig:feature_shape_functions} shows representative response curves
for the atom-related input dimensions used by the model.

\begin{figure}[t]
    \centering
    \includegraphics[width=0.5\textwidth]
    {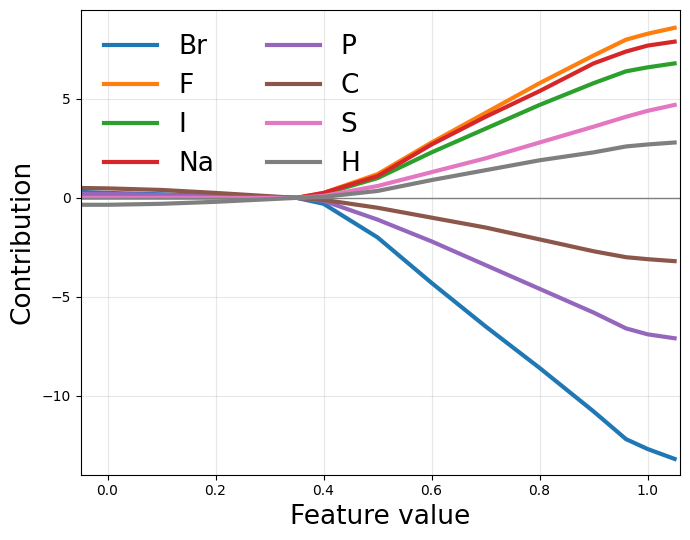}
    \caption{
    Selected feature-response functions learned by \HARMONIA on the
    Mutagenicity dataset. Each curve shows the contribution associated with
    one atom-related feature as its normalized feature value varies.
    }
    \label{fig:feature_shape_functions}
\end{figure}

The curves demonstrate that the shared neural bases do not collapse the
features to a common response shape. Instead, the model produces distinct
feature-specific functions with different signs, magnitudes, and nonlinear
trends. 

Importantly, these functions should be interpreted as properties of the
trained predictive model. They describe how a feature contributes conditional
on the learned representation and observed data support; they do not by
themselves establish a causal chemical effect.

\subsection{Ablation study}
\label{subsec:app_ablation_study}

The normalized RRWP parameterization in \HARMONIA is introduced to make
the structural component comparable across features and runs.
For feature $k$, the structural propagation operator is
\begin{equation}
    \mathbf P_k
    =
    \sum_{t=0}^{T-1}
    \theta_{t,k}\mathbf M^t .
\end{equation}
Without an explicit normalization constraint, the decomposition between the
feature response and structural influence admits a multiplicative scale
ambiguity. For any $a>0$, replacing $f_k$ by $af_k$ and the structural
coefficients by $\theta_{t,k}/a$ leaves their product, and hence the prediction,
unchanged. \HARMONIA removes this arbitrary scale by parameterizing
\begin{equation}
\theta_{t,k}
=
\frac{\alpha_{t,k}^{2}}
{\sum_{s=0}^{T-1}\alpha_{s,k}^{2}},
\qquad
\theta_{t,k}\ge 0,
\qquad
\sum_{t=0}^{T-1}\theta_{t,k}=1.
\label{eq:app_rrwp_simplex_ablation}
\end{equation}
In implementation, $10^{-8}$ is added to each squared coefficient for numerical
safety. Since every $\mathbf M^t$ is row-stochastic, the resulting
$\mathbf P_k$ is a convex combination of random-walk operators. The exposed
hop coefficients therefore lie on a common reference scale, and the structural
operator retains the non-expansive property discussed in
Section~\ref{subsec:normalized_rrwp}.

We evaluate the empirical consequence of this simplex constraint by comparing
the normalized model in Eq.~\eqref{eq:app_rrwp_simplex_ablation} with an
otherwise identical model using unconstrained structural coefficients,
\begin{equation}
    \theta_{t,k}=\alpha_{t,k}.
\end{equation}
All other components, including the MoNB feature model, routing mechanism,
number of walk orders, optimization procedure, and data splits, are held fixed.
Because the unconstrained variant removes both the unit-sum and nonnegativity
constraints, this experiment should be interpreted as an ablation of the
simplex-constrained structural parameterization rather than as an isolation of
normalization alone.

\paragraph{Synthetic setup.}
We use a controlled synthetic setting in which both the feature mechanism and
the structural mechanism are known exactly. Each data seed contains a fixed
graph with $64$ nodes and $20$ continuous raw features, of which four are
informative and sixteen are null. The four informative features use polynomial
($f(x)=1.5(x^2-1)$), periodic
($f(x)=\sin(\pi x)$), saturating
($f(x)=\tanh(2x)$), and localized
($f(x)=\exp[-8(x-0.4)^2]$) response functions, respectively. Their
ground-truth structural profiles over $T=4$ random-walk orders are
\begin{align}
\boldsymbol{\theta}^{\star}_{\mathrm{poly}}
    &= [0.70,\,0.20,\,0.07,\,0.03],\\
\boldsymbol{\theta}^{\star}_{\mathrm{periodic}}
    &= [0.10,\,0.70,\,0.15,\,0.05],\\
\boldsymbol{\theta}^{\star}_{\mathrm{saturating}}
    &= [0.05,\,0.15,\,0.70,\,0.10],\\
\boldsymbol{\theta}^{\star}_{\mathrm{localized}}
    &= [0.03,\,0.07,\,0.20,\,0.70].
\end{align}
Thus, the four causal features are dominated respectively by $0$-, $1$-,
$2$-, and $3$-hop propagation. Synthetic logits are generated from the known
feature functions and the row-normalized random-walk powers
$\mathbf I,\mathbf M,\mathbf M^2,\mathbf M^3$, and binary labels are sampled
from the corresponding Bernoulli probabilities.

We use three fixed graph/data seeds and ten optimization seeds for each data
seed and each parameterization. Within every matched pair, the graph, features,
labels, train/validation/test split, non-topology initialization, minibatch
order, optimizer, scheduler, and early stopping protocol are identical.
Results are first averaged across optimization seeds within each data seed and
are then reported as mean $\pm$ sample standard deviation across the three
data-seed means.

\paragraph{Faithfulness metrics.}
We evaluate explanation faithfulness at three complementary levels.

At the feature/hop level, Effective NRMSE measures how closely the
model reproduces the oracle effective response to an intervention on a causal
feature. For each intervention, the predicted and oracle propagated logit
response curves are compared, and the RMSE is normalized by the standard
deviation of the oracle curve. Structural recovery is evaluated using
intervention-based hop probes. The induced node-wise logit changes are fitted
with the walk-power basis by least squares; retained profiles are oriented and
$\ell_1$-normalized before comparison with the known
$\boldsymbol{\theta}^{\star}_k$. Hop L1 is the resulting
$\ell_1$ distance between recovered and oracle hop profiles.

At the source-node level, all features of one source node are replaced
by their training-set means and the resulting signed change in target logit is
compared with the same intervention under the known data-generating process.
Signed NRMSE is the RMSE between predicted and oracle signed node-contribution
vectors, normalized by the oracle standard deviation. Precision@4 measures the
overlap between the four largest-magnitude predicted source-node contributions
and the four largest-magnitude oracle contributions.

At the edge level, each candidate edge in the target node's three-hop
neighborhood is removed, after which the row normalized transition matrix and
the powers $\mathbf M^0,\ldots,\mathbf M^3$ are recomputed. Edge importance is
the absolute change in the target logit. Spearman correlation measures rank
agreement between predicted and oracle edge importance over all candidate
edges. For
\begin{equation}
K=
\left\lceil
0.2\,|\mathcal E_{\mathrm{cand}}|
\right\rceil,
\end{equation}
Precision@20\% measures the overlap between the predicted and oracle top-$K$
edges.

\begin{figure}[t]
    \centering
    \begin{subfigure}[t]{0.62\linewidth}
        \centering
        \includegraphics[width=\linewidth]
        {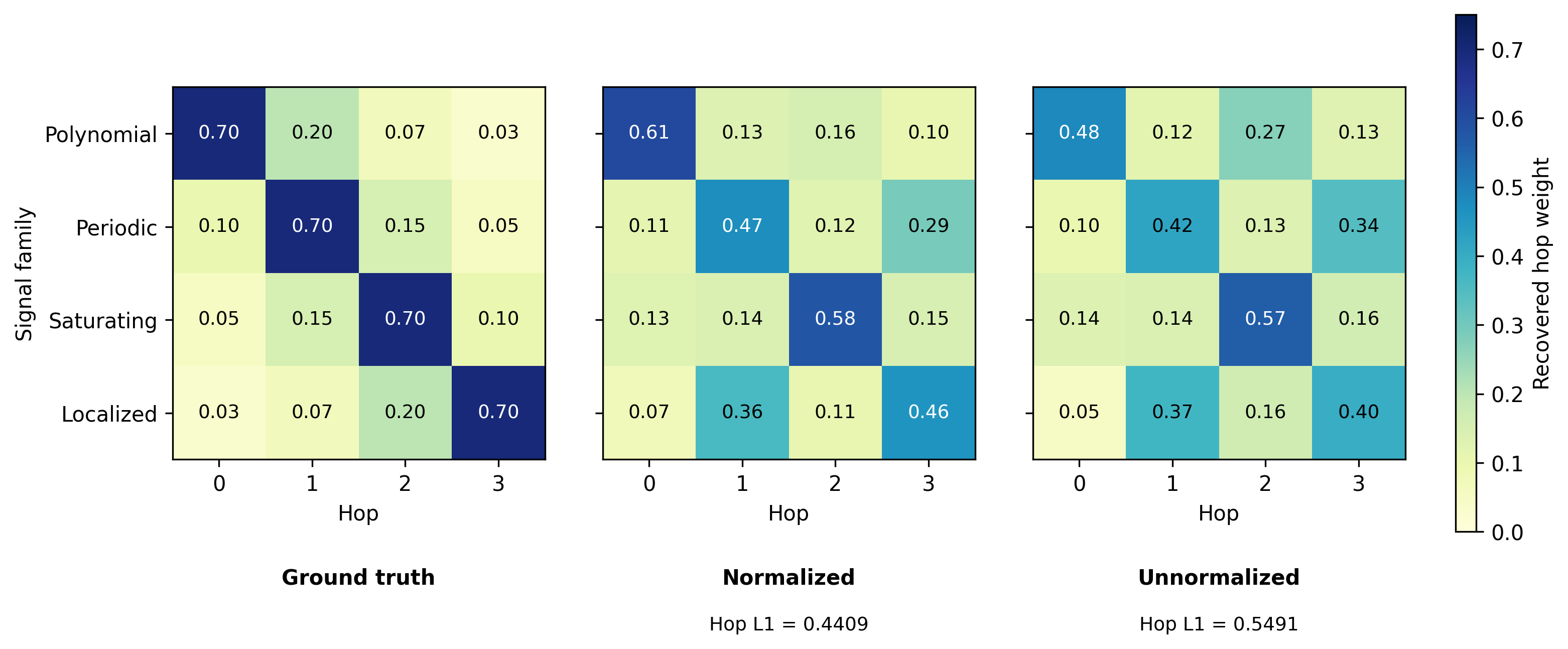}
        \caption{Recovery of feature-specific hop profiles.}
        \label{fig:rrwp_norm_hop}
    \end{subfigure}
    \hfill
    \begin{subfigure}[t]{0.36\linewidth}
        \centering
        \includegraphics[width=\linewidth]
        {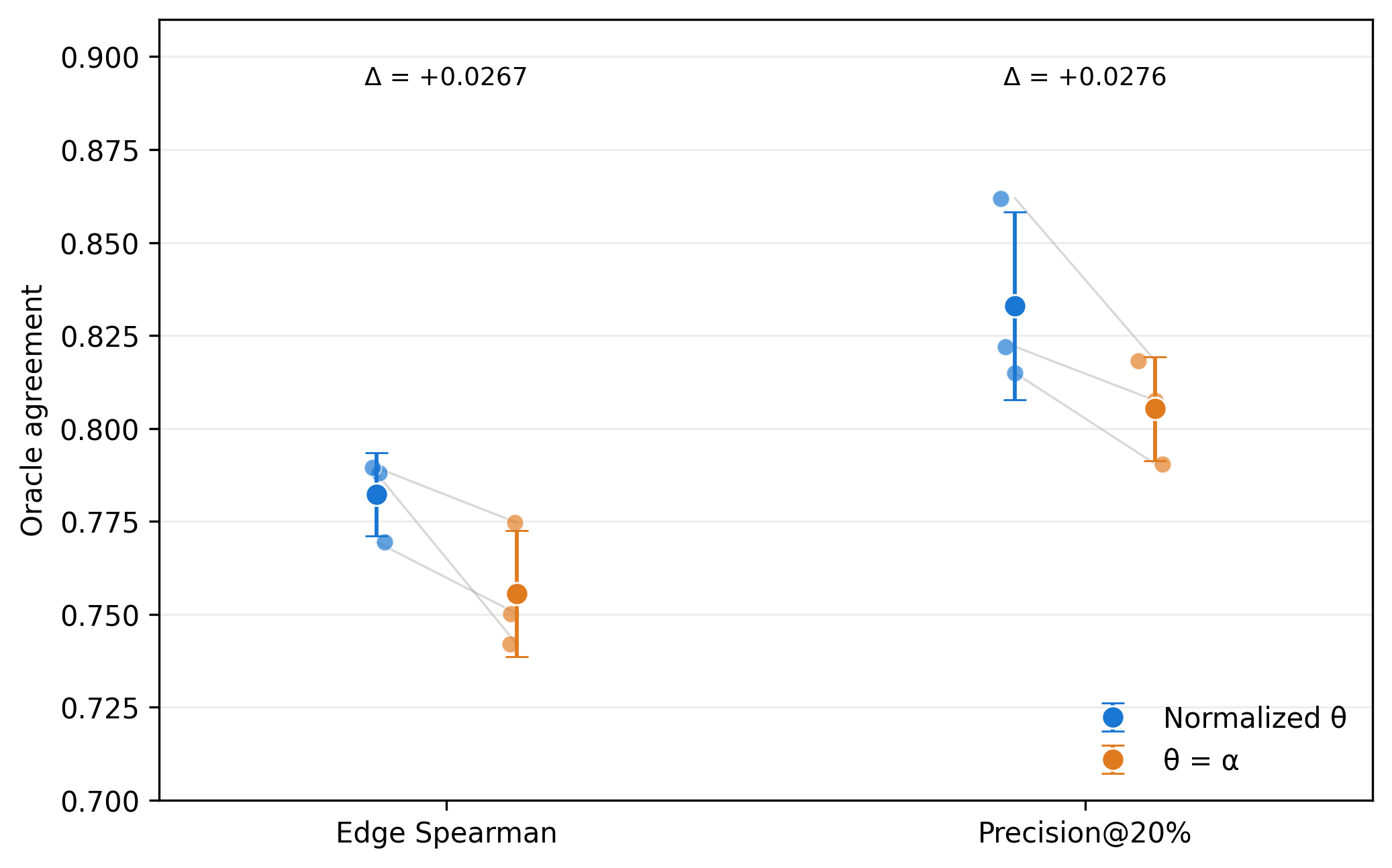}
        \caption{Agreement with oracle edge importance.}
        \label{fig:rrwp_norm_edge}
    \end{subfigure}

    \caption{
    Effect of the RRWP simplex constraint on synthetic explanation
    faithfulness.
    \textbf{(a)} Ground-truth and recovered hop profiles for the four
    informative features. Recovered profiles use the same orientation and
    $\ell_1$ normalization as the intervention-based hop evaluation.
    The simplex-normalized model reduces mean hop-profile $\ell_1$ error
    from $0.5491$ to $0.4409$, corresponding to a $19.7\%$ relative
    reduction.
    \textbf{(b)} Agreement between predicted and oracle edge importance.
    The normalized model increases edge-ranking Spearman correlation from
    $0.7556$ to $0.7823$ and Precision@20\% from $0.8053$ to $0.8329$.
    Points denote individual data-seed means, lines connect matched seeds,
    and error bars show sample standard deviation across the three data seeds.
    }
    \label{fig:rrwp_normalization_faithfulness}
\end{figure}

\begin{table}[t]
\centering
\caption{
Representative explanation-faithfulness metrics for simplex-normalized and
unconstrained RRWP coefficients. For error metrics, Change reports the relative
reduction from the unconstrained variant; for precision metrics it reports the
percentage point improvement; and for Spearman correlation it reports the
absolute improvement. Results are mean $\pm$ sample standard deviation across
three fixed synthetic data seeds.
}
\label{tab:rrwp_normalization_faithfulness}
\small
\setlength{\tabcolsep}{6pt}
\begin{tabular}{llccc}
\toprule
Level
& Metric
& Normalized $\theta$
& Unnormalized $\theta=\alpha$
& Change \\
\midrule

\multirow{2}{*}{Feature / Hop}
& Effective NRMSE $\downarrow$
& \textbf{0.2049 $\pm$ 0.0139}
& 0.2279 $\pm$ 0.0072
& \textbf{$-10.1\%$} \\

& Hop L1 $\downarrow$
& \textbf{0.4409 $\pm$ 0.0221}
& 0.5491 $\pm$ 0.0188
& \textbf{$-19.7\%$} \\

\midrule

\multirow{2}{*}{Node}
& Signed NRMSE $\downarrow$
& \textbf{0.3125 $\pm$ 0.0302}
& 0.3252 $\pm$ 0.0281
& \textbf{$-3.9\%$} \\

& Precision@4 $\uparrow$
& \textbf{0.8509 $\pm$ 0.0195}
& 0.8409 $\pm$ 0.0179
& \textbf{$+1.0$ pp} \\

\midrule

\multirow{2}{*}{Edge}
& Spearman $\uparrow$
& \textbf{0.7823 $\pm$ 0.0112}
& 0.7556 $\pm$ 0.0170
& \textbf{$+0.0267$} \\

& Precision@20\% $\uparrow$
& \textbf{0.8329 $\pm$ 0.0253}
& 0.8053 $\pm$ 0.0140
& \textbf{$+2.76$ pp} \\

\bottomrule
\end{tabular}
\end{table}

\paragraph{Results.}
The main effect of the simplex constraint appears in the faithfulness of the
recovered mechanism rather than in feature discovery. At the feature level,
normalization reduces Effective NRMSE from $0.2279$ to $0.2049$, a $10.1\%$
relative reduction. The largest gain occurs in the structural decomposition:
Hop L1 decreases from $0.5491$ to $0.4409$, corresponding to a $19.7\%$
relative reduction. Figure~\ref{fig:rrwp_norm_hop} shows the same effect
visually. The normalized model preserves the intended dominant propagation
scale more closely and places less recovered mass on spurious walk orders.

The improvement also extends to explanations at finer graph resolutions.
Signed source-node contribution NRMSE decreases from $0.3252$ to $0.3125$,
while source node Precision@4 increases from $0.8409$ to $0.8509$. At the edge
level, Spearman agreement with oracle edge importance increases from $0.7556$
to $0.7823$, and Precision@20\% increases from $0.8053$ to $0.8329$.
Figure~\ref{fig:rrwp_norm_edge} shows that the edge-level improvement is
consistent across the three matched data seeds.

\newpage

\end{document}